\documentclass[nohyperref]{article}

\usepackage{microtype}
\usepackage{graphicx}
\usepackage{subfigure}
\usepackage{booktabs} %

\usepackage{hyperref}

\usepackage[accepted]{icml2023}

\usepackage{amsmath}
\usepackage{amssymb}
\usepackage{mathtools}
\usepackage{amsthm}
\usepackage{thmtools} 
\usepackage{thm-restate}
\usepackage{enumitem}
\newcommand{\algoname}{\textsc{Safe-DSHB}}

\newcommand{\aggregation}{F}

\newcommand{\sample}{\mathsf{subsample}}

\newcommand{\sigmadp}{\sigma_{\mathrm{DP}}}
\newcommand{\sigmadpbar}{\overline{\sigma}_{\mathrm{DP}}}
\newcommand{\sigmabar}{\overline{\sigma}}

\newcommand{\gcov}{G_{\mathrm{cov}}}
\newcommand{\ones}{\textbf{1}}

\def\D{\mathcal{D}}
\def\A{\mathcal{A}}
\def\Y{\mathcal{Y}}
\def\N{\mathcal{N}}
\def\R{\mathbb{R}}
\def\M{\mathcal{M}}

\def\U{\mathcal{U}}
\def\P{\mathcal{P}}

\def\H{\mathcal{H}}
\def\X{\mathcal{X}}

\declaretheorem[name=Theorem,numberwithin=section]{theorem}
\declaretheorem[name=Proposition,numberwithin=section]{proposition}
\declaretheorem[name=Corollary,numberwithin=section]{corollary}
\declaretheorem[name=Lemma,numberwithin=section]{lemma}
\declaretheorem[name=Definition,numberwithin=section]{definition}

\newcommand{\expect}[1]{\mathop{{}\mathbb{E}}\left[{#1}\right]}
\newcommand{\condexpect}[2]{\mathbb{E}_{#1}\left[{#2}\right]}

\newcommand{\suchthat}{\ensuremath{~\middle|~}}
\newcommand{\knowing}{\suchthat{}}
\newcommand{\card}[1]{\left\lvert{#1}\right\rvert}
\newcommand{\absv}[1]{\card{#1}}

\newcommand{\norm}[1]{\left\lVert{#1}\right\rVert}

\newcommand{\floor}[1]{\left\lfloor{#1}\right\rfloor}
\newcommand{\ceil}[1]{\left\lceil{#1}\right\rceil}

\newcommand{\indexvar}[3]{\ensuremath{{{#3}^{\ifthenelse{\equal{#1}{}}{}{\left({#1}\right)}}_{#2}}}}
\newcommand{\indexvarNoPar}[3]{\ensuremath{{{#3}^{\ifthenelse{\equal{#1}{}}{}{\left{#1}\right}}_{#2}}}}

\newcommand{\params}[2]{\indexvarNoPar{#1}{#2}{\theta}}

\newcommand{\mechanism}[1]{\ensuremath{\mathcal{M}({#1})}}

\providecommand{\iprod}[2]{\ensuremath{\left\langle #1,\,#2  \right\rangle}}
\providecommand{\mnorm}[1]{\ensuremath{\left\lvert#1\right\rvert}}
\providecommand{\norm}[1]{\ensuremath{\left\lVert#1\right\rVert }}

\newcommand{\loss}{\mathcal{L}}
\newcommand{\weight}[1]{\params{}{#1}}

\newcommand{\gradient}[2]{\indexvar{#1}{#2}{g}}

\newcommand{\proba}[2]{\ensuremath{\text{P}\!\left({#1}\ifthenelse{\equal{#2}{}}{}{\knowing{}{#2}}\right)}}

\DeclareMathOperator*{\argmin}{argmin}

\renewcommand{\paragraph}[1]{\textbf{#1}~}

\newcommand{\drift}[1]{\epsilon_{#1}}
\newcommand{\dev}[1]{\delta_{#1}}
\newcommand{\mmt}[2]{m^{(#1)}_{#2}}

\newcommand{\AvgMmt}[1]{\overline{m}_{#1}}

\usepackage{tikz-cd}
\usepackage{ctable}
\usetikzlibrary{shapes.geometric}
\usepackage{nicefrac}
\usepackage{hyperref}

\usepackage[capitalize,noabbrev]{cleveref}

\usepackage[page,header]{appendix}
\usepackage{titletoc}
\usepackage[most]{tcolorbox}

\theoremstyle{plain}
\newtheorem{assumption}[theorem]{Assumption}
\theoremstyle{remark}
\newtheorem{remark}[theorem]{Remark}

\usepackage[textsize=small]{todonotes}

\icmltitlerunning{On the Privacy-Robustness-Utility Trilemma in Distributed Learning}

\begin{document}

\twocolumn[
\icmltitle{On the Privacy-Robustness-Utility Trilemma in Distributed Learning}

\begin{icmlauthorlist}
\icmlauthor{Youssef Allouah}{yyy}
\icmlauthor{Rachid Guerraoui}{yyy}
\icmlauthor{Nirupam Gupta}{yyy}
\icmlauthor{Rafaël Pinot}{yyy}
\icmlauthor{John Stephan}{yyy}
\end{icmlauthorlist}

\icmlaffiliation{yyy}{Ecole Polytechnique Fédérale de Lausanne (EPFL), Switzerland}

\icmlcorrespondingauthor{Youssef Allouah}{youssef.allouah@epfl.ch}

\icmlkeywords{Differential Privacy, Byzantine, Robust, Distributed Optimization, Federated Learning}

\vskip 0.3in
]

\printAffiliationsAndNotice{}  %

\begin{abstract} 
The ubiquity of distributed machine learning (ML) in sensitive public domain applications calls for algorithms that protect data \emph{privacy}, while being \emph{robust} to faults and adversarial behaviors. Although privacy and robustness have been extensively studied independently in distributed ML, their synthesis remains poorly understood. We present the first {\em tight} analysis of the error incurred by any algorithm ensuring robustness against a fraction of adversarial machines, as well as {\em differential privacy} (DP) for honest machines' data against any other curious entity. Our analysis exhibits a fundamental trade-off between privacy, robustness, and utility. 
To prove our lower bound, we consider the case of mean estimation, subject to distributed DP and robustness constraints, and devise reductions to centralized estimation of one-way marginals. We prove our matching upper bound by presenting a new distributed ML algorithm using a high-dimensional robust aggregation rule. The latter amortizes the dependence on the dimension in the error (caused by adversarial workers and DP), while being agnostic to the statistical properties of the data. 
\end{abstract}

\section{Introduction}
\label{sec:intro}
Distributed machine learning (ML) has been playing a pivotal role in a wide range of applications~\citep{DistributedNetworks2012,abadi2016deep}, due to an unprecedented growth in the complexity of ML models and the volume of data being used for training purposes.
Distributed ML breaks a complex ML task into sub-tasks that are performed in a collaborative fashion. In the standard {\em server-based} architecture, $n$ machines (a.k.a., {\em workers}) collaboratively train a global model on their datasets, with the help of a coordinator (the \emph{server}). This is typically achieved through a distributed implementation of the renowned {\em stochastic gradient descent} ({SGD}) algorithm~\cite{bertsekas2015parallel}. In distributed SGD (or DSGD), the server maintains a model which is updated iteratively by averaging gradients of the loss function associated with the model, computed by the different workers upon sampling random points from their local datasets. 
DSGD is particularly useful in cases where the data held by the workers is too sensitive to be shared, 
e.g., medical data collected by several hospitals~\cite{sheller2020federated}.

\paragraph{Privacy.} Although DSGD inherently ensures privacy of the workers' data to an extent, by not sharing it explicitly, information leakage can still be significant. When the ML model maintained at the server is publicly released, it may be exposed to membership inference~\cite{shokri} or model inversion attacks~\cite{Fredrikson2015ModelInversion,Hitaj2017PrivacyLeakageFromCollLearning,MelisSCS19GradientLeakage} by external entities. %
Furthermore, upon observing the gradients and transient models during the learning procedure, {\em curious} machines (be they workers or the server itself)  can infer sensitive information about the datasets held locally by the machines, or even reconstruct  data points in certain scenarios~\citep{PhongPPDeepLearning2017,WangUserLevelPrivacyLeakage2019,DLG,zhao2020idlg}.

\paragraph{Robustness.} 
In real-world distributed systems, it is arguably inevitable to encounter faulty workers that may deviate from their prescribed algorithm. This may result from hardware and software bugs, data corruption, network latency, or malicious adversaries controlling a subset of workers. To cover all such possible scenarios, 
it is common to assume that a fraction of the machines 
can be \emph{adversarial}\footnote{Sometimes called ``Byzantine'' in the parlance of distributed computing~\cite{lamport82}.} and arbitrarily deviate from their algorithms.
In the context of DSGD, adversarial workers may send incorrect gradients~\cite{feng2015distributed, su2016fault} to the server 
and critically influence the learning procedure, as shown in~\cite{little, empire}. 

\paragraph{Integrating privacy and robustness.}
With the growing concerns and legal obligations regarding the processing of public data in AI-driven technologies~\cite{gdpr},  privacy and robustness issues question the very applicability of ML in critical public domain services, such as healthcare or banking. It is thus natural to  
seek distributed ML methods that simultaneously ensure privacy and robustness. In fact, these aspects 
 have separately received significant attention in the past. On the one hand, the standard statistical privacy requirement of $(\varepsilon,\delta)$-{\em differential privacy} ($(\varepsilon,\delta)$-DP) has been studied to a great extent in the context of distributed ML~\cite{choudhury2019differential,hu2020personalized,noble2022differentially}. On the other hand, numerous provably robust adaptations of DSGD have been proposed~\cite{krum,meamed,yin2018byzantine,gupta2021byzantine,farhadkhani2022byzantine}. Yet, the synthesis of privacy and robustness remains highly understudied in distributed ML. The few works on this topic, such as~\cite{podc, zhu22bridging, xiang2022beta, ma2022differentially}, only focus on {\em per-step privacy}, and provide loose upper bounds on the learning error.
On the other hand, the guarantees presented in~\cite{cheu2021manipulation,acharya2021robust} only apply to discrete distribution estimation subject to {\em non-interactive local DP}~\cite{kasiviswanathan2011can}, a restricted case of distributed ML where each worker holds a single data point and can be queried only once.

 An orthogonal line of work studied the case where the server is assumed not to be curious, i.e., data only needs to be protected against the public release of the model~\cite{DworkRobustStats2009,liu2021robust,hopkins2022efficient,liu2022differential}. In this setting, it was recently shown that privacy and robustness are {\em mutually} beneficial~\cite{georgiev2022privacy,hopkins2022robustness}. However, the assumption of a non-curious server may not be viable,
 especially in applications such as healthcare and finance, where sovereignty of data must be protected at every stage of the learning procedure~\cite{lowy2022private}. In this paper, we focus on the setting where the server itself may be curious, 
 and we show that 
 privacy and robustness are actually at odds.

\subsection{Contributions}
\label{sec:contribution}

We precisely characterize the privacy-robustness-utility trilemma in distributed learning.
Specifically, we present the first {\em tight} analysis of the error incurred by any distributed ML algorithm that simultaneously ensures 
(i) robustness against a minority of adversarial workers,
and (ii) differential privacy (DP) of each worker's data against curious entities including other workers and the server.
In short, we show that, in addition to the usual separate costs of privacy and robustness, the learning accuracy necessarily degrades due to their interplay. 

\paragraph{Main results.}
We consider a system of $n$ workers up to $f$ of which (of unknown identity) may be adversarial, and the remainder are {\em honest}. 
The server is assumed {\em honest-but-curious}~\cite{bonawitz2016practical}. Each honest worker holds a dataset comprising $m$ points. 
The goal of the server is to learn a model, parameterized by a $d$-dimensional vector, incurring minimum loss over the collective dataset of the honest workers. We denote by $G$ the {\em heterogeneity}~\cite{karimireddy2020scaffold,karimireddy2022byzantinerobust} between the honest datasets.

We show that a distributed learning algorithm that is robust to $f$ adversarial workers, while ensuring $(\varepsilon, \delta)$-DP of each honest worker's data against the server (and other curious workers) incurs a training error in
\begin{equation}
\label{eq:LB}
    \widetilde{\Omega} \left( \textcolor{blue!90}{\frac{d}{\varepsilon^2 n m^2}} + \textcolor{violet!90}{\frac{f}{n} \cdot \frac{1}{\varepsilon^2 m^2}} + \textcolor{red!90}{\frac{f}{n} \cdot G^2}\right),
\end{equation}
where $\widetilde{\Omega}$ ignores the logarithmic terms. 

The 
\textcolor{blue!90}{first} and the \textcolor{red!90}{third} terms in~\eqref{eq:LB} are the respective errors due to privacy and robustness separately. Importantly, the \textcolor{violet!90}{second} term represents the additional cost of satisfying privacy and robustness simultaneously. 
We then present a new distributed ML algorithm, \algoname{}\footnote{Safe Distributed Stochastic Heavy Ball method, inspired from the optimization literature~\cite{gadat2018stochastic}.}, which we prove yields a matching upper bound (up to a logarithmic factor) for the class of smooth and strongly convex loss functions, while ensuring both privacy and robustness.
We also obtain an upper bound for smooth {\em non-convex} learning problems.

The key to proving the tightness of this trade-off is the robust high-dimension aggregation rule we introduce, namely \emph{SMEA}\footnote{Smallest Maximum Eigenvalue Averaging.}. As an important consequence of our result, we observe that the privacy-robustness trade-off (\textcolor{violet!90}{second} term) is dominated by the privacy cost alone (\textcolor{blue!90}{first} term) when the dimension $d$ is larger than the number of adversarial workers $f$. This observation however does not mean that the trade-off is not significant, but rather that it can be adequately controlled when using SMEA. This would not have been possible otherwise with the use of existing aggregation rules such as coordinate-wise or geometric median, for which the upper bound has an additional dimension factor in the privacy-robustness trade-off.

\paragraph{Independent contributions.}
As a byproduct of our analysis, we obtain several results that are of independent interest to both the robust distributed ML and the privacy communities. Indeed, our upper bound is tight for strongly convex losses, even when removing the privacy constraints. This is mainly due to the use of momentum in \algoname{} (see Section~\ref{sec:techniques} below) which allows obtaining an excess error that is independent of the variance of local stochastic gradients. This improves over the state-of-the-art analysis on robust distributed learning with strongly convex losses~\cite{data2021byzantine}, which induces a suboptimal excess error. Besides, our analysis features a tighter dependence on heterogeneity in the excess error.
Our lower bound on the cost of privacy (without robustness) also improves over the state-of-the-art~\cite{lowy2021private} as we make no assumptions on the {\em interactivity} of the algorithm and impose weaker conditions on the DP parameter $\varepsilon$ (see Section~\ref{sec:lb}).

\subsection{Overview of Proof Techniques} 
\label{sec:techniques}

\paragraph{Lower bound.}
We prove our lower bound by reducing distributed mean estimation to centralized estimation of one-way marginals (i.e. row-wise averages). We distinguish cases depending on the presence of adversarial workers. In each case, we start with a distributed algorithm $\A$ whose interactions with each worker are $(\varepsilon,\delta)$-DP, and then construct a \emph{centralized} algorithm $\M$ using $\A$. Depending on the case, we then use either the advanced composition theorem~\cite{dwork2014algorithmic} or an indistinguishability argument on the honest identities to relate the DP and utility guarantees of $\M$ to those of $\A$. We conclude by applying lower bounds on centralized private estimation of one-way marginals~\cite{steinke2016between} to $\M$.

\paragraph{Upper bound.}
To prove our matching upper bound, we present~\algoname{}, a privacy-preserving robust adaptation of DSGD. Our algorithm incorporates {\em Polyak's momentum}~\cite{POLYAK19641} and a {\em Gaussian mechanism}~\cite{dwork2014algorithmic} at the worker level, as well as 
{\em SMEA}, our robust aggregation rule at the server level.
We identify a key property that, if satisfied by an aggregation rule, mitigates the curse of dimensionality that could impact the  Gaussian mechanism. This property, called {\em $(f,\kappa)$-robust averaging}, requires the squared distance between the aggregate and the average of honest vectors to be bounded by $\kappa$ times the spectral norm of the empirical covariance matrix of the honest vectors. 
Our aggregation rule, SMEA, satisfies $(f,\kappa)$-robust averaging for $\kappa = \mathcal{O}{(\nicefrac{f}{n})}$, 
while being agnostic to the statistical properties of honest inputs. Another critical element of our analysis is 
the tuning of the momentum coefficients to control the trade-off between the deviation from the true gradient and the reduction of the {\em drift} between
honest workers' momentums. 
We achieve this through a novel {\em Lyapunov function} (a.k.a. potential function in optimization literature~\cite{schmidt2017minimizing}).

\subsection{Prior Work}
\label{sec:rel}

Only a handful of works addressed the interplay between DP and robustness in distributed ML.
It was conjectured that ensuring both these requirements is {\em impractical}, in the sense that it would require the batch size to grow with the model dimension~\cite{podc}. However, the underlying analysis relied upon the criterion of $(\alpha, \, f)$-Byzantine resilience~\cite{krum}, which has been recently shown to be a restrictive sufficient condition~\cite{Karimireddy2021}. 
Subsequent works~\cite{zhu22bridging,xiang2022beta,ma2022differentially} augmented the RSA learning algorithm~\cite{li2019rsa} with the sign-flipping or sign-Gaussian privacy mechanisms.
However, these works only focus on {\em per-step} DP, and the presented upper bounds on the error of the proposed algorithms are loose. 

Another line of work targeted the specific learning problem of \emph{discrete distribution estimation} subject to {\em non-interactive} {\em local DP}~\cite{duchi2013local} and robustness constraints. 
The bounds for this problem~\cite{cheu2021manipulation,acharya2021robust} are comparable to ours 
in the particular scenario where
each worker holds a single data point and  the algorithm is non-interactive (can query each worker once).
Although a recent paper~\cite{chhor2022robust} considered a more general case where workers hold a batch of data points, the algorithm was still assumed non-interactive, and the data distribution  identical for all the workers. 
It was also shown recently~\cite{li2022robustness}  that local DP and robustness are disentangled when the adversarial workers corrupt the data before randomization only, which however need not be the case in general. The aforementioned works being tailored to non-interactive local DP, it is not clear how to extend their results to the general distributed ML setting.

Significant attention was given to robust mean estimation under DP~\cite{DworkRobustStats2009,liu2021robust,hopkins2022efficient,liu2022differential}. However, as we pointed out, the corresponding results do not readily apply to our setting, as they would require the server to be {\em non-curious}.
Moreover,  robust mean estimation~\cite{diakonikolas2019robust,ashtiani2022private,liu2022differential} typically assumes the honest inputs to be identically distributed, which need not be the case in a general distributed setting.

\subsection{Paper Outline}
Section~\ref{sec:problem} defines the problem and recalls some useful concepts.
Sections~\ref{sec:lb} and~\ref{sec:ub} present our lower bound and the analysis of \algoname{}.
Section~\ref{sec:tight} presents SMEA and derives our matching upper bound.
Section~\ref{sec:conclusion} discusses future work.
We defer full proofs to appendices~\ref{app:lb}-\ref{app:ub}, and experimental evaluation to \cref{app:experiments}.

\section{Problem Statement}
\label{sec:problem}
We consider the classical server-based architecture comprising $n$ workers
$w_1, \dots ,w_n$, 
and a central server. The workers hold local datasets $\D_1, \dots,\D_n$, each composed of $m$ data points from an input space $\X$, i.e., $ \D_i \coloneqq \{x_1^{(i)}, \dots, x_m^{(i)}\} \in \X^m$. For a given parameter vector $\theta \in \R^d$, a data point $x \in \X$ has a real-valued loss function $\ell(\theta; x)$. The empirical loss function for each worker $w_i$ is defined by 
\begin{equation*}
    \loss{(\theta; \D_i)} \coloneqq \frac{1}{m}\sum_{x \in \D_i} \ell{(\theta;x)}.
\end{equation*}
The goal of the server is to compute an optimal parameter vector $\theta^{*}$ minimizing the global empirical loss function $\loss{(\theta; \D_1, \ldots, \D_n)}$ defined to be
\begin{equation*}
    \loss{(\theta; \D_1, \ldots, \D_n)} \coloneqq \frac{1}{n} \sum_{i = 1}^n \loss{(\theta; \D_i)}.
\end{equation*}
We assume that each loss $\loss{(\cdot; \D_i)}$ is differentiable, and that $\loss$ is lower bounded, i.e., $\inf_{\theta \in \R^d} \loss(\theta; \D_1, \ldots, \D_n)$ is finite.

\subsection{Robustness}
We consider a setting where at most $f$ out of $n$ workers may be adversarial. Such workers may send arbitrary messages to the server, and need not follow the prescribed protocol. The identity of adversarial workers is a priori unknown to the server. Let $\H \subseteq \{1, \ldots, \, n\}$, with $\card{\H} = n-f$. We define 
\begin{align*}
    \loss_{\H}{(\theta{})} \coloneqq  \loss{(\theta; \D_i, \, i \in \H)} \coloneqq \frac{1}{\absv{\H}} \sum_{i \in \H} \loss(\theta{}; \D_i).
\end{align*}
If $\H$ represents the indices of honest workers, the function $\loss_{\H}$ is referred to as the global {\em honest loss}. An algorithm is deemed robust to adversarial workers if it enables the server to compute a minimum of the global honest loss~\cite{gupta2020fault}.
Formally, we define robustness as follows.

\begin{definition}[{\bf $(f, \varrho)$-robust}]
\label{def:resilience}
A \emph{distributed} algorithm is said to be {\em $(f, \varrho)$-robust} 
if
it outputs a parameter $\hat{\theta}$ 
such that
    $$\expect{\loss_{\H}{(\hat{\theta})}-\loss_*} \leq \varrho,$$
where $\loss_* \coloneqq \inf_{\theta \in \R^d} \loss_\H(\theta)$, and the expectation is taken over the randomness of the algorithm. 
\end{definition}
In other words, an algorithm $\A$ is said to be $(f, \varrho)$-robust if, in every execution of $\A$, the server outputs a {\em $\varrho$-approximate} minimizer of the honest loss, despite the presence of up to $f$ 
adversarial workers.
Note that $(f, \varrho)$-robustness is in general impossible for any $\varrho$ when $f \geq \frac{n}{2}$~\cite{liu2021approximate}. Thus, throughout the paper, we assume that $f < \frac{n}{2}$.

\subsection{Differential Privacy}
Each honest worker $w_i, i \in \H,$ aims to protect the privacy of their dataset $\D_i$ against all other entities, i.e., the server and the other workers. To define our privacy requirement formally, we recall below
the definition of item-level differential privacy (DP)~\cite{dwork2014algorithmic}, where two datasets are said to be adjacent if they differ by one item.

\begin{definition}[{\bf $(\varepsilon,\delta)$-DP}]
\label{def:dp}
Let $\varepsilon \geq 0$, $\delta \in [0, 1]$.
A randomized algorithm $\M:\X^m \rightarrow \mathcal{Y}$ satisfies $(\varepsilon, \delta)$-DP if for any {\em adjacent} datasets
$\D, \D' \in \X^m$ and 
subset $S \subseteq \mathcal{Y}$, we have
\begin{equation}
    \label{eq:dp-def}
    \mathbb{P}[\mechanism{\D} \in S] \leq e^{\varepsilon}\cdot \mathbb{P} \left[ \mechanism{\D'} \in S \right] + \delta.
\end{equation}
\end{definition}

We consider the server to be \emph{honest-but-curious}, i.e., it follows the prescribed algorithm correctly, but may try to infer sensitive information about the workers' datasets. Thus, the workers must enforce privacy locally at their end.
We assume that the server can only query the dataset of a worker $w_i$ through a dedicated communication channel, and that there is no direct communication between the workers. Hence, for privacy in this context, we require
the communications between the server and each honest worker to satisfy the criterion of DP in~\eqref{eq:dp-def}. 
In our context, we formalize this property below, inspired from~\cite{smith2017interaction}.
\begin{definition}[{\bf $(\varepsilon, \, \delta)$-distributed DP}]
    \label{def:distributed-dp}
    Let $\varepsilon \geq 0$, $\delta \in [0, 1]$. Consider a randomized \emph{distributed} algorithm $\A : \X^{m \times n} \to \Y$. Let $Z_i$ be a function that outputs the transcript of communications between the server and worker $w_i$ during the execution of $\A$. Algorithm $\A$ is said to satisfy {\em $(\varepsilon,\delta)$-distributed DP} if for all $i \in \H$, $Z_i$ satisfies $(\varepsilon, \, \delta)$-DP with respect to the dataset held by worker $w_i$.
\end{definition}

The above criterion of distributed DP reduces to {\em local DP}~\cite{kasiviswanathan2011can,duchi2013local} when each local dataset comprises a single item (i.e., $m=1$). Moreover, an algorithm satisfying $(\varepsilon,\delta)$-distributed DP may be fully {\em interactive}, i.e., the queries made to the workers by the server may share arbitrary dependence~\cite{kasiviswanathan2011can}. Hereafter, a distributed algorithm satisfying $(\varepsilon,\delta)$-distributed DP is simply said to be $(\varepsilon,\delta)$-DP.

\subsection{Assumptions}
Our results are derived under standard assumptions.
First, we recall that data heterogeneity can be modeled following the assumption below~\cite{karimireddy2020scaffold,karimireddy2022byzantinerobust}.
\begin{assumption}[Bounded heterogeneity]
\label{asp:hetero}
There exists $G < \infty$ such that for all $\theta{} \in \mathbb{R}^d$,
\begin{align*}
    \frac{1}{\card{\H}} \sum_{i \in \H}\norm{\nabla \loss{\left(\theta{}; \D_i\right)} - \nabla \loss_{\H}(\theta{})}^2 \leq G^2.
\end{align*}
\end{assumption}

To present the convergence guarantees of \algoname{}, we make the following standard assumption on the variance of stochastic gradients~\cite{bottou2018optimization}.

\begin{assumption}[Bounded variance]
\label{asp:bnd_var}
There exists $\sigma < \infty$ such that for each honest worker $w_i, i \in \H$, and all $\theta{} \in \R^d$,
\begin{equation*}
    \frac{1}{m}\sum_{x \in \D_i}
    \norm{ \nabla_{\theta{}}{\ell{(\theta;x)}} - \nabla \loss{\left(\theta{}; \D_i\right)}}^2
    \leq \sigma^2.
\end{equation*}
\end{assumption}

Additionally, we also assume the point-wise gradients to be bounded, as usually done when analyzing differentially private ML algorithms to circumvent the complications due to clipping~\cite{agarwal2018cpsgd,noble2022differentially}.

\begin{assumption}[Bounded gradient]
\label{asp:bnd_norm}
There exists $C<\infty$ such that for all $\theta \in \mathbb{R}^d$, $i \in \H$, and $x \in \D_i$,
\begin{equation*}
    \norm{\nabla \ell(\theta{}; x)} \leq C.
\end{equation*}
\end{assumption}
\section{Lower Bound}
\label{sec:lb}
We now prove our lower bound on the error incurred by a $(f, \varrho)$-robust distributed algorithm, when ensuring $(\varepsilon,  \delta)$-DP. 
The main result is given in Theorem~\ref{th:finallb}, whose full proof is deferred to~\cref{app:final-lb}.
To give insights about the proof, we detail three separate cases in sections~\ref{sec:lb_noadversary},~\ref{sec:lb_no_privacy} and~\ref{sec:lb_adversarial}
where we respectively
study $f = 0$,
$f \geq 1$ but no privacy is enforced,
and the adversarial setting $f \geq 1$ with privacy.

\begin{restatable}{theorem}{finalLB}
\label{th:finallb}
Let $\mathcal{X} = \mathbb{R}^d$, $\ell = \norm{\cdot}^2$, $n \geq 3$, $0 \leq f < n/2$, $ m\geq 1$, and $\varepsilon, \delta \in (0,1)$. Consider arbitrary datasets $\mathcal{D}_1,\ldots,\mathcal{D}_n \in \mathcal{X}^m$ such that Assumption~\ref{asp:hetero} is satisfied with $G \geq 1$.
Let $\mathcal{A} : \mathcal{X}^{m \times n} \to \mathbb{R}^d$ be an $(\varepsilon,\delta)$-DP distributed algorithm.
Assume that $\varepsilon \leq 1/4\sqrt{2n\ln{(m+1)}}$, and that $2^{-m^{1-\gamma}}\leq n \delta \leq 1/8m^{1+\gamma}$ for some $\gamma \in (0,1)$.
For any 
$\varrho \leq \frac{f+1}{100(n-f)}$, if $\mathcal{A}$ is $(f,\varrho)$-robust, then
    \begin{equation*}
        \varrho = \widetilde{\Omega} \left(\frac{d}{\varepsilon^2 n m^2} + \frac{f}{n} \cdot \frac{1}{\varepsilon^2 m^2} + \frac{f}{n} \cdot G^2\right).
    \end{equation*}
\end{restatable}

\paragraph{Comparison with prior work.}
Our lower bound generalizes that of the non-adversarial centralized case. Specifically, specializing our lower bound to the case $n=1$ yields the bound $\Omega{\left(\frac{d}{\varepsilon^2 m^2}\right)}$, which corresponds to the lower bound from centralized private ERM~(Theorem~V.5, \citet{bassily2014private})\footnote{Notice that the loss function in \cite{bassily2014private} is not divided by the number of samples $m$.}.
Second, we improve over a result from the \emph{non-adversarial} private distributed learning literature~(Theorem~D.3, \citet{lowy2021private}), where a similar lower bound is shown.
While we consider distributed algorithm $\A$ as a black-box verifying $(\varepsilon,\delta)$-DP (as per \cref{def:distributed-dp}), the mentioned work imposes additional structure on $\A$ by assuming it to be round-based and to satisfy {\em compositionality}, which essentially abstracts the class of round-based algorithms whose DP guarantees can be computed from advanced composition.
Moreover, as the number of data points per worker $m$ is typically greater than the number of workers $n$, our condition $\varepsilon = \mathcal{O}{\left(\nicefrac{1}{\sqrt{n \log{m}}}\right)}$ is arguably weaker than $\varepsilon = \mathcal{O}{\left(\nicefrac{1}{m}\right)}$ in~\cite{lowy2021private}.

\paragraph{Discussion on assumptions.}
The assumptions on $\varepsilon, \delta, \varrho$ are only needed to use the lower bound from \cite{steinke2016between}, which additionally features the $\log{(1/\delta)}$ factor.
One could use the same proof technique as in \cite{bassily2014private} and remove these assumptions, at the expense of loosening the bound,
e.g. an additional $\log{m}$ factor in the denominator of the first term appears.

\subsection{Case I: Non-adversarial Setting}
\label{sec:lb_noadversary}
In this particular case, we assume all the workers to be honest, i.e., $f = 0$. However, the algorithm satisfies $(\varepsilon,  \delta)$-distributed DP. We show the following result.

\begin{restatable}{proposition}{dplbhd}
\label{prop:dp-lb-hd}
Let $n, m \geq 1$, and $\varepsilon, \delta \in (0,1)$.
Consider $\X = \{\pm \frac{1}{\sqrt{d}}\}^d$ and $\ell = \norm{\cdot}^2$.
Consider an arbitrary $(\varepsilon,\delta)$-DP distributed algorithm $\A : \X^{m \times n} \to \R^d$.
Assume that $\varepsilon \leq 1/4\sqrt{2n\ln{(m+1)}}$ and that $2^{-m^{1-\gamma}}\leq n \delta \leq 1/8m^{1+\gamma}$ for some $\gamma \in (0,1)$.
For any $\varrho \leq 1/100$, if $\A$ is $(0,\varrho)$-robust, then we must have
    \begin{equation*}
        \varrho = \Omega{\left(\frac{d}{\varepsilon^2 n m^2}\right)}.
    \end{equation*}
\end{restatable}
\begin{proof}[Sketch of proof]
We consider the quadratic loss function.
We derive a \emph{centralized} DP algorithm $\M$ from $\A$, and then reduce to private estimation of one-way marginals~\cite{steinke2016between}.
Algorithm $\M$ runs $\A$ on $n$ copies of the same dataset $\D \in \X^m$.
Thus, $\M$ inherits the error guarantee $\varrho$ from $\A$ on estimating the average of $\D$, but with a weaker $(\varepsilon_n,\delta_n)$-DP guarantee, due to the composition of 
$n$ adaptive $(\varepsilon,\delta)$-DP queries (since $\A$ can query each of the $n$ copies of $\D$ up to $(\varepsilon,\delta)$-DP budget). 
Using the centralized DP lower bound from \cite{steinke2016between}, we have $\varrho = \Omega{(d \log{(1/\delta_n)}/\varepsilon_n^2 m^2)}$.
We bound $\varepsilon_n$ and $\delta_n$ via {\em advanced composition}~\cite{dwork2014algorithmic} as follows: $\varepsilon_n = \mathcal{O}{(\varepsilon \sqrt{n \log{(1/\delta')}})}$ (provided that $\varepsilon$ is small enough) and $\delta_n \leq n\delta + \delta'$, where $\delta'$ is carefully chosen to ensure that $\log{(1/\delta_n)}/\log{(1/\delta')} = \Omega{(1)}$ (provided $\delta$ is small enough). Substituting the above values of $\varepsilon_n$ and $\delta_n$ in the above lower bound on $\varrho$ proves the proposition.
\end{proof}

\subsection{Case II: No Privacy}
\label{sec:lb_no_privacy}
Finally, we adapt the lower bound from robust distributed ML~\cite{karimireddy2022byzantinerobust} to our robustness definition~(\cref{def:resilience}) in \cref{prop:heterogeneity} below.
\begin{restatable}{proposition}{heterolb}
\label{prop:heterogeneity}
Let Assumption~\ref{asp:hetero} hold.
Let $n \geq 1$, $1 \leq f < n/2$, and $\kappa = \frac{16f(n-2f)}{(n-f)^2}$.
Consider $\X = \{\pm \frac{G}{\sqrt{\kappa d}}\}^d$ and $\ell = \norm{\cdot}^2$.
If a distributed algorithm is $(f,\varrho)$-robust,
then 
\begin{equation*}
    \varrho = \Omega{\left(\frac{f}{n}\cdot G^2 \right)}.
\end{equation*}
\end{restatable}

    \subsection{Case III: Adversarial Setting}
    \label{sec:lb_adversarial}
We now state, in \cref{prop:dp-lb-byz} below, the part of our bound where privacy and robustness are coupled.

\begin{restatable}{proposition}{dplbbyz}
\label{prop:dp-lb-byz}
Let $n \geq 3$, $1 \leq f < n/2$, $ m\geq 1$, $\varepsilon, \delta \in (0,1)$, and $\kappa = \frac{16f(n-2f)}{(n-f)^2}$.
Consider $\X = \{\pm \frac{1}{\sqrt{d}}\}^d \cup \{\pm \frac{1}{\sqrt{\kappa d}}\}^d$ and $\ell = \norm{\cdot}^2$.
Consider any $(\varepsilon,\delta)$-DP distributed algorithm $\A : \X^{m \times n} \to \R^d$.
Assume 
that $2^{-o{(m)}} \leq \delta \leq 1/m^{1+\Omega{(1)}}$.
For any $\varrho \leq \frac{f+1}{100(n-f)}$, if $\A$ is $(f,\varrho)$-robust, then we must have
    \begin{equation*}
        \varrho = \Omega{\left(\frac{f+1}{n-f} \cdot \frac{\log{(1/\delta)}}{\varepsilon^2 m^2}\right)}.
    \end{equation*}
\end{restatable}
\begin{proof}[Sketch of proof]
We 
consider the quadratic loss function, and reduce to the case $d=1$ with a careful choice of datasets.
We derive a \emph{centralized} DP algorithm $\M$ from $\A$, and then reduce to private estimation of one-way marginals~\cite{steinke2016between}.
Algorithm $\M$ runs $\A$ on input dataset $\D \in \X^m$ together with the remaining $n-1$ datasets crafted as follows: $f$ `adversarial' datasets are filled with $-1$, while $n-f-1$ 
`honest' datasets are filled with $+1$.
This ensures that, in all cases, $\M$ estimates the average of $\D$ better than at least an $f$-sized minority of datasets.
Therefore, as $\A$ guarantees error $\varrho$ on estimating the \emph{average of every group} of $n-f$ datasets' averages (by \cref{def:resilience}), we can bound the error of estimating the average of $\D$ by $\tilde{\varrho} = \Theta{(\frac{n-f}{f+1}\varrho)}$. We conclude by applying the aforementioned DP lower bound to $\M$, which is $(\varepsilon,\delta)$-DP and ensures error $\tilde{\varrho}$ in estimating the average of $\D$.
\end{proof}

\section{Our Algorithm: \algoname{}}
\label{sec:ub}

We prove in this section that our lower bound is tight. Specifically, we present a new distributed algorithm, \algoname{}, which yields a matching upper bound. Upon describing~\algoname{} in Section~\ref{sec:algo_describe}, we analyze its privacy in Section~\ref{sec:privacy_algo} and convergence guarantees in Section~\ref{sec:conv_algo} for smooth \emph{strongly convex} and \emph{non-convex} loss functions.

\subsection{Description of \algoname{}}
\label{sec:algo_describe}
Similar to DSGD, \algoname{} is an iterative algorithm where the server initiates each iteration (or step) $t \geq 0$ by broadcasting its current model parameter vector $\theta_{t}$ to all the workers. The initial parameter vector $\theta_0$ is chosen arbitrarily by the server.
Upon receiving $\theta_{t}$ from the server, each honest worker $w_i$ samples a mini-batch $S_t^{(i)}$ of $b \leq m$ data points randomly from its local dataset $\D_i$ \emph{without replacement}. Then, 
$w_i$ computes the gradients $\nabla \ell{\left( \theta_{t}; x \right)}$ for all $x \in S_t^{(i)}$, clips each of them using a threshold value $C$ and averages the clipped gradients to obtain a gradient estimate $g_t^{(i)}$. Specifically,
\begin{equation*}
    g_t^{(i)} = \frac{1}{b} \sum_{x \in S_t^{(i)}} \nabla \ell{\left( \theta_{t}; x \right)} \cdot \min \left\{ 1, \, \frac{C}{\norm{\nabla \ell{\left( \theta_{t}; x \right)}}} \right\}.
\end{equation*}

\begin{algorithm}[htb!]
\caption{\algoname{}}
\label{algo:robust-dpsgd}
\textbf{Initialization:} 
Initial model $\theta_0$, initial momentum $m_0^{(i)} = 0$ for each honest worker $w_i$, robust aggregation $F$, DP noise $\textcolor{teal}{\sigma_{\mathrm{DP}}}$, batch size $b$, clipping threshold $C$, learning rates $\{\gamma_t\}$, momentum coefficients $\{\beta_t\}$, and total number of steps $T$.

\begin{algorithmic}[1]
\FOR{$t=0 \dots T-1$}
\STATE \textcolor{violet}{\bf Server broadcasts} $\theta_{t}$ to all workers.
\FOR{\textbf{\textup{every}} \textcolor{teal}{\bf honest worker} $w_i, i \in \H,$ \textbf{\textup{in parallel}}}
\STATE Sample a mini-batch $S_t^{(i)}$ of size $b$ at random from $\D_i$ without replacement.
\STATE Clip and average the mini-batch gradients: 
  $$ g_t^{(i)} = \frac{1}{b} \sum_{x \in S_t^{(i)}} \textbf{Clip}\left(\nabla \ell{\left( \theta_{t}; x \right)} ; C \right),$$ 
where $\textbf{Clip}(g; C) \coloneqq g \cdot \min \left\{ 1, \, C/\norm{g} \right\}$.
\STATE Add noise to the mini-batch average gradient:
$$ \tilde{g}_{t}^{(i)} = g_t^{(i)} + \textcolor{teal}{\xi_t^{(i)} ; \, \, \,  \xi_t^{(i)} \sim \N{(0,\sigma_{\mathrm{DP}}^2 I_d)}}.$$

\STATE Send
$m_{t}^{(i)}= \beta_{t-1} m_{t-1}^{(i)} + (1-\beta_{t-1}) \tilde{g}_t^{(i)}$.
\ENDFOR
\STATE \textcolor{violet}{\bf Server aggregates}: $R_t = F{(m_{t}^{(1)},\dots,m_{t}^{(n)})}$.
\STATE \textcolor{violet}{\bf Server updates} the model:
$\theta_{t+1} = \theta_{t} - \gamma_t R_{t}$.
\ENDFOR
\STATE \textbf{return} $\hat{\theta}$ uniformly sampled from $\{\theta_0,\dots,\theta_{T-1}\}$.
\end{algorithmic}
\end{algorithm}

To protect the privacy of its data,
$w_i$ then obfuscates $g_t^{(i)}$ with Gaussian noise to obtain $\tilde{g}_t^{(i)}$, i.e.,
\begin{align*}
    \tilde{g}_t^{(i)} = g_t^{(i)} + \xi_t^{(i)} ~ ; \quad \xi_t^{(i)} \sim \mathcal{N}\left(0, \sigma_{\mathrm{DP}}^2 I_d\right),
\end{align*}
where $I_d$ denotes the identity matrix of dimension $d \times d$,
and $\mathcal{N}\left(0, \sigma_{\mathrm{DP}}^2 I_d\right)$ denotes a $d$-dimensional Gaussian distribution with mean $0$ and covariance $\sigma_{\mathrm{DP}}^2 I_d$. Finally, $w_i$ uses this noisy gradient to update its local {\em Polyak's momentum}~\cite{POLYAK19641} denoted by $m_t^{(i)}$, which is then sent to the server. Specifically, for $t \geq 1$, %
 \begin{equation*}
     m_t^{(i)} = \beta_{t-1} m_{t-1}^{(i)} + (1-\beta_{t-1}) \tilde{g}_t^{(i)},
 \end{equation*}
where $m_0^{(i)} = 0$ by convention, and $\beta_t \in [0, \, 1]$ is referred to as the momentum coefficient. 
Recall that if worker $w_i$ is adversarial, then it may send an arbitrary value for its momentum $m_t^{(i)}$. 
Upon receiving the local momentums from all the workers, the server aggregates them using $F$ to obtain $R_t = F(m_t^{(1)}, \ldots, m_t^{(n)}).$ Finally, the server updates the model $\theta_{t}$ to
\begin{equation*}
    \theta_{t+1} = \theta_{t} - \gamma_t \, R_t \label{eqn:model_update}
\end{equation*}
where $\gamma_t \geq 0$ is the learning rate at step $t$. The above procedure is repeated for a total of $T$ steps, after which the server outputs $\hat{\theta}$ which is sampled uniformly from the set $\{\theta_0,\ldots, \, \theta_{T-1} \}$.
The complete learning procedure 
is summarized in Algorithm~\ref{algo:robust-dpsgd}.

\subsection{Privacy of~\algoname{}}
\label{sec:privacy_algo}

We present below the DP guarantee of \algoname{}. 
To state closed-form expressions, we will assume that the batch size $b$ is sufficiently small compared to $m$ the number of data points per worker.
This assumption is only made for pedagogical reasons, but is not necessary for the privacy analysis to hold. In particular, the expressions that result from removing this assumption are difficult to read and interpret~\cite{wang2019subsampled}. We defer the full DP analysis without this assumption to Appendix~\ref{app:privacy}.

\begin{restatable}{theorem}{privacy}
\label{thm:privacy} 
Consider Algorithm~\ref{algo:robust-dpsgd}.
Let $\varepsilon > 0, \delta \in (0,1)$ be such that $\varepsilon \leq \log{(1/\delta)}$.
There exists a constant $k > 0$ such that, for a sufficiently small batch size $b$, when $\sigmadp \geq k \cdot \frac{2C}{b} \max{\{1, \,  \frac{b \sqrt{T \log{(1/\delta)}}}{m\varepsilon}\}}$, Algorithm~\ref{algo:robust-dpsgd} is $(\varepsilon,\delta)$-DP.
\end{restatable}

\subsection{Convergence of~\algoname{}}
\label{sec:conv_algo}

To present the convergence of~\algoname{} we first introduce below a criterion, namely {\em $(f,\kappa)$-robust averaging}, for an aggregation rule $F$ that proves crucial in our analysis.

\begin{restatable}{definition}{robustness}
\label{def:robust-averaging}
Let $n \geq 1$, $0 \leq f < n/2$ and $\kappa \geq 0$. 
An aggregation rule $\aggregation$ is said to be {\em $(f, \kappa)$-robust averaging} if for any vectors $x_1, \ldots, \, x_n \in \R^d$, and any set $S \subseteq \{1,\ldots,n\}$ of size $n-f$, the output $\hat{x} = \aggregation(x_1, \ldots, \, x_n)$ satisfies
\begin{align*}
    \norm{\hat{x} - \overline{x}_S}^2 \leq \kappa \cdot \lambda_{\max}{\left(\frac{1}{\card{S}} \sum_{i \in S}(x_i - \overline{x}_S)(x_i - \overline{x}_S)^\top\right)},
\end{align*}
where $\overline{x}_S \coloneqq \frac{1}{\card{S}} \sum_{i \in S} x_i$ and $\lambda_{\max}$ denotes the maximum eigenvalue. We refer to $\kappa$ as the robustness coefficient of $F$.
\end{restatable}

\paragraph{Comparison to prior work.}
Our robustness criterion is stronger than existing ones:
$(f,\kappa)$-robustness~\cite{allouah2023fixing},
$(f,\lambda)$-resilience~\cite{farhadkhani2022byzantine} and $(c,\delta_{\max})$-ARAgg~\cite{karimireddy2022byzantinerobust}.
The last two works bound the error with the diameter of honest inputs, i.e., maximum squared pairwise distance. The latter is greater than the empirical variance (bound used in $(f,\kappa)$-robustness~\cite{allouah2023fixing}), which itself is greater than the maximum eigenvalue of the empirical covariance (that we use) in high-dimensional spaces (i.e., $d > 1$).
In fact, the tight analysis of aggregation functions (e.g., trimmed mean, Krum) conducted in~\cite{allouah2023fixing} through the lens of $(f,\kappa)$-robustness directly implies our $(f,\kappa')$-robust averaging criterion, with $\kappa' \leq d\cdot \kappa$.
However, aggregation rules that are optimal w.r.t. $(f,\kappa)$-robustness~\cite{allouah2023fixing} may be suboptimal in our context, as we need to suppress the dimension dependence of $\kappa$ for our tight bounds.

{\bf Tighter heterogeneity metric.} We introduce a new metric $\gcov$ for quantifying the heterogeneity between the local gradients of honest workers' loss functions, which is arguably tighter than $G$ defined in Section~\ref{sec:lb_no_privacy}. Specifically,
$$\gcov^2 \coloneqq \sup_{\theta \in \R^d} \sup_{\norm{v}\leq 1} \frac{1}{\card{\H}}\sum_{i \in \H}\iprod{v}{\nabla{\loss{(\theta; \D_i)}} - \nabla{\loss_{\H}{(\theta)}}}^2.$$
Note that $\gcov^2$ above represents an upper bound on the spectral norm of the empirical covariance of honest gradients, which is smaller than their empirical variance $G^2$. Moreover, if the gradients have a well-conditioned empirical covariance, then $\gcov$ has weaker dependence on $d$.

We state our convergence result below in Theorem~\ref{thm:convergence}. Essentially, we analyze the convergence of \algoname{} with an $(f,\kappa)$-robust averaging aggregation $F$,
under assumptions~\ref{asp:bnd_var} and~\ref{asp:bnd_norm}, for smooth strongly convex and non-convex loss functions. We use the following notation:
\begin{align}
    &\loss_* = \inf_{\theta \in \R^d} \loss_{\H}(\theta), ~ 
    \loss_0 = \loss_{\H}{(\theta_0)} - \loss_*, ~
    a_1= 240,\nonumber\\ 
    &a_2 = 480, ~
    a_3 = 5760, ~
    \text{ and } ~  a_4 = 270. \label{eqn:a1a2}
\end{align}

\begin{restatable}{theorem}{convergence}
\label{thm:convergence}
Suppose that assumptions~\ref{asp:bnd_var} and \ref{asp:bnd_norm} hold true, and that $\loss_\H$ is $L$-smooth. Let $F$ satisfy the condition of $(f,\kappa)$-robust averaging. 
We let
\begin{align*}
    \sigmabar^2 = \frac{\sigma_b^2 + d \sigmadp^2}{n-f}
    +4 \kappa\left(\sigma_b^2+36\sigmadp^2 \left(1+\frac{d}{n-f} \right)\right),
\end{align*}
where $\sigma_b^2 = 2(1-\frac{b}{m})\frac{\sigma^2}{b}$. Consider \cref{algo:robust-dpsgd} with $T \geq 1$, the learning rates $\gamma_t$ and momentum coefficients $\beta_t$ specified below. We prove that the following holds, where the expectation $\expect{\cdot}$ is over the randomness of the algorithm.
\vspace{-1em}
\begin{enumerate}[leftmargin=*]
\setlength{\itemsep}{0.2em}
    \item \textbf{Strongly convex:} Assume that $\loss_\H$ is $\mu$-strongly convex.
    If $\gamma_t= \frac{10}{\mu(t+a_1 \frac{L}{\mu})}$ and $\beta_t = 1 - 24L \gamma_t$ then
    \begin{align*}
        \expect{\loss_{\H}{(\theta_{T})}-\loss_*}
    &\leq \frac{4a_1\kappa \gcov^2}{\mu}  +
    \frac{2 a_1^2 L \sigmabar^2}{\mu^2 T} +  \frac{2a_1^2 L^2 \loss_0}{\mu^2 T^2}.
    \end{align*}
    \item \textbf{Non-convex:} 
    If $\gamma = \min{\left\{\frac{1}{24L}, ~ \frac{\sqrt{a_4 \loss_0}}{2\sigmabar \sqrt{a_3 L T}}\right\}}$ and $\beta_t = 1 - 24L \gamma$ then 
\begin{align*}
   \expect{\|\nabla \loss_{\H} {(\hat{\weight{}})}\|^2} \hspace{-2pt}
    \leq  a_2 \kappa \gcov^2 +  \frac{\sqrt{a_3 a_4 L \loss_0} \sigmabar}{\sqrt{T}} + \frac{a_4 L \loss_0}{T}. 
\end{align*}
\end{enumerate}
\end{restatable}
\begin{proof}[Sketch of proof]
We show that at each step $t$, the descent $\loss_{\H}{(\theta_{t+1})}-\loss_{\H}{(\theta_{t})}$ can be bounded from above. Doing so is however non-trivial, as one needs to consider two conflicting effects: (i) the drift between honest momentums, and (ii) the deviation between the average honest momentum and the true gradient. To control this trade-off, we use increasing momentum coefficients and decreasing learning rates, and introduce an adapted \emph{Lyapunov function} $V_t$. Ignoring the constants, the function can be written as follows:
\begin{equation*}
    V_t \coloneqq (t+ K)^2 \cdot \expect{\loss_{\H}{(\theta_t)} - \loss_* + \frac{1}{L} \delta_t + \frac{\kappa}{L} \Delta_t},
\end{equation*}
where $\delta_t \coloneqq \norm{\overline{m}_t - \nabla \loss_{\H}{(\theta_t)}}^2$ represents the deviation of the momentum from the true gradient, $\Delta_t \coloneqq \lambda_{\max}{\left(\frac{1}{\card{\H}}\sum_{i \in \H}(m^{(i)}_t-\overline{m}_t)(m^{(i)}_t-\overline{m}_t)^\top \right)}$ represents the drift between the honest momentums, and $K \coloneqq \frac{L}{\mu}$ denotes the condition number of $\loss_\H$.
\end{proof}

\begin{remark}
Our strongly convex upper bound also holds true for the larger class of smooth {\em $\mu$-PL} functions~\cite{karimi2016linear}, which includes some non-convex functions.
\end{remark}

\paragraph{Comparison to prior work.}
Our convergence rate in $\mathcal{O} \left( \frac{1}{T} \right)$ for strongly convex losses is optimal in the non-adversarial and privacy-free setting~\cite{agarwal2009information}.
We improve over the state-of-the-art strongly convex analysis~\cite{data2021byzantine}, without privacy, which features a suboptimal excess term proportional to the stochastic noise $\sigmabar^2$. Essentially, we remove this dependency on $\sigmabar^2$ thanks to the use of momentum, although our convergence rate is in $\mathcal{O} \left( \frac{1}{T} \right)$ instead of being exponential as in~\cite{data2021byzantine}. In fact, making $\sigmabar^2$ vanish at a rate $\frac{1}{T}$ is crucial in our setting, as the DP noise $\sigmadp^2$ scales with $T$ (\cref{thm:privacy}). We also improve over the state-of-the-art non-convex analysis~\cite{farhadkhani2022byzantine}. Namely, our analysis features a tighter characterization of the data heterogeneity $\gcov$, instead of the traditional heterogeneity metric $G$. 

\section{Tight Upper Bound}
\label{sec:tight}
We present a new aggregation rule named {\em SMEA} (Smallest Maximum Eigenvalue Averaging) in Section~\ref{sec:smea_agg}, and show that it yields a tight upper bound in Section~\ref{sec:match_ub}.

\subsection{Robust Aggregation: SMEA}
\label{sec:smea_agg}

Consider a set of $n$ vectors $x_1, \ldots, \, x_n$. Let 
$S^{*}$ be an arbitrary subset of $[n]$ of size $n-f$ with the smallest empirical {\em maximum eigenvalue}, i.e., 
\begin{equation*}
    S^{*} \in \argmin_{\underset{\card{S} = n-f}{S \subseteq [n] }}\lambda_{\max}{\left(\frac{1}{\card{S}} \sum_{i \in S}(x_i - \overline{x}_S)(x_i - \overline{x}_S)^\top \right)}.
\end{equation*}
SMEA outputs the average of the inputs in $S^{*}$, i.e., 
\begin{equation*}
    \mathrm{SMEA}(x_1, \ldots, \, x_n) \coloneqq \frac{1}{\card{S^*}} \sum_{i \in S^{*}} x_i.
\end{equation*}
Note that SMEA draws inspiration from the {\em minimum diameter averaging} method~\cite{mhamdi18a}, which itself is reminiscent of the {\em minimal volume ellipsoid} method~\cite{rousseeuw1985multivariate}. We show that our aggregation rule satisfies the criterion of $(f,\kappa)$-robust averaging. 

\begin{restatable}{proposition}{smea}
\label{prop:smea}
Let $f < n/2$. SMEA is $(f,\kappa)$-robust averaging with
\vspace{-5pt}
\begin{equation*}
    \kappa = \frac{4f}{n-f}\left(1+\frac{f}{n-2f}\right)^2.
\end{equation*}
\end{restatable}
Proposition~\ref{prop:smea} implies that, when $n \geq (2+\eta)f$ for some constant $\eta >0$, SMEA satisfies $(f,\kappa)$-robust averaging with $\kappa = \mathcal{O}{\left(\nicefrac{f}{n}\right)}.$ Importantly, SMEA satisfies this high-dimensional robustness property while being agnostic to the statistical properties of the valid inputs, knowledge of which is key in designing efficient robust estimators~\cite{diakonikolas2017being,steinhardt2018resilience}~(see \cref{app:filter}).

\paragraph{Computational complexity.}
However, as SMEA involves computing the maximum eigenvalue of $d$-dimensional symmetric matrices, which is in $\mathcal{O}{\left( d^3 \right)}$, the worst-case computational complexity of SMEA is $\mathcal{O}{\big( {n \choose f} \cdot d^3 \big)}$, which is exponential in $f$. This shortcoming of our method should be addressed in the future.

\subsection{Upper Bound}
\label{sec:match_ub}

Upon combining the results in theorems~\ref{thm:privacy},~\ref{thm:convergence}, \cref{prop:smea}, and ignoring the vanishing terms in $T$, we obtain Corollary~\ref{cor:tradeoff} that quantifies the privacy-robustness-utility trade-off of \algoname{} using the SMEA aggregation rule.

\begin{restatable}{corollary}{tradeoff}
\label{cor:tradeoff}
Consider Algorithm~\ref{algo:robust-dpsgd} with aggregation $F = \mathrm{SMEA}$, under the strongly convex setting of Theorem~\ref{thm:convergence}.
Suppose that assumptions~\ref{asp:hetero}, \ref{asp:bnd_var}, \ref{asp:bnd_norm} hold, and that $n \geq (2+\eta)f$, for some absolute constant $\eta>0$.
Let $\varepsilon>0, \delta \in (0,1)$ be such that $\varepsilon \leq \log{(1/\delta)}$.
Then, there exists a constant $k>0$ such that, if $ \sigmadp = k \cdot \nicefrac{2C}{b} \max{\{1, \,  \nicefrac{b \sqrt{T \log{(1/\delta)}}}{\varepsilon m}\}}$, then Algorithm~\ref{algo:robust-dpsgd} is $(\varepsilon,\delta)$-DP and $(f, \varrho)$-robust where
    \begin{align*}
        \varrho = \mathcal{O} \left( \frac{d\,\log{(1/\delta)}}{\varepsilon^2 n m^2}+
        \frac{f}{n} \cdot \frac{\log{(1/\delta)}}{\varepsilon^2 m^2}+
        \frac{f}{n}G^2
    \right).
    \end{align*}
\end{restatable}
\paragraph{Tightness.}
Our upper bound is tight, in the sense that it matches the lower bound, up to the logarithmic factor $\log{(1/\delta)}$ in the first term. We believe that it is not possible to improve upon our upper bound in general, but rather that it may be possible to improve our lower bound in \cref{prop:dp-lb-hd}, by including the factor $\log{(1/\delta)}$. This could be done, for example, by assuming the stronger Rényi DP property~\cite{mironov2017renyi}, satisfied by the Gaussian mechanism, instead of relying on the advanced composition theorem.

\section{Conclusions and Future Work}
\label{sec:conclusion}

Applying machine learning in sensitive public domains requires algorithms that protect data privacy, while being robust to faults and adversarial behaviors.
We present the first tight analysis of the error incurred by any distributed ML algorithm ensuring robustness to adversarial workers and differential privacy for honest machines' data against any other curious entity. 
Our algorithm \algoname{} yields a tight upper bound for the class of smooth strongly convex problems, up to a logarithmic factor.
Proving a tighter lower bound on the privacy cost, featuring the usual $\log{(1/\delta)}$ factor, is an appealing goal.
Proving similar bounds for the non-strongly convex class is also of interest.
Also, in \cref{app:experiments}, we conduct small-scale experiments showing encouraging results using our aggregation rule SMEA (as well as other aggregation rules).
Yet, while SMEA is simple and agnostic to the statistical properties of honest data, it has a high computational complexity.
Deploying it on larger scale systems goes through designing variants with lower complexity, and this is also an interesting research direction.

\section*{Acknowledgements}
This work was supported in part by SNSF grants 200021\_200477 and 200021\_182542, and an EPFL-Ecocloud postdoctoral grant.
The authors are thankful to the anonymous reviewers for their constructive comments.
\bibliography{references}
\bibliographystyle{icml2023}

\newpage
\onecolumn
\appendix

\section*{Organization of the Appendix}
\cref{app:lb} contains the proof of our lower bounds.
\cref{app:robustness} contains proofs of claims related to $(f,\kappa)$-robust averaging and SMEA.
\cref{app:privacy} contains the privacy analysis of \algoname{}.
\cref{app:ub} contains the convergence analysis of \algoname{}.
\cref{app:experiments} contains the experimental setup and results of our empirical evaluation.

\section{Lower Bounds}
\label{app:lb}

In Section~\ref{app:lb-centralized}, we recall lower bounds on centralized private algorithms.
We then extend these results to distributed private algorithms.
We start by the lower bound due to privacy alone in Section~\ref{app:lb-privacy}.
Next, we show the lower bound due to robustness alone in Section~\ref{app:lb-robustness}.
We then show the lower bound due to the privacy-robustness tradeoff in Section~\ref{app:lb-privacy-robustness}.
Finally, we merge the previous results to show the final lower bound in Section~\ref{app:final-lb}.

\subsection{Lower Bounds in Centralized DP}
\label{app:lb-centralized}

We recall lower bounds~\cite{steinke2016between} on the error incurred by \emph{centralized} differentially private mechanisms for estimating $d$-dimensional one-way marginals; i.e., the average of rows of a dataset. Recall that \citeauthor{steinke2016between} prove a sharper bound (by factor $\log{(1/\delta)}$) than \citeauthor{bassily2014private}, whose work is based on lower bounds using fingerprinting codes~\cite{bun2014fingerprinting}.
We recall below the main lower bound from \cite{steinke2016between}.

\begin{lemma}[Theorem~1.1, \citet{steinke2016between}]
\label{lemma:dp-lb}
    Let $m,d \geq 1, \varepsilon, \delta \in (0,1)$ and $\X = \{\pm 1\}^d, \Y = [\pm 1]^d$.
    Consider any $(\varepsilon,\delta)$-DP centralized algorithm $\M : \X \to \Y$.
    Assume that $\delta \leq 1/m^{1+\Omega{(1)}}$ and that $\delta \geq 2^{-o{(m)}}$.
    Let $\D \in \X^m$ and $\overline{\D}$ denote the average of records of $\D$.
    For any $\varrho \leq 1/10$ such that for every $\D \in \X^m$, $\expect{\norm{\M{(\D)} - \overline{\D}}_1} \leq d \varrho$, we have:
    \begin{equation*}
        m = \Omega{\left( \frac{\sqrt{d \log{(1/\delta)}}}{\varepsilon \varrho} \right)}.
    \end{equation*}
\end{lemma}

Observe in Lemma~\ref{lemma:dp-lb} that the lower bound assumption $\delta \leq 1/m^{1+\Omega{(1)}}$ is slightly more restrictive than the folklore assumption $\delta = o{(1/m)}$~\cite{dwork2014algorithmic}. The latter ensures that $(\varepsilon,\delta)$-DP precludes some intuitively non-private algorithms, e.g., when $\delta \geq 1/m$, the algorithm that returns $\floor{m \delta}$ random elements of the dataset is $(0,\delta)$-DP.

\subsection{Case I: Non-adversarial Setting}
\label{app:lb-privacy}

We prove below our lower bound due to privacy, stated in \cref{prop:dp-lb-hd}.
\dplbhd*
\begin{proof}
Let $n, m, d \geq 1$, $\varepsilon, \delta \in (0,1)$, and $\varrho \leq 1/100$. 
Consider $\X = \left\{\pm 1/\sqrt{d}\right\}^d$ and $\ell = \norm{\cdot}^2$.
We consider an arbitrary  distributed algorithm $\A : \X^{m \times n} \to \R^d$ that satisfies $(\varepsilon,\delta)$-distributed DP (see Definition~\ref{def:distributed-dp}), and $(0, \, \varrho)$-robustness (see Definition~\ref{def:resilience}).
We assume that $\varepsilon \leq 1/4\sqrt{2n\ln{(m+1)}}$ and that $2^{-m^{1-\gamma}}\leq n \delta \leq 1/8m^{1+\gamma}$ for some $\gamma \in (0,1)$. 

{\bf Proof outline.} We consider the centralized algorithm $\M$ which takes as input dataset $\D \in \X^m$ and executes $\A{(\D_1,\ldots,\D_n)}$ on $n$ copies of $\D$, i.e., $\D_1=\ldots=\D_n = \D$. Then, we derive the DP guarantee and utility of $\M$ using the facts that $\A$ satisfies $(\varepsilon, \delta)$-distributed DP (see Definition~\ref{def:distributed-dp}) and $(0, \, \varrho)$-robustness, respectively. 
Finally, we apply the \emph{centralized} DP lower bound on $\M$ (stated in Lemma~\ref{lemma:dp-lb}) to conclude the proof.

\textbf{Privacy guarantee of $\M$.}
We first analyze the DP guarantees of $\M$ inherited from $\A$.

Recall from \cref{def:distributed-dp} that, since $\A$ is $(\varepsilon,\delta)$-DP, it can communicate with \emph{each} database $\D_i$ subject to $(\varepsilon,\delta)$-DP.
Thus, when running $\M$, in the worst case, algorithm $\A$ may adaptively query the same database $\D$ a total of $n$ times, subject to $(\varepsilon,\delta)$-DP budget for each query.
Therefore, $\M$ is $(\varepsilon_n,\delta_n)$-DP where $(\varepsilon_n,\delta_n)$ is the privacy guarantee resulting from composing $(\varepsilon,\delta)$-DP across $n$ adaptive queries.
Thanks to the advanced composition theorem~\cite{dwork2014algorithmic}, we obtain that, for any $\delta' \in (0,1)$,
\begin{equation}
    \varepsilon_n = \varepsilon\sqrt{2n \ln{(1/\delta')}} + n \varepsilon (e^\varepsilon - 1), ~~
    \delta_n = n \delta + \delta'.
    \label{eq:advanced-composition}
\end{equation}
As $\varepsilon \in (0,1)$, we have $e^\varepsilon - 1 \leq 2 \varepsilon$ and thus 
\begin{equation}
    \varepsilon_n \leq \varepsilon\sqrt{2n \ln{(1/\delta')}} + 2n \varepsilon^2.
    \label{eq:epsilon_n}
\end{equation}

We now set $\delta'$ as follows:
\begin{equation}
    \delta' = \frac{1}{(m+1)^{1+\gamma}} \in (0,1).
    \label{eq:delta'}
\end{equation}

We verify below the privacy conditions on $\M$ of \cref{lemma:dp-lb}. We first prove that $\ln{(1/\delta')} \in [n\varepsilon^2, 1/16n\varepsilon^2)$,
and then that $\varepsilon_n \leq 4 \varepsilon \sqrt{n \ln{(1/\delta')}} < 1$.

\underline{\textit{Bound on $\ln{(1/\delta')}$:}}
Since we assume $\varepsilon \leq 1/4\sqrt{2n\ln{(m+1)}}$ (with $m \geq 1$), we have 
$$n \varepsilon^2 \leq 1/16 \leq 1/16n\varepsilon^2.$$
On the other hand, as $m \geq 1$, it follows from the expression~\eqref{eq:delta'} of $\delta'$ that $1/\delta' \geq 2$ and $\ln{(1/\delta')} \geq 1/4 \geq n\varepsilon^2$.

Also, since $\varepsilon \leq 1/4\sqrt{2n\ln{(m+1)}}$ we have $\ln{(m+1)} \leq 1/32n\varepsilon^2$, and thus (because $\gamma \in (0,1)$) we have 
$$\ln{(1/\delta')} = (1+\gamma)\ln{(m+1)} < 2\ln{(m+1)} \leq 1/16n\varepsilon^2.$$
This proves that
\begin{equation}
\ln{(1/\delta')} \in [n\varepsilon^2, 1/16n\varepsilon^2).
    \label{eq:claim1}
\end{equation}

\underline{\textit{Bound on $\varepsilon_n$:}}
Thanks to \eqref{eq:claim1}, we have $\ln{(1/\delta')} \geq n \varepsilon^2$. 
Thus, by taking square roots we have $\varepsilon \sqrt{n} \leq \sqrt{\ln{(1/\delta')}}$.

\vspace{+1pt}
Therefore, $n \varepsilon^2 \leq \varepsilon\sqrt{n \ln{(1/\delta')}}$. Then, using the bound on $\varepsilon_n$ in \eqref{eq:epsilon_n}, we obtain
$$\varepsilon_n \leq \varepsilon\sqrt{2n \ln{(1/\delta')}} + 2n \varepsilon^2 \leq \varepsilon\sqrt{2n \ln{(1/\delta')}} + 2\varepsilon \sqrt{n \ln{(1/\delta')}} \leq 4 \varepsilon \sqrt{n \ln{(1/\delta')}}.$$
On the other hand, since we showed in \eqref{eq:claim1} that $\ln{(1/\delta')} < 1/16n\varepsilon^2$, we have $4 \varepsilon \sqrt{n \ln{(1/\delta')}} < 1$. This proves that
\begin{equation}
\varepsilon_n \leq 4 \varepsilon \sqrt{n \ln{(1/\delta')}} < 1.
    \label{eq:claim2}
\end{equation}

From \eqref{eq:claim2}, we have $\varepsilon_n \in (0,1)$.
From \eqref{eq:advanced-composition}, we have $\delta_n = n \delta + \delta'$.
Thus, by assumption on $\delta$ and \eqref{eq:delta'}, the parameter $\delta_n$ satisfies both $\delta_n \geq n \delta \geq 2^{-m^{1-\gamma}} = 2^{-o(m)}$ and $\delta_n = n\delta + \delta' \leq 1/8m^{1+\gamma} + 1/(m+1)^{1+\gamma} = 1/m^{1+\Omega{(1)}}$. 

\textbf{Utility guarantees of $\M$.}
We now analyze the utility guarantees of $\M$, inherited from $\A$.

Let 
$\D \in \X^m$ be an arbitrary set of $m$ points from the specified space $\X = \left\{\pm 1/\sqrt{d}\right\}^d$.
Recall that $\A$ is assumed $(0,\varrho)$-robust.
By \cref{def:resilience}, for any $\D_1, \ldots, \D_n \in \X^m$, the output $\hat{\theta} = \A{(\D_1,\ldots,\D_n)}$ verifies
    \begin{equation}
        \label{eq4}
        \varrho \geq \expect{\loss{(\hat{\theta}; \D_1,\ldots,\D_n)}-\inf_{\theta \in \R^d} \loss{(\theta; \D_1,\ldots,\D_n)}},
    \end{equation}

In this particular case, since $\D_1, \ldots, \D_n = \D$ and $\ell{(\theta;x)} \coloneqq \norm{\theta - x}^2$, we have for all $\theta \in \R^d$,
\begin{equation}
    \loss{(\theta; \D_1,\ldots,\D_n)} 
    =\frac{1}{nm} \sum_{i = 1}^n \sum_{x \in \D_i} \norm{\theta-x}^2 
    = \frac{1}{m} \sum_{x \in \D} \norm{\theta-x}^2 
    = \loss{(\theta; \D)}.
    \label{eq:loss-expression}
\end{equation}
We can rewrite the above upon applying the bias-variance decomposition: for any $x_1, \ldots, x_n$ we have $\frac{1}{n} \sum_{i=1}^n \norm{x_i - \overline{x}}^2 = \frac{1}{n} \sum_{i=1}^n \norm{x_i}^2 - \norm{\overline{x}}^2$ where $\overline{x} = \frac{1}{n} \sum_{i=1}^n x_i$.
Thus, denoting $\overline{\D} \coloneqq \frac{1}{m} \sum_{x \in \D} x$, we can rewrite \eqref{eq:loss-expression} as
\begin{equation}
    \loss{(\theta; \D_1,\ldots,\D_n)} 
    =\loss{(\theta; \D)}
    = \norm{\theta - \overline{\D}}^2 + \frac{1}{m} \sum_{x \in \D} \norm{\overline{\D}-x}^2.
    \label{eq:simple-loss-expression}
\end{equation}
This loss is minimized at $\theta = \overline{\D}$, and the minimum value $\loss_* \coloneqq \frac{1}{m} \sum_{x \in \D} \norm{\overline{\D}-x}^2$.
Thus, substituting the expression of $\loss$ from \eqref{eq:simple-loss-expression} in \eqref{eq4}, we obtain that
\begin{align*}
    \varrho 
    &\geq \expect{\loss{(\hat{\theta}; \D)} - \loss_* }
    = \expect{\norm{\hat{\theta} - \overline{\D}}^2}.
\end{align*}
Note that by construction of $\M$, we have $\M{(\D)} = \A{(\D,\ldots,\D)} = \hat{\theta}$. Thus, from above we obtain that 
\begin{align*}
    \varrho \geq \expect{\norm{\M{(\D)} - \overline{\D}}^2}.
\end{align*}
Thus, as $\norm{\cdot}_1 \leq \sqrt{d} \norm{\cdot}$, by taking square roots above, applying Jensen's inequality and multiplying by $d$, we obtain that
\begin{equation}
    \label{eq1}
    d \sqrt{\varrho}
    \geq d \sqrt{\expect{\norm{\M{(\D)}  - \overline{\D}}^2}}
    \geq d\, \expect{\norm{\M{(\D)} - \overline{\D}}}
    \geq \sqrt{d} \, \expect{\norm{\M{(\D)}  - \overline{\D}}_1}
    = \expect{\norm{\sqrt{d} \cdot \M{(\D)}  - \sqrt{d} \cdot \overline{\D}}_1}.
\end{equation}
Recall that $\X = \{\pm 1/\sqrt{d}\}^d$.
As in Theorem~5.2 of \cite{steinke2016between}, we define a mechanism $\M':\left\{\pm 1\right\}^{d \times m} \to [\pm 1]^{d}$ as follows: on input $\D' \subseteq \left\{\pm 1\right\}^{d \times m}$ let $\D = \D' / \sqrt{d} \in \X^m$, return $\sqrt{d}\cdot \M{(\D)} $ truncated to $[\pm 1]^d$.
Thus, by \eqref{eq1}, mechanism $\M'$ verifies for all $\D' \subseteq \left\{\pm 1\right\}^{d \times m}$ that
\begin{equation}
    d \sqrt{\varrho} \geq \expect{\norm{\M'{(\D)} 
 - \overline{\D'}}_1}.
 \label{eq:m'}
\end{equation}

\textbf{Invoking \cref{lemma:dp-lb}.}
Note that $\M'$, similar to $\M$, is also $(\varepsilon_n, \delta_n)$-DP by the argument of {\em post-processing}. Recall that we have shown earlier that $\varepsilon_n, \delta_n$ satisfy the conditions of \cref{lemma:dp-lb}. Since $\varrho \leq 1/100$, we also have $\sqrt{\varrho} \leq 1/10$.
Therefore, upon applying \cref{lemma:dp-lb} to $\M'$, in conjunction with \eqref{eq:m'}, we deduce that 
\begin{equation*}
    m = \Omega{\left(\frac{\sqrt{d \log{(1/\delta_n)}}}{\varepsilon_n \sqrt{\varrho}}\right)}.
\end{equation*}
By rearranging terms above and taking squares, we obtain that
\begin{equation}
    \label{eq2}
    \varrho = \Omega{\left(\frac{d \log{(1/\delta_n)}}{\varepsilon_n^2 m^2}\right)}.
\end{equation}
Recall that we have already shown in \eqref{eq:claim2} and \eqref{eq:advanced-composition}, respectively, that $\varepsilon_n \leq 4 \varepsilon \sqrt{n \ln{(1/\delta')}}$ and $\delta_n = n\delta + \delta'$, where $\delta' = 1/(m+1)^{1+\gamma}$ (defined in \eqref{eq:delta'}).
Therefore,~\eqref{eq2} yields
\begin{equation}
    \label{eq3}
    \varrho = \Omega{\left(\frac{d \log{(1/(n\delta + \delta'))}}{\varepsilon^2 n m^2 \log{(1/\delta')}}\right)}.
\end{equation}
As $\ln(1+x) \leq x$, 
substituting $\delta'$ from \eqref{eq:delta'}, 
and using the assumption that $\delta \leq 1/8nm^{1+\gamma}$, $\gamma \in (0,1), m \geq 1$, we obtain that
\begin{align*}
    \frac{\ln{(1/(n\delta + \delta'))}}{\ln{(1/\delta')}}
    &= \frac{\ln{(1/\delta'(1+n\delta/\delta'))}}{\ln{(1/\delta')}}
    = 1 + \frac{\ln{(1/(1+n\delta/\delta'))}}{\ln{(1/\delta')}}\\
    &= 1 - \frac{\ln{(1+n\delta/\delta')}}{\ln{(1/\delta')}}
    \geq 1 - \frac{n\delta}{\delta'\ln{(1/\delta')}}
    = 1 - \frac{n\delta (m+1)^{\gamma+1}}{(1+\gamma)\ln{(m+1)}}\\
    &\geq 1- \frac{(m+1)^{\gamma+1}}{8(1+\gamma)m^{1+\gamma}\ln{(m+1)}}
    \geq 1- \frac{(2m)^{\gamma+1}}{8(1+\gamma)m^{1+\gamma}\ln{(m+1)}}\\
    &=1- \frac{2^{\gamma+1}}{8(1+\gamma)\ln{(m+1)}}
    \geq 1- \frac{4}{8\ln{(m+1)}} \geq 1 - \frac{1}{2 \ln{(2)}} 
    = \Omega{(1)}.
\end{align*}
Finally, substituting from above in~\cref{eq3} proves the desired result, i.e., 
\begin{equation*}
    \varrho = \Omega{\left(\frac{d}{\varepsilon^2 n m^2}\right)}.
\end{equation*}
\end{proof}

\subsection{Case II: No Privacy}
\label{app:lb-robustness}

We prove below the lower bound due to robustness stated in Proposition~\ref{prop:heterogeneity}.
\heterolb*
\begin{proof}
The proof is similar to that of Theorem~III~\cite{karimireddy2022byzantinerobust}.
Let $n \geq 1$, $1 \leq f < n/2$, $\kappa = \frac{16f(n-2f)}{(n-f)^2}$, and $G > 0$. 
Consider $\X = \{\pm \frac{G}{\sqrt{\kappa d}}\}^d$ and $\ell = \norm{\cdot}^2$.
Let Assumption~\ref{asp:hetero} hold.
Assume that algorithm $\mathcal{A}$ is $(f,\varrho)$-robust.

Denote by $x = \frac{G}{\sqrt{\kappa d}} \cdot \ones \in \R^d$, where $\ones \in \R^d$ is the vector of ones.
Consider the following datasets $\D_1 = \ldots = \D_{n-f} = \{x\}^m$ (i.e. all rows are $x$) and $\D_{n-f+1} = \ldots = \D_n = \{-x\}^m$ (i.e. all rows are $-x$).
Consider the two situations of honest identities $\H_1 = \{1,\ldots,n-f\}$ and $\H_2 = \{f+1,\ldots,n\}$.

We first show that the loss functions $\loss{(\cdot~;\D_1)}, \ldots, \loss{(\cdot~;\D_n)}$ (defined using $\ell$ in \cref{sec:problem}) satisfy Assumption~\ref{asp:hetero} in both situations.
This is straightforward in situation $\H_1$ since honest losses are identical.
In situation $\H_2$, we have for all $\theta \in \R^d$, 
\begin{align*}
  \nabla{\loss_{\H_2}{(\theta)}} = \frac{1}{n-f}\sum_{i \in \H_2} \nabla{\loss{(\theta;\D_i)}} = \frac{n-2f}{n-f} 2(\theta-x) + \frac{f}{n-f}2(\theta + x) = 2\left(\theta - \frac{n-3f}{n-f}x \right).
\end{align*}
Observe that, as $n > 2f$, the intersection $\H_1 \cap \H_2 = \{f+1,\ldots,n-f\}$ is non-empty.
Therefore, thanks to the choice of $x$, we now show that Assumption~\ref{asp:hetero} holds, as for all $\theta \in \R^d$ we have
\begin{align*}
    \frac{1}{\card{\H_2}} \sum_{i \in \H_2} \norm{\nabla{\loss{(\theta;\D_i)}} - \nabla{\loss_{\H_2}{(\theta)}}}^2
    &= \frac{\card{\H_1 \cap \H_2}}{n-f}\norm{\nabla{\loss{(\theta;\D_{f+1})}}-\nabla{\loss_{\H_2}{(\theta)}}}^2 \\
    &\quad+ \frac{\card{\H_2 \setminus \H_1}}{n-f}\norm{\nabla{\loss{(\theta;\D_{n})}}-\nabla{\loss_{\H_2}{(\theta)}}}^2\nonumber\\
    &= \frac{n-2f}{n-f}\norm{2(\theta-x)-2(\theta - \frac{n-3f}{n-f}x)}^2\\ 
    &\quad + \frac{f}{n-f}\norm{2(\theta+x)-2(\theta - \frac{n-3f}{n-f}x)}^2\nonumber\\
    &= \frac{4(n-2f)}{n-f}\norm{\frac{-2f}{n-f}x}^2 + \frac{4f}{n-f}\norm{\frac{2(n-2f)}{n-f}x}^2
    = \frac{16f(n-2f)}{(n-f)^2}\norm{x}^2 \\
    &= \kappa \norm{x}^2
    = G^2.
\end{align*}
Now, denote $\loss_{*,\H_1} \coloneqq \inf_{\R^d} \loss_{\H_1}$ and $\loss_{*, \H_2} \coloneqq \inf_{\R^d} \loss_{\H_2}$. 
Since learning algorithm $\mathcal{A}$ is $(f,\varrho)$-robust, 
it outputs $\hat{\theta}$ such that $\expect{\loss_{\H_1}{(\hat{\theta})}-\loss_{*,\H_1}} \leq \varrho$ and $\expect{\loss_{\H_2}{(\hat{\theta})}-\loss_{*,\H_2}} \leq \varrho$.
Note that situations $\H_1$ and $\H_2$ are indistinguishable to algorithm $\mathcal{A}$ because it ignores the honest identities, and thus $\hat{\theta}$ is the same in both situations.

Recall that the expression of loss $\loss_{\H_1}$ is
\begin{equation*}
    \loss_{\H_1} = \frac{1}{\card{\H_1}} \sum_{i \in \H_1} \loss{(\theta;\D_i)}
    = \frac{1}{\card{\H_1}} \sum_{i \in \H_1} \norm{\theta-x}^2 = \norm{\theta-x}^2.
\end{equation*}
Therefore, the loss is minimized at $\theta = x$ and we have $\loss_{*,\H_1} = \loss_{\H_1}(x) = 0$. Thus, we have
\begin{equation*}
    \expect{\loss_{\H_1}{(\hat{\theta})}-\loss_{*,\H_1}} = \expect{\norm{\hat{\theta}-x}^2}.
\end{equation*}
On the other hand, after some algebraic manipulations, the expression of loss $\loss_{\H_2}$ is 
\begin{align*}
    \loss_{\H_2}{(\theta)} &= \frac{1}{\card{\H_2}} \sum_{i \in \H_2} \loss{(\theta;\D_i)}
    = \frac{\card{\H_1 \cap \H_2}}{n-f} \cdot \norm{\theta-x}^2 + \frac{\card{\H_2 \setminus \H_1}}{n-f} \cdot \norm{\theta+x}^2\\
    &= \frac{n-2f}{n-f} \cdot (\norm{\theta}^2+\norm{x}^2 - 2 \iprod{\theta}{x}) + \frac{f}{n-f} \cdot (\norm{\theta}^2+\norm{x}^2 + 2 \iprod{\theta}{x})\\
    &= \norm{\theta - \frac{n-3f}{n-f}x}^2 + \kappa \norm{x}^2.
\end{align*}
Therefore, the loss is minimized at $\theta = \frac{n-3f}{n-f}x$ and we have $\loss_{*,\H_2} = \kappa \norm{x}^2$. Thus, we obtain
\begin{equation*}
    \expect{\loss_{\H_2}{(\hat{\theta})}-\loss_{*,\H_2}} = \expect{\norm{\hat{\theta} - \frac{n-3f}{n-f}x}^2}.
\end{equation*}

Recall that $\kappa = \frac{16f(n-2f)}{(n-f)^2}$.
Therefore, invoking Jensen's inequality, we have
\begin{align}
\varrho 
&\geq \max{\left\{\expect{\loss_{\H_1}{(\hat{\theta})}-\loss_{*,\H_1}},\expect{\loss_{\H_2}{(\hat{\theta})}-\loss_{*,\H_2}}\right\}}
\geq \frac{1}{2} \left(\expect{\loss_{\H_1}{(\hat{\theta})}-\loss_{*,\H_1}}+\expect{\loss_{\H_2}{(\hat{\theta})}-\loss_{*,\H_2}} \right)\nonumber\\
&= \frac{1}{2} \left(\norm{\hat{\theta}-x}^2 + \norm{\hat{\theta} - \frac{n-3f}{n-f}x}^2 \right)
\geq \frac{1}{4} \norm{\frac{2f}{n-f}x}^2 
= \left(\frac{f}{n-f}\right)^2 \frac{G^2}{\kappa}
= \frac{1}{16} \cdot \frac{f}{n-2f} G^2.
\label{eq:lb_noprivacy_last}
\end{align}
Since $n -2f \leq n$, we obtain $\varrho \geq \frac{1}{16} \cdot \frac{f}{n}\, G^2$, which concludes the proof.
\end{proof}

\subsection{Case III: Adversarial Setting}
\label{app:lb-privacy-robustness}

We show below the lower bound from \cref{prop:dp-lb-byz} due to the privacy-robustness tradeoff. 

\dplbbyz*
\begin{proof}

Let $n \geq 3$, $1 \leq f < n/2$, $ m\geq 1$, $d \geq 1$, $\varepsilon, \delta \in (0,1)$, $\kappa = \frac{16f(n-2f)}{(n-f)^2}$, 
and $\varrho \leq \frac{f+1}{100(n-f)}$.
Consider $\X = \{\pm 1/\sqrt{d}\}^d \cup \{\pm 1/\sqrt{\kappa d}\}^d$ and $\ell = \norm{\cdot}^2$.
We consider a distributed algorithm $\A : \X^{m \times n} \to \R^d$ that satisfies $(\varepsilon,\delta)$-distributed DP where $2^{-o{(m)}} \leq \delta \leq 1/m^{1+\Omega{(1)}}$ and $(f, \varrho)$-robustness. 

We consider the following datasets. Let $\ones$ denote the vector of ones in $\R^d$. For $i \in \{2, \ldots, n-f\}$, we set $$\D_i = \D_+ \coloneqq \{+\frac{1}{\sqrt{d}} \cdot \ones \}^m,$$ i.e., all rows are $+\frac{1}{\sqrt{d}} \cdot \ones \in \R^d$. For $i \in \{n-f+1, \ldots, n\}$ we set $$\D_i = \D_- \coloneqq \{-\frac{1}{\sqrt{d}} \cdot \ones\}^m,$$ i.e., all rows are $-\frac{1}{\sqrt{d}} \cdot \ones \in \R^d$.
Finally, we fix $\D_1 \in \X^m$ to be an arbitrary dataset with every element having identical coordinates. That is, for arbitrary $\alpha_{1,1}, \ldots, \alpha_{1,m} \in \{\pm 1\}$, we set
$$\D_1 = \left\{\frac{\alpha_{1,1}}{\sqrt{d}} \cdot \ones, \ldots, \frac{\alpha_{1,m}}{\sqrt{d}} \cdot \ones \right\}.$$

{\bf Proof outline.} We consider the centralized algorithm $\M : \X^{m} \to \R^d$ which takes as input dataset $\D_1 \in \X^m$ and executes $\A{(\D_1, \D_2, \ldots,\D_n)}$, where the datasets $\D_2, \ldots,\D_n$ are fixed above. 
We first derive the DP and utility guarantees $\M$ inherits from $\A$, which satisfies $(\varepsilon, \delta)$-distributed DP (see Definition~\ref{def:distributed-dp}) and $(f, \varrho)$-robustness, and then conclude the proof by applying the centralized lower bound \cref{lemma:dp-lb} to $\M$.

\textbf{Privacy guarantees of $\M$.}
We first state the privacy guarantees of $\M$ inherited from $\A$.

As per Definition~\ref{def:distributed-dp}, since $\A$ is $(\varepsilon,\delta)$-DP, all communications with worker $w_1$ (whose dataset is $\D_1$) are $(\varepsilon,\delta)$-DP. 
It follows directly that $\M$ is $(\varepsilon,\delta)$-DP by post-processing.

\textbf{Utility guarantees of $\M$.}
We now analyze the utility guarantees of $\M$ inherited from $\A$.

Since $\A$ is $(f,\varrho)$-robust (\cref{def:resilience}), the output $\hat{\theta} = \M{(\D_1)} = \A{(\D_1,\ldots,\D_n)}$ verifies
    \begin{equation}
        \label{eq5}
        \varrho \geq \expect{\loss_{\H}{(\hat{\theta})}-\loss_*},
    \end{equation}
for any set of honest identities $\H \subseteq \{1, \ldots, n\}, \card{\H}=n-f$, where we denote $\loss_{*} \coloneqq \inf_{\R} \loss_{\H}$.

\begin{tcolorbox}[colback=gray!10,boxrule=0pt,breakable,enhanced]
\underline{\textit{Reduction to one-dimensional space:}}
We now show that we can simply consider $d=1$, without loss of generality.
For this, we develop the RHS of \eqref{eq5}. We have for any $\theta \in \R^d$ and $\H \subseteq \{1, \ldots, n\}, \card{\H}=n-f$:
\begin{align}
    \loss_{\H}{(\theta)}
    &= \frac{1}{\card{\H}} \sum_{i \in \H} \frac{1}{m} \sum_{x \in \D_i} \norm{\theta-x}^2.
    \label{eq:loss-full-0}
\end{align}
The above function is minimized at $\theta^*_{\H} \coloneqq \frac{1}{\card{\H}} \sum_{i \in \H}\overline{\D}_i$ the average of one-way marginals $\overline{\D}_i \coloneqq \frac{1}{m} \sum_{x \in \D_i}x$. Therefore, the minimum of $\loss_\H$ is $\loss_{*,\H} \coloneqq \loss_{\H}{(\theta^*_\H)}$.\\
Recall the following bias-variance decomposition: for any $x_1, \ldots, x_n \in \R^d$ we have $\frac{1}{n} \sum_{i=1}^n \norm{x_i - \overline{x}}^2 = \frac{1}{n} \sum_{i=1}^n \norm{x_i}^2 - \norm{\overline{x}}^2$, where we denoted $\overline{x} \coloneqq \frac{1}{n} \sum_{i=1}^n x_i$.
Therefore, recalling \eqref{eq:loss-full-0} and $\theta^*_{\H} = \frac{1}{\card{\H}} \sum_{i \in \H}\overline{\D}_i$, we have
\begin{align}
    \loss_{\H}{(\theta)} - \loss_{*,\H}
    &= \loss_{\H}{(\theta)} - \loss_{\H}{(\frac{1}{\card{\H}} \sum_{i \in \H}\overline{\D}_i)}
    = \norm{\theta - \frac{1}{\card{\H}} \sum_{i \in \H}\overline{\D}_i}^2.
    \label{eq:loss-full-1}
\end{align}
Recall our setting of datasets in the beginning of the proof: in particular, for every $i \in \{1, \ldots, n\}$, each element of dataset $\D_i$ has identical coordinates. Thus, there is $\alpha_i \in [\pm1]$ such that $\overline{\D}_i = \frac{\alpha_i}{\sqrt{d}} \cdot \ones$.
Plugging this in \eqref{eq:loss-full-1} yields:
\begin{align}
    \loss_{\H}{(\theta)} - \loss_{*,\H}
    &= \norm{\theta - \frac{1}{\card{\H}} \sum_{i \in \H}\overline{\D}_i}^2
    = \norm{\theta - \frac{1}{\sqrt{d} \cdot \card{\H}} \sum_{i \in \H}\alpha_i \ones}^2
    = \sum_{k=1}^d \absv{\theta_k - \frac{1}{\sqrt{d} \cdot \card{\H}} \sum_{i \in \H}\alpha_i}^2 \nonumber\\
    &= \frac{1}{d} \sum_{k=1}^d \absv{\sqrt{d} \cdot \theta_k - \frac{1}{\card{\H}} \sum_{i \in \H}\alpha_i}^2,
\end{align}
where $\theta_k$ denotes the $k$-th coordinate of $\theta \in \R^d$.
Upon applying \eqref{eq5} and then Jensen's inequality, we obtain
\begin{align}
    \varrho 
    \geq \expect{\loss_{\H}{(\hat{\theta})} - \loss_{*,\H}}
    &= \frac{1}{d} \sum_{k=1}^d \expect{\absv{\sqrt{d} \cdot \hat{\theta}_k - \frac{1}{\card{\H}} \sum_{i \in \H}\alpha_i}^2}
    \geq \expect{\absv{\frac{1}{d}\sum_{k=1}^d \sqrt{d} \cdot \hat{\theta}_k-\frac{1}{\card{\H}} \sum_{i \in \H}\alpha_i}^2} \nonumber\\
    &= \expect{\absv{\sum_{k=1}^d \frac{\hat{\theta}_k}{\sqrt{d}}-\frac{1}{\card{\H}} \sum_{i \in \H}\alpha_i}^2}.
    \label{eq:loss-full-2}
\end{align}
Therefore, everything happens as if $d=1$. That is, data universe $\X = \{\pm 1\}$, and datasets $\D_+ = \{+1\}^m$, $\D_- = \{-1\}^m$, and $\D_1 = \{\alpha_{1,1},\ldots, \alpha_{1,m}\}$ being arbitrary in $\X^m$.
Indeed, denote $\tilde{\theta} \coloneqq \sum_{k=1}^d \frac{\hat{\theta}_k}{\sqrt{d}} \in \R$. 
Recall that, now that $d=1$, each $\alpha_i \in [\pm1]$ is such that $\overline{\D}_i = \alpha_i$.
In this one-dimensional setting of datasets, we develop the RHS of \eqref{eq:loss-full-2}, by using the aforementioned bias-variance decomposition backwards:
\begin{align*}
    \varrho \geq \expect{\absv{\tilde{\theta}-\frac{1}{\card{\H}} \sum_{i \in \H}\alpha_i}^2}
    &= \expect{\frac{1}{\card{\H}}\sum_{i \in \H} \frac{1}{m} \sum_{x \in \D_i} \absv{\tilde{\theta} - x}^2} - \expect{\frac{1}{\card{\H}} \sum_{i \in \H} \frac{1}{m} \sum_{x \in \D_i}\absv{x- \frac{1}{\card{\H}} \sum_{i \in \H}\overline{\D}_i}^2}\nonumber\\
    &= \expect{\loss_{\H}{(\tilde{\theta})}-\loss_{*,\H}}.
\end{align*}
Thus, \eqref{eq5} holds with loss $\ell$ being the one-dimensional quadratic loss and mechanism $\widetilde{\M}$ returning $\tilde{\theta}$ instead of $\hat{\theta}$.
Since $\tilde{\theta}$ is a function of $\hat{\theta}$ without access to $\D_1$, $\widetilde{\M}$ is also $(\varepsilon, \delta)$-DP by post-processing.
{\bf Throughout the remainder of the proof, we set $d=1$ without loss of generality.}
\end{tcolorbox}

We consider below the RHS of \eqref{eq5}. We have for any $\theta \in \R$:
\begin{align}
    \loss_{\H}{(\theta)}
    &= \frac{1}{\card{\H}} \sum_{i \in \H} \frac{1}{m} \sum_{x \in \D_i} \absv{\theta-x}^2.
    \label{eq:loss-full}
\end{align}
The above function is minimized at $\theta^*_{\H} \coloneqq \frac{1}{\card{\H}} \sum_{i \in \H}\overline{\D}_i$ the average of one-way marginals $\overline{\D}_i \coloneqq \frac{1}{m} \sum_{x \in \D_i}x$.

Next, following \eqref{eq7}, we consider {\bf two possible cases} of honest identities, a priori indistinguishable to the algorithm. In the first case, 
we consider the set of honest identities $\H$ to be $\H_1 = \{1, \ldots, \, n-f\}$. In the second case, 
we consider the set of honest identities $\H$ to be $\H_2 \coloneqq \{1\} \cup  \{f+2, \ldots, n\}$. As $\card{\H} = n-f$, upon invoking \cref{def:resilience} in both the cases, we obtain a upper bound on $\expect{|\hat{\theta} - \overline{\D}_1|^2}$ in terms of $\varrho$.

\textit{\underline{First case}:}
Consider $\H$ to be $\H_1 = \{1, \ldots, n-f\}$.
Recall that $\D_i = \D_+$ for all $i \in \{2, \ldots, n-f\}$.
By \eqref{eq:loss-full}, we have for all $\theta \in \R$:
\begin{align*}
    \loss_{\H_1}{(\theta)}
    &= \frac{1}{\card{\H_1}} \sum_{i \in \H_1} \frac{1}{m} \sum_{x \in \D_i} \absv{\theta-x}^2
    = \frac{1}{\card{\H_1}} \frac{1}{m} \sum_{x \in \D_1} \absv{\theta-x}^2 + \frac{\card{\H_1}-1}{\card{\H_1}} \frac{1}{m} \sum_{x \in \D_+} \absv{\theta-x}^2 \nonumber\\
    &= \frac{1}{n-f} \frac{1}{m} \sum_{x \in \D_1} \absv{\theta-x}^2 + (1-\frac{1}{n-f})\absv{\theta-\overline{\D}_+}^2 \nonumber\\
    &\geq \frac{1}{n-f} \absv{\theta-\overline{\D}_1}^2 + (1-\frac{1}{n-f})\absv{\theta-\overline{\D}_+}^2. &(\text{Jensen's inequality})
\end{align*}
Thus, from above we obtain that 
\begin{align}
    \label{sit1-eq1}
    \expect{\loss_{\H_1}{(\hat{\theta})}}
    &\geq \frac{1}{n-f} \expect{|\hat{\theta}-\overline{\D}_1|^2} + (1-\frac{1}{n-f})\expect{|\hat{\theta}-\overline{\D}_+|^2}.
\end{align}
Now, recall the following bias-variance decomposition: for any $x_1, \ldots, x_n \in \R$ we have $\frac{1}{n} \sum_{i=1}^n \absv{x_i - \overline{x}}^2 = \frac{1}{n} \sum_{i=1}^n \absv{x_i}^2 - \absv{\overline{x}}^2$ where $\overline{x} \coloneqq \frac{1}{n} \sum_{i=1}^n x_i$.
Thus, from~\eqref{eq:loss-full} we obtain that $\theta^*_{\H_1} = \frac{1}{\card{\H_1}} \sum_{i \in \H_1} \overline{\D}_i $. Thus, as $\absv{x}^2=1$ for all $x \in \X$, we have
\begin{align*}
    \loss_{*,\H_1} 
    = \loss_{\H_1}{(\theta^*_{\H_1})}
    &= \frac{1}{m \card{\H_1}} \sum_{i \in \H_1}\sum_{x \in \D_i} \absv{\theta^*_{\H_1}-x}^2
    = \frac{1}{m \card{\H_1}} \sum_{i \in \H_1} \sum_{x \in \D_i} \absv{x}^2 - \absv{\theta^*_{\H_1}}^2 \nonumber\\
    &= 1 - \absv{\theta^*_{\H_1}}^2
    = 1 - \absv{\frac{1}{\card{\H_1}}\sum_{i \in \H_1}\overline{\D}_i}^2
    = 1 - \absv{\frac{1}{n-f}\overline{\D}_1 + (1-\frac{1}{n-f})\overline{\D}_+}^2 \nonumber\\
    &= 1 - \absv{\frac{1}{n-f}\overline{\D}_1 + 1-\frac{1}{n-f}}^2
    = 1 - \frac{1}{(n-f)^2}\absv{\overline{\D}_1 + n-f-1}^2.
\end{align*}
Note that, as $\D_1 \in \X^m=\{\pm 1\}^m$, we have $\overline{\D}_1 \in [\pm 1]$. Also, since $f<n/2$ and $n \geq 3$, we have $n-f-2 \geq 0$. Therefore, $ \absv{\overline{\D}_1 + n-f-1}^2 \geq \absv{n-f-2}^2$. 
Substituting this in the above, we obtain that
\begin{align}
    \label{sit1-eq2}
    \loss_{*,\H_1} 
    &= 1 - \frac{1}{(n-f)^2}\absv{\overline{\D}_1 + n-f-1}^2
    \leq 1 - \frac{1}{(n-f)^2}\absv{n-f-2}^2
    = 1 - \absv{1-\frac{2}{n-f}}^2 \nonumber\\
    &=\frac{2}{n-f}(2-\frac{2}{n-f}) 
    = \frac{4}{n-f}(1-\frac{1}{n-f})
    \leq \frac{4}{n-f} \leq \frac{4(f+1)}{n-f}.
\end{align}
Substituting from~\eqref{sit1-eq1} and~\eqref{sit1-eq2} in~\eqref{eq5} we obtain that
\begin{align}
    \label{sit1-eq3}
    \varrho + \frac{4(f+1)}{n-f} 
    &\geq  \varrho + \loss_{*,\H_1} 
    \geq \expect{\loss_{\H_1}{(\hat{\theta})}}
    \geq \frac{1}{n-f} \expect{|\hat{\theta}-\overline{\D}_1|^2} + (1-\frac{1}{n-f})\expect{|\hat{\theta}-\overline{\D}_+|^2}.
\end{align}

\textit{\underline{Second case}:}
Consider $\H$ to be $\H_2 = \{1\} \cup  \{f+2, \ldots, n\}$.
Recall that $\D_i = \D_-$ for all $i \in \{n-f+1, \ldots, n\}$.
By \eqref{eq:loss-full}, we have for all $\theta \in \R$:
\begin{align*}
    \loss_{\H_2}{(\theta)}
    &= \frac{1}{\card{\H_2}} \sum_{i \in \H_2} \frac{1}{m} \sum_{x \in \D_i} \absv{\theta-x}^2 \nonumber\\
    &= \frac{1}{\card{\H_2}} \frac{1}{m} \sum_{x \in \D_1} \absv{\theta-x}^2 + \left(\frac{\card{\H_2}-1-f}{\card{\H_2}} \right) \frac{1}{m} \sum_{x \in \D_+} \absv{\theta-x}^2 + \left(\frac{f}{\card{\H_2}} \right) \frac{1}{m} \sum_{x \in \D_-} \absv{\theta-x}^2 \nonumber\\
    &= \left(\frac{1}{n-f} \right) \frac{1}{m} \sum_{x \in \D_1} \absv{\theta-x}^2 + \left(\frac{n-2f-1}{n-f} \right)\absv{\theta-\overline{\D}_+}^2 + \frac{f}{n-f} \absv{\theta-\overline{\D}_-}^2 \nonumber\\
    &\geq \left(\frac{1}{n-f} \right) \frac{1}{m} \sum_{x \in \D_1} \absv{\theta-x}^2 + \frac{f}{n-f} \absv{\theta-\overline{\D}_-}^2 &(n\geq2f+1) \nonumber\\
    &\geq  \frac{1}{n-f} \absv{\theta-\overline{\D}_1}^2 + \frac{f}{n-f} \absv{\theta-\overline{\D}_-}^2. &(\text{Jensen's inequality})
\end{align*}
Substituting $\theta = \hat{\theta}$, and taking expectation yields
\begin{align}
    \label{sit2-eq1}
    \expect{\loss_{\H_2}{(\hat{\theta})}}
    &\geq \frac{1}{n-f} \expect{|\hat{\theta}-\overline{\D}_1|^2} + \frac{f}{n-f}\expect{|\hat{\theta}-\overline{\D}_-|^2}.
\end{align}
Now, recall the following bias-variance decomposition: for any $x_1, \ldots, x_n \in \R$ we have $\frac{1}{n} \sum_{i=1}^n \absv{x_i - \overline{x}}^2 = \frac{1}{n} \sum_{i=1}^n \absv{x_i}^2 - \absv{\overline{x}}^2$, where we denoted $\overline{x} \coloneqq \frac{1}{n} \sum_{i=1}^n x_i$.
Using this in \eqref{eq:loss-full}, and that $\forall x \in \X, \absv{x}^2=1$, we get
\begin{align*}
    \loss_{*,\H_2} 
    = \loss_{\H_2}{(\theta^*_{\H_2})}
    &= \frac{1}{m \card{\H_2}} \sum_{i \in \H_2}\sum_{x \in \D_i} \absv{\theta^*_{\H_2}-x}^2
    = \frac{1}{m \card{\H_2}} \sum_{i \in \H_2} \sum_{x \in \D_i} \absv{x}^2 - \absv{\theta^*_{\H_2}}^2 \nonumber\\
    &= 1 - \absv{\theta^*_{\H_2}}^2
    = 1 - \absv{\frac{1}{\card{\H_2}}\sum_{i \in \H_2}\overline{\D}_i}^2
    = 1 - \absv{\frac{1}{n-f}\overline{\D}_1 + \frac{n-2f-1}{n-f}\overline{\D}_+ + \frac{f}{n-f}\overline{\D}_-}^2 \nonumber\\
    &= 1 - \absv{\frac{1}{n-f}\overline{\D}_1 + \frac{n-2f-1}{n-f} -\frac{f}{n-f}}^2
    = 1 - \absv{1 + \frac{1}{n-f}\overline{\D}_1 -\frac{2f+1}{n-f}}^2\\
    &= \left(1- 1 - \frac{1}{n-f}\overline{\D}_1 +\frac{2f+1}{n-f}\right)\left(1+1 + \frac{1}{n-f}\overline{\D}_1 -\frac{2f+1}{n-f}\right)\\
    &= \left(\frac{2f+1-\overline{\D}_1}{n-f}\right)\left(2 - \frac{2f+1-\overline{\D}_1}{n-f}\right).
\end{align*}
Note that, as $\D_1 \in \X^m=\{\pm 1\}^m$, we have $\overline{\D}_1 \in [\pm 1]$. This, together with $n \geq 2f+1$, implies that both the terms in the product above are non-negative.
Moreover, as $\overline{\D}_1 \geq -1$, the first term can be bounded by $$\frac{2f+1-\overline{\D}_1}{n-f} \leq \frac{2(f+1)}{n-f}.$$ 
Similarly, as $\overline{\D}_1 \leq 1$, the second term can be bounded by $$2-\frac{2f+1-\overline{\D}_1}{n-f} \leq 2 - \frac{2f}{n-f} \leq 2.$$
Consequently, we have
\begin{equation}
    \label{sit2-eq2}
    \loss_{*,\H_2} \leq \frac{4(f+1)}{n-f}.
\end{equation}

Invoking \eqref{eq5} with the set of honest identities $\H_2$, and using the bounds shown in \eqref{sit2-eq1}, \eqref{sit2-eq2} yields:
\begin{align}
    \label{sit2-eq3}
    \varrho + \frac{4(f+1)}{n-f} 
    &\geq  \varrho + \loss_{*,\H_2} 
    \geq \expect{\loss_{\H_2}{(\hat{\theta})}}
    \geq \frac{1}{n-f} \expect{|\hat{\theta}-\overline{\D}_1|^2} + \frac{f}{n-f}\expect{|\hat{\theta}-\overline{\D}_-|^2}.
\end{align}

\textit{\underline{Final step}:}
We deduce from \eqref{sit1-eq3}, \eqref{sit2-eq3} that 
\begin{align}
    \varrho + \frac{4(f+1)}{n-f} 
    &\geq \max\Big\{
    \frac{1}{n-f} \expect{|\hat{\theta}-\overline{\D}_1|^2} + (1-\frac{1}{n-f})\expect{|\hat{\theta}-\overline{\D}_+|^2}, \nonumber\\
    &\qquad \qquad \frac{1}{n-f} \expect{|\hat{\theta}-\overline{\D}_1|^2} + \frac{f}{n-f}\expect{|\hat{\theta}-\overline{\D}_-|^2}
    \Big\} \nonumber\\
    &= \frac{1}{n-f} \expect{|\hat{\theta}-\overline{\D}_1|^2} + \max\left\{
    (1-\frac{1}{n-f})\expect{|\hat{\theta}-\overline{\D}_+|^2}, \frac{f}{n-f}\expect{|\hat{\theta}-\overline{\D}_-|^2}
    \right\} \nonumber \\
    &\geq \frac{1}{n-f} \expect{|\hat{\theta}-\overline{\D}_1|^2} + \frac{f}{n-f} \max\left\{
    \expect{|\hat{\theta}-\overline{\D}_+|^2}, \expect{|\hat{\theta}-\overline{\D}_-|^2}
    \right\},
    \label{eq:final-step1}
\end{align}
where the last inequality is due to $f < \frac{n}{2}$, which implies that $1 - \frac{1}{n-f} \geq \frac{f}{n-f}$.
Besides, observe that, as $\D_1 \in \X^m = \{ \pm 1 \}^m$, we have $\overline{\D}_1 \in [\pm 1]$.
Recall that $\overline{\D}_+ = +1$ and $\overline{\D}_- = -1$.
Thus, it holds that
\begin{equation}
    \label{eq7}
    \expect{|\hat{\theta} - \overline{\D}_1|^2}
    \leq \max{(\expect{|\hat{\theta} - \overline{\D}_+|^2}, \expect{|\hat{\theta} - \overline{\D}_-|^2})}.
\end{equation}
Indeed, since $\overline{\D}_1 \in [\pm 1]$, we can write $\overline{\D}_1 = \lambda \times (+1) + (1-\lambda) \times (-1)$ for some $\lambda \in [0,1]$.
Thus, using Jensen's inequality and then taking expectations, we have
$\expect{|\hat{\theta} - \overline{\D}_1|^2} \leq \lambda \expect{|\hat{\theta} - 1|^2} + (1-\lambda) \expect{|\hat{\theta} +1|^2} \leq \max{(\expect{|\hat{\theta} - 1|^2},  \expect{|\hat{\theta} +1|^2})}$.

Using \eqref{eq7} in~\eqref{eq:final-step1}, we obtain, for every $\D_1 \in \X^m$, that
\begin{align}
    \label{eq8}
    \varrho + \frac{4(f+1)}{n-f} 
    \geq \frac{f+1}{n-f}\expect{|\hat{\theta}-\overline{\D}_1|^2}.
\end{align}
Before concluding, recall that $1\leq f \leq \frac{n}{2}$, thus applying Proposition~\ref{prop:heterogeneity} with $G=1$ yields 
\begin{equation}
    \varrho = \Omega{\left( \frac{f}{n}\right)} = \Omega{\left( \frac{f+1}{n-f}\right)}.
    \label{eq:hetero-rho}
\end{equation}
Indeed, since the data universe considered in the proof includes $\{\pm \frac{1}{\sqrt{\kappa d}}\}^d$, we can apply Proposition~\ref{prop:heterogeneity}.
Plugging this back in \eqref{eq8}, we have for every $\D_1 \in \X^m$ that
\begin{equation*}
    \varrho = \Omega{\left(\frac{f+1}{n-f}\expect{|\hat{\theta}-\overline{\D}_1|^2}\right)}.
\end{equation*}

\textbf{Invoking \cref{lemma:dp-lb}.}
Hence, since $\varrho \leq \frac{f+1}{100(n-f)}$, we can proceed in the same way as in the proof of \cref{prop:dp-lb-hd} to leverage \cref{lemma:dp-lb} (with $d=1$) for showing
\begin{equation*}
    \frac{n-f}{f+1}\varrho = \Omega{\left(\frac{\log{(1/\delta)}}{\varepsilon^2 m^2}\right)}.
\end{equation*}
We finally conclude the desired result by rearranging terms and ignoring absolute constants:
\begin{equation*}
    \varrho = \Omega{\left(\frac{f+1}{n-f} \cdot \frac{\log{(1/\delta)}}{\varepsilon^2 m^2}\right)}.
\end{equation*}
\end{proof}

\subsection{Final Lower Bound}
\label{app:final-lb}
We prove below the final lower bound stated in Theorem~\ref{th:finallb}.
\finalLB*
\begin{proof}
The proof consists in showing that the setting we consider in the above theorem allows us to merge the lower bounds from propositions~\ref{prop:dp-lb-hd}, \ref{prop:dp-lb-byz}, and \ref{prop:heterogeneity}. First, we remark that the case $f=0$ corresponds to simply showing that $\varrho = \widetilde{\Omega} \left(\frac{d}{\varepsilon^2 n m^2}\right)$, which follows immediately from Proposition~\ref{prop:dp-lb-hd} directly (see Step 1 below for verifying the applicability of the proposition). In the remainder of the proof, we will assume $f > 0$ and $\eta > 0$. Let $\H$ denote the set of honest nodes of size $n-f$.

\textit{\underline{Step 1}:} 
To derive the first term in $\Omega \left( \frac{d}{\varepsilon^2 n m^2} \right)$, we remark that all the conditions of Proposition~\ref{prop:dp-lb-hd} on $\varepsilon, \delta, \varrho, n, m$ hold under the assumptions stated in the theorem. 
Consider $\mathcal{D}_1, \ldots, \mathcal{D}_n \in \{\pm 1/\sqrt{8d}\}^{d\times m} \subset \mathcal{X}^m$. Note that in this case, we have
\begin{align*}
\frac{1}{\card{\H}} \sum_{i \in \H} \norm{\nabla{\loss{(\theta;\mathcal{D}_i)}} - \nabla{\loss_{\H}{(\theta)}}}^2
\leq 1 \leq G^2.
\end{align*}
Hence, $\mathcal{D}_1, \ldots, \mathcal{D}_n$ is a valid collection of datasets with regard to the theorem statement.
Since $\mathcal{A}$ is assumed to be $(f,\varrho)$-robust, it guarantees an error less than or equal to $\varrho$ on the honest global loss $\loss{(\theta; \mathcal{D}_i, \, i \in \H)}$. Using the proof technique of Proposition~\ref{prop:dp-lb-hd}, we can show that (as $f<n/2$ and $\card{\H}=n-f \leq n$)
\begin{equation}
    \varrho = \Omega \left(\frac{d}{\varepsilon^2 \card{\H} m^2}\right) = \Omega \left(\frac{d}{\varepsilon^2 n m^2}\right) .  \label{eq:first}
\end{equation}

\textit{\underline{Step 2}:} 
To derive the second term in $\Omega(\frac{f}{n}\cdot\frac{1}{\varepsilon^2 m^2})$, we remark that all conditions of Proposition~\ref{prop:dp-lb-byz} on $\varepsilon, \delta, \varrho, n, f, m$, and $\mathcal{A}$ are verified. Note also that, similar to Step 1, the datasets considered in the proof Proposition~\ref{prop:heterogeneity}, scaled by a constant, are also valid instances with regard to the theorem statement. Using the proof technique of Proposition~\ref{prop:heterogeneity} we can show that (since $0 < f < n/2$, we have $f+1 \geq f$ and $n-f \leq n$)
\begin{equation}
    \varrho = \Omega{\left(\frac{f+1}{n-f} \cdot \frac{\log{(1/\delta)}}{\varepsilon^2 m^2}\right)} = \widetilde{\Omega}{\left(\frac{f}{n} \cdot \frac{1}{\varepsilon^2 m^2}\right)}, \label{eq:second}
\end{equation}
where we ignore the logarithmic term in $\widetilde{\Omega}(\cdot)$.

\textit{\underline{Step 3}:}  
To obtain the third term in $\Omega \left(\frac{f}{n}\cdot G^2\right)$, we first remark that Assumption~\ref{asp:hetero} 
holds, as well as all the conditions in Proposition~\ref{prop:heterogeneity} on $n,f,m$ and $\mathcal{A}$. As the input domain in Proposition~~\ref{prop:heterogeneity} is a subset of $\mathcal{X}$, using the proof technique of Proposition~~\ref{prop:heterogeneity} we can show that
\begin{equation}
        \label{eq:third}
        \varrho = \Omega{\left(\frac{f}{n}\cdot G^2\right)}.
    \end{equation}

\textit{\underline{Final step}:}  
Combining~\eqref{eq:first},~\eqref{eq:second}, and~\eqref{eq:third} proves the theorem, i.e., we obtain that
\begin{align*}
    \varrho &= \widetilde{\Omega}{\left(\max{\left\{ \frac{d}{\varepsilon^2 n m^2}, \frac{f}{n} \cdot \frac{1}{\varepsilon^2 m^2}, \frac{f}{n} \cdot G^2\right\}}\right)} =\widetilde{\Omega} \left(\frac{d}{\varepsilon^2 n m^2} + \frac{f}{n} \cdot \frac{1}{\varepsilon^2 m^2} + \frac{f}{n} \cdot G^2\right).
\end{align*}
\end{proof}

\clearpage
\section{Robustness Analysis}
\label{app:robustness}

In this section, we prove all our claims related to $(f,\kappa)$-robustness and SMEA.
In Section~\ref{app:smea}, we analyze SMEA.
In Section~\ref{app:filter}, we discuss Filter~\cite{diakonikolas2017being,steinhardt2018resilience}, a related algorithm.

We first recall the definition of our robustness criterion:
\robustness*

\subsection{Smallest Maximum Eigenvalue Averaging (SMEA)}
\label{app:smea}

Given a set of $n$ vectors $x_1, \ldots, \, x_n \in \R^d$, the SMEA algorithm first searches for a set $S^{*}$ of cardinality $n-f$ with the smallest empirical {\em maximum eigenvalue}, i.e., 
\begin{equation}
    S^{*} \in \argmin_{\underset{\card{S} = n-f}{S \subseteq \{ 1, \ldots, \, n} \} } \lambda_{\max}{\left(\frac{1}{\card{S}} \sum_{i \in S}(x_i - \overline{x}_S)(x_i - \overline{x}_S)^\top \right)}. \label{eqn:def_S*}
\end{equation}
Then the algorithm outputs the average of the inputs in set $S^{*}$:  
\begin{equation}
    \mathrm{SMEA}(x_1, \ldots, \, x_n) \coloneqq \frac{1}{\card{S^*}} \sum_{i \in S^{*}} x_i.
    \label{eq:smea-output}
\end{equation}

\smea*
\begin{proof}
Let $n \geq 1$ and $0 \leq f < n/2$.
Fix a set $S \subseteq \left \{ 1, \dots, n \right \}$ such that $\card{S}=n-f$.
Recall the definition of $S^*$ in \eqref{eqn:def_S*}.
Denote by $\overline{x}_{S^*}$ the output of SMEA defined in \eqref{eq:smea-output}:
\begin{equation}
    \overline{x}_{S^*} \coloneqq \frac{1}{\card{S^*}} \sum_{i \in S^{*}} x_i.
    \label{eq:smea-output-2}
\end{equation}

From \eqref{eq:smea-output-2}, we have
\begin{align*}
    \norm{\overline{x}_{S^*}-\overline{x}_{S}}^2
    &= \norm{\frac{1}{n-f}\sum_{i\in S^*} x_i - \frac{1}{n-f}\sum_{i\in S} x_i}^2
    = \norm{\frac{1}{n-f}\sum_{i\in S^*\setminus S} x_i - \frac{1}{n-f}\sum_{i\in S\setminus S^*} x_i}^2 \\
    &= \norm{\frac{1}{n-f}\sum_{i\in S^*\setminus S} (x_i-\overline{x}_{S^*}) - \frac{1}{n-f}\sum_{i\in S\setminus S^*} (x_i-\overline{x}_{S}) + \frac{\card{S^*\setminus S}}{n-f}(\overline{x}_{S^*}-\overline{x}_{S})}^2 \\
    &= \norm{\frac{1}{n-f}\sum_{i\in S^*\setminus S} (x_i-\overline{x}_{S^*}) - \frac{1}{n-f}\sum_{i\in S\setminus S^*} (x_i-\overline{x}_{S})}^2
    + \frac{\card{S^*\setminus S}^2}{(n-f)^2}\norm{\overline{x}_{S^*}-\overline{x}_{S}}^2 \\
    &\quad+2\frac{\card{S^*\setminus S}}{n-f}\iprod{\overline{x}_{S^*}-\overline{x}_{S}}{\frac{1}{n-f}\sum_{i\in S^*\setminus S} (x_i-\overline{x}_{S^*}) - \frac{1}{n-f}\sum_{i\in S\setminus S^*} (x_i-\overline{x}_{S})}.
\end{align*}
However, notice that
\begin{align*}
    \frac{1}{n-f}\sum_{i\in S^*\setminus S} (x_i-\overline{x}_{S^*}) - \frac{1}{n-f}\sum_{i\in S\setminus S^*} (x_i-\overline{x}_{S})
    &= \frac{1}{n-f}\sum_{i\in S^*\setminus S}x_i - \frac{1}{n-f}\sum_{i\in S\setminus S^*}x_i
    -\frac{\card{S^* \setminus S}}{n-f}(\overline{x}_{S^*}-\overline{x}_{S}) \\
    &= \frac{1}{n-f}\sum_{i\in S^*}x_i - \frac{1}{n-f}\sum_{i\in S}x_i
    -\frac{\card{S^* \setminus S}}{n-f}(\overline{x}_{S^*}-\overline{x}_{S}) \\
    &= \left(1-\frac{\card{S^* \setminus S}}{n-f}\right)(\overline{x}_{S^*}-\overline{x}_{S}).
\end{align*}
This implies that
\begin{align*}
    \norm{\overline{x}_{S^*}-\overline{x}_{S}}^2
    &= \norm{\frac{1}{n-f}\sum_{i\in S^*\setminus S} (x_i-\overline{x}_{S^*}) - \frac{1}{n-f}\sum_{i\in S\setminus S^*} (x_i-\overline{x}_{S})}^2 \\
    &\quad + \left [ \frac{\card{S^*\setminus S}^2}{(n-f)^2} + 2 \frac{\card{S^*\setminus S}}{n-f}\left(1-\frac{\card{S^* \setminus S}}{n-f}\right) \right] \norm{\overline{x}_{S^*}-\overline{x}_{S}}^2 \\
    &= \norm{\frac{1}{n-f}\sum_{i\in S^*\setminus S} (x_i-\overline{x}_{S^*}) - \frac{1}{n-f}\sum_{i\in S\setminus S^*} (x_i-\overline{x}_{S})}^2 
    + \left [ 1 - \left (1- \frac{\card{S^*\setminus S}}{n-f} \right)^2 \right] \norm{\overline{x}_{S^*}-\overline{x}_{S}}^2
\end{align*}
By rearranging the terms, applying Jensen's inequality, and using the fact that $\sup_{\norm{v}\leq 1} \absv{\iprod{v}{x}} = \norm{x}$, we obtain
\begin{align}
    \label{eq:mva-1}
    \left (1- \frac{\card{S^*\setminus S}}{n-f} \right)^2 \norm{\overline{x}_{S^*}-\overline{x}_{S}}^2
    &= \norm{\frac{1}{n-f}\sum_{i\in S^*\setminus S} (x_i-\overline{x}_{S^*}) - \frac{1}{n-f}\sum_{i\in S\setminus S^*} (x_i-\overline{x}_{S})}^2 \nonumber\\
    &= \sup_{\norm{v} \leq 1}
    \absv{\iprod{v}{\frac{1}{n-f}\sum_{i\in S^*\setminus S} (x_i-\overline{x}_{S^*}) - \frac{1}{n-f}\sum_{i\in S\setminus S^*} (x_i-\overline{x}_{S})}}^2 \nonumber\\
    &= \sup_{\norm{v} \leq 1}
    \absv{\frac{1}{n-f}\sum_{i\in S^*\setminus S} \iprod{v}{x_i-\overline{x}_{S^*}}- \frac{1}{n-f}\sum_{i\in S\setminus S^*} \iprod{v}{x_i-\overline{x}_{S}}}^2 \nonumber\\
    &\leq \frac{\card{S^*\setminus S}+\card{S \setminus S^*}}{(n-f)^2}
    \sup_{\norm{v} \leq 1} 
    \left[ \sum_{i\in S^*\setminus S} \absv{\iprod{v}{x_i-\overline{x}_{S^*}}}^2+ \sum_{i\in S\setminus S^*} \absv{\iprod{v}{x_i-\overline{x}_{S}}}^2  \right] \nonumber\\
    &\leq \frac{\card{S^*\setminus S}+\card{S \setminus S^*}}{(n-f)^2}
    \left[ \sup_{\norm{v} \leq 1} \sum_{i\in S^*\setminus S} \absv{\iprod{v}{x_i-\overline{x}_{S^*}}}^2
    + \sup_{\norm{v} \leq 1}\sum_{i\in S\setminus S^*} \absv{\iprod{v}{x_i-\overline{x}_{S}}}^2  \right] \nonumber\\
    &\leq \frac{2f}{(n-f)^2}
    \left[ \sup_{\norm{v} \leq 1} \sum_{i\in S^*\setminus S} \absv{\iprod{v}{x_i-\overline{x}_{S^*}}}^2
    + \sup_{\norm{v} \leq 1}\sum_{i\in S\setminus S^*} \absv{\iprod{v}{x_i-\overline{x}_{S}}}^2  \right],
\end{align}
where the last inequality is due to the fact that $\card{S^*}=\card{S}=n-f$, as we must have
\begin{equation}
    \card{S \setminus S^*} = \card{S^* \setminus S} = \card{S \cup S^*} - \card{S} \leq n - (n-f) = f.
    \label{eq:quorum}
\end{equation}

The first term on the RHS of \eqref{eq:mva-1} can be bounded by construction of $S^*$, and using the fact that $\sup_{\norm{v}\leq1}\iprod{v}{Mv} = \lambda_{\max}{(M)}$:
\begin{align*}
    \sup_{\norm{v} \leq 1} \sum_{i\in S^*\setminus S} \absv{\iprod{v}{x_i-\overline{x}_{S^*}}}^2
    &\leq \sup_{\norm{v} \leq 1} \sum_{i\in S^*} \absv{\iprod{v}{x_i-\overline{x}_{S^*}}}^2
    = \sup_{\norm{v} \leq 1} \iprod{v}{\sum_{i \in S^*}(x_i - \overline{x}_{S^*})(x_i - \overline{x}_{S^*})^\top v}\\
    &= \lambda_{\max}{\left(\sum_{i \in S^*}(x_i - \overline{x}_{S^*})(x_i - \overline{x}_{S^*})^\top\right)}
    \leq \lambda_{\max}{\left(\sum_{i \in S}(x_i - \overline{x}_S)(x_i - \overline{x}_S)^\top\right)}.
\end{align*}
The second term on the RHS of \eqref{eq:mva-1} can be bounded similarly:
\begin{align*}
    \sup_{\norm{v} \leq 1} \sum_{i\in S\setminus S^*} \absv{\iprod{v}{x_i-\overline{x}_{S}}}^2 \leq \sup_{\norm{v} \leq 1} \sum_{i\in S} \absv{\iprod{v}{x_i-\overline{x}_{S}}}^2
    = \lambda_{\max}{\left(\sum_{i \in S}(x_i - \overline{x}_S)(x_i - \overline{x}_S)^\top\right)}.
\end{align*}
Plugging these two bounds back in \eqref{eq:mva-1}, we obtain
\begin{align*}
    \left (1- \frac{\card{S^*\setminus S}}{n-f} \right)^2 \norm{\overline{x}_{S^*}-\overline{x}_{S}}^2 \leq \frac{4f}{n-f} \frac{1}{n-f} \lambda_{\max}{\left(\sum_{i \in S}(x_i - \overline{x}_S)(x_i - \overline{x}_S)^\top\right)}.
\end{align*}
Finally, since $\card{S^*\setminus S}\leq f$ (see \eqref{eq:quorum}), we have $\left (1- \frac{\card{S^*\setminus S}}{n-f} \right)^2 \geq \left (1- \frac{f}{n-f} \right)^2 = \left (\frac{n-2f}{n-f} \right)^2$.
We can therefore obtain
\begin{align*}
    \norm{\overline{x}_{S^*}-\overline{x}_{S}}^2 \leq \frac{4f(n-f)}{(n-2f)^2} \cdot \lambda_{\max}{\left(\frac{1}{\card{S}} \sum_{i \in S}(x_i - \overline{x}_S)(x_i - \overline{x}_S)^\top\right)}.
\end{align*}
The proof concludes by noticing that $\frac{4f(n-f)}{(n-2f)^2} = \frac{4f}{n-f}\left(1+\frac{f}{n-2f}\right)^2$.
\end{proof}

\subsection{Filter Algorithm}
\label{app:filter}

In this section, we present the Filter algorithm~\cite{diakonikolas2017being,steinhardt2018robust} in \cref{algo:filter} and discuss its robustness properties, stated in \cref{prop:filter}, in the distributed ML context we consider. Recall that Filter was also used in~\cite{data2021byzantine}.

\begin{algorithm}[ht!]
\caption{Filter algorithm~\cite{diakonikolas2017being,steinhardt2018robust}}
\textbf{Input:} 
vectors $x_1, \ldots, x_n \in \R^d$,
spectral norm bound $\sigma_0^2$,
constant factor $\eta > 0$.
\begin{algorithmic}[1]
\label{algo:filter}
\STATE Initialize $c_1, \ldots, c_n = 1$, $\hat{\sigma}_c = +\infty$.
\WHILE{True}
\STATE Compute the empirical mean 
$\hat{\mu}_c = \sum_{i=1}^n c_i x_i/\sum_{i=1}^n c_i$.
\STATE Compute the empirical covariance 
$\hat{\Sigma}_c = \sum_{i=1}^n c_i (x_i-\hat{\mu}_c)(x_i-\hat{\mu}_c)^\top/\sum_{i=1}^n c_i$.
\STATE Compute maximum eigenvalue $\hat{\sigma}_c^2$ of $\hat{\Sigma}_c$ and an associated eigenvector $\hat{v}_c$.
\IF{$\hat{\sigma}_c^2 > \eta \cdot \sigma_0^2$}
\STATE \textbf{return} $\hat{\mu}_c$
\ELSE
\STATE Compute weight $\tau_i = \iprod{\hat{v}_c}{x_i-\hat{\mu}_c}^2$.
\STATE Update $c_i \gets c_i (1-\tau_i/\tau_{\max})$, where $\tau_{\max} = \max_{1 \leq i \leq n} \tau_i$.
\ENDIF
\ENDWHILE
\end{algorithmic}
\end{algorithm}

In \cref{prop:filter}, we recall the robustness guarantees of the Filter procedure (\cref{algo:filter}). The proposition is followed by a discussion further below.
\begin{proposition}
\label{prop:filter}
Let $n \geq 1, 0 \leq f < n/2$, $x_1, \ldots, x_n \in \R^n$, and $S \subseteq [n], \card{S}=n-f$.
Denote $\overline{x}_S \coloneqq \frac{1}{\card{S}}\sum_{i \in S}x_i$.

Set the parameters 
$$\sigma_0^2 \geq \lambda_{\max}{\left(\frac{1}{\card{S}} \sum_{i \in S}(x_i - \overline{x}_S)(x_i - \overline{x}_S)^\top\right)}$$ 
and 
$$\eta = 2n(n-f)/(n-2f)^2.$$
\\
Then, the output $\widehat{x}$ of the Filter procedure (\cref{algo:filter}) with parameters $\sigma_0^2$ and $\eta$ satisfies
\begin{equation*}
    \norm{\widehat{x} - \overline{x}_S}^2 \leq \kappa \cdot \sigma_0^2,
\end{equation*}
with $\kappa = \frac{4fn}{(n-2f)^2} + \frac{2f}{n-f} = \frac{6f}{n-2f}\left(1+\frac{f}{n-2f}\right)$.
\end{proposition}
\begin{proof}
The proof follows directly from (Theorem~4.2, \cite{zhu2022robust}) combined with (Lemma~2.2, \cite{zhu2022robust}).
\end{proof}

\paragraph{Discussion.}
Note that Filter does not satisfy $(f,\kappa)$-robust averaging (see \cref{def:robust-averaging}) as its parameter $\sigma_0^2$ must depend on the maximum eigenvalue of the honest inputs.
Indeed, such dependency is precluded by $(f,\kappa)$-robust averaging.
Moreover, in our learning setting, the bound $\sigma_0^2$ potentially depends on the noise of stochastic gradients $\sigma^2$ and the heterogeneity metric $G^2$, which are unknown a priori. Thus, devising aggregation rules agnostic to the statistical properties of the honest inputs, like SMEA, is even more desirable in our setting.

\clearpage
\section{Privacy Analysis}
\label{app:privacy}

\subsection{Preliminaries}

We first recall definitions and useful lemmas on Differential Privacy (DP) and Rényi Differential Privacy (RDP), including the privacy amplification by subsampling (without replacement) results for RDP.

\begin{definition}[Rényi Differential Privacy, \cite{mironov2017renyi}]
Let $\alpha > 1$ and $\varepsilon>0$.
A randomized algorithm $\M$ is $(\alpha,\varepsilon)$-RDP if for any adjacent datasets $\D, \D' \in \X^m$ it holds that
\begin{equation*}
    D_{\alpha}{\left ( \mechanism{\D} \middle || \mechanism{\D'} \right )} \leq \varepsilon,
\end{equation*}
where $D_{\alpha}{\left ( \mechanism{\D} \middle || \mechanism{\D'} \right )} \coloneqq \frac{1}{\alpha-1}\log{\condexpect{\theta \sim \mechanism{\D'}}{\left(\frac{\mechanism{\D}{(\theta)}}{\mechanism{\D'}{(\theta)}}\right)^\alpha}}$ is the Rényi divergence of order $\alpha$.
\end{definition}

\begin{lemma}[RDP Adpative Composition, \cite{mironov2017renyi}]
\label{lem:composition}
If $\M_1$ that takes the dataset as input is $(\alpha,\varepsilon_1)$-RDP, and $\M_2$ that takes the dataset and the output of $\M_1$ as input is $(\alpha,\varepsilon_2)$-RDP, then their composition is $(\alpha,\varepsilon_1+\varepsilon_2)$-RDP.
\end{lemma}

\begin{lemma}[RDP to DP conversion, \cite{mironov2017renyi}]
\label{lem:rdp-dp}
If $\M$ is $(\alpha,\varepsilon)$-RDP, then $\M$ is $(\varepsilon+\frac{\log{(1/\delta)}}{\alpha-1},\delta)$-DP for all $\delta \in (0,1)$.
\end{lemma}

\begin{definition}[$\ell_2$-sensitivity, \cite{dwork2014algorithmic}]
The $\ell_2$-sensitivity of a function $g \colon \X^m \to \R^d$ is
\begin{align*}
    \Delta{(g)} \coloneqq \sup_{\D,\D' \text{ adjacent}} \norm{g{(\D)}-g{(\D')}}.
\end{align*}
\end{definition}

\begin{lemma}[RDP for Gaussian Mechanisms, \cite{mironov2017renyi}]
\label{lem:gm}
If $g:\X^m \to \R^d$ has $\ell_2$-sensitivity smaller than $\Delta$, then the Gaussian mechanism $G_{\sigma,g} = g + \N{(0,\sigma^2 I_d)}$ is $(\alpha, \frac{\Delta^2}{2\sigma^2}\alpha)$-RDP.
\end{lemma}

\begin{definition}[Subsampling Mechanism]
Consider a dataset $\D \subseteq \X^m$, a constant $b \in [m]$, and define $r \coloneqq \nicefrac{b}{m}$. The procedure $\sample_r: \X^m \rightarrow \X^b$ selects $b$ points at random and without replacement from $\D$.
\end{definition}

\begin{lemma}[RDP for Subsampled Mechanisms, \cite{wang2019subsampled}]
\label{lem:subsample}
Let $\alpha \in \mathbb{N}, \alpha \geq 2,$ and $r \in (0,1)$ the sampling parameter.
If $\M$ is $(\alpha, \varepsilon(\alpha))$-RDP, then $\M \circ \sample_r$ is $(\alpha, \varepsilon'(\alpha))$-RDP, with
\begin{align}
    \label{eq:eps-prime}
    \varepsilon'(\alpha) &= \frac{1}{\alpha-1}\log
    \Big( 1 + r^2 \binom{\alpha}{2} \min{\left \{ 4(e^{\varepsilon(2)}-1), e^{\varepsilon(2)} \min{\{ 2, (e^{\varepsilon(\infty)}-1)^2\}}  \right \}} \nonumber\\
    &\quad + \sum_{j=3}^\alpha r^j \binom{\alpha}{j} e^{(j-1)\varepsilon(j)} \min{\{2,(e^{\varepsilon(\infty)}-1)^j\}} \Big).
\end{align}
\end{lemma}

\begin{lemma}[Real-valued RDP for Subsampled Mechanisms]
\label{lem:rv-subsample}
Let $\alpha \in \R, \alpha > 1,$ and $r \in (0,1)$ the sampling parameter.
If $\M$ is $(\alpha, \varepsilon(\alpha))$-RDP, then $\M \circ \sample_r$ is $(\alpha, \varepsilon''(\alpha))$-RDP, with
\begin{align*}
    \varepsilon''{(\alpha)} = (1-\alpha+\floor{\alpha})\frac{\floor{\alpha}-1}{\alpha-1}\varepsilon'{(\floor{\alpha})} + (\alpha-\floor{\alpha})\frac{\ceil{\alpha}-1}{\alpha-1}\varepsilon'{(\ceil{\alpha})},
\end{align*}
where $\varepsilon'$ is defined in Equation~\eqref{eq:eps-prime}.
\end{lemma}
\begin{proof}
The result follows immediately from Corollary~10 and Remark~7 in \cite{wang2019subsampled}.
\end{proof}

\subsection{Proof of Theorem~\ref{thm:privacy} and Theorem~\ref{thm:exact-privacy}}

We state below the DP guarantees without approximation:

\begin{theorem}
\label{thm:exact-privacy}
Let $\delta \in (0,1)$.
Algorithm~\ref{algo:robust-dpsgd} is $(\varepsilon^*,\delta)$-DP with 
\begin{align*}
    \varepsilon^* = \inf_{\alpha > 1} \Big( T \varepsilon_1{(\alpha)} + \frac{\log{(1/\delta)}}{\alpha-1} \Big),
\end{align*}
where for every $\alpha>1$,
\begin{align*}
\begin{cases}
&\varepsilon_1(\alpha) \coloneqq
(1-\alpha+\floor{\alpha})\frac{\floor{\alpha}-1}{\alpha-1} \varepsilon'(\floor{\alpha})
+ (\alpha-\floor{\alpha})\frac{\ceil{\alpha}-1}{\alpha-1}\varepsilon'(\ceil{\alpha}), \\
&\varepsilon'(\alpha) \coloneqq
\frac{1}{\alpha-1}\log
    \Big( 1 + r^2 \binom{\alpha}{2} \min{\left \{ 4(e^{\varepsilon(2)}-1), 2e^{\varepsilon(2)}  \right \}}
    + 2\sum_{j=3}^\alpha r^j \binom{\alpha}{j} e^{(j-1)\varepsilon(j)} \Big),\\
&\varepsilon(\alpha) \coloneqq \big(\frac{2C}{b}\big)^2 \frac{\alpha}{2\sigmadp^2}.
\end{cases}
\end{align*}
\end{theorem}
\begin{proof}
To derive the above DP guarantees, we first track the privacy loss for a single iteration of Algorithm~\ref{algo:robust-dpsgd} using RDP. Then we apply adaptive composition to track the end-to-end privacy loss of the algorithm. Finally, we optimize over the privacy loss for several levels of RDP to compute the noise parameter needed for DP.

\textbf{Single-iteration privacy.}
First, we analyze a single fixed iteration $t \in \{0, \dots, T-1\}$ of Algorithm~\ref{algo:robust-dpsgd}.
To do so, we divide the analysis into two steps, i.e. Step I and Step II, as shown in Figure~\ref{fig:privacy-analysis}.

\begin{figure}
    \centering
\[
    \begin{tikzcd}[sep=large,
every arrow/.append style = {-stealth, shorten > = 2pt, shorten <=2pt},
                    ]
 \theta_t \ar[r,"(\textbf{I})"] & 
\tilde{g}_t^{(i)} \ar[r,"(\textbf{II})"]  &
 \theta_{t+1}
    \end{tikzcd}
\]
    \caption{
    (\textbf{I}): Subsampling + Gaussian mechanism,
    (\textbf{II}): Post-processing.
    }
    \label{fig:privacy-analysis}
\end{figure}

\underline{\textit{Step~(\textbf{I}):}}
This step corresponds to lines~2-6 in Algorithm~\ref{algo:robust-dpsgd}. 
Recall that our definition of DP for a distribution algorithm (given in \cref{def:distributed-dp}) requires that the transcript of communications of each worker satisfies (centralized) $(\varepsilon,\delta)$-DP with respect to their own data.
Thus, since the workers only send their local momentum to the server, we show that for any $i \in \H$ computing $\tilde{g}_t^{(i)}$ from $\D_i$ and $\theta_t$ is RDP for any $\alpha > 1$.

Let $i \in \H , \alpha>1$ and $r=\nicefrac{b}{m}$.
First, we show that $\Delta \coloneqq \frac{2C}{b}$ is an upper bound of the $\ell_2$-sensitivity of the mini-batch (clipped) averaging.
To see this, consider two adjacent training sets $\D_i,\Tilde{\D_i}$, the mini-batch average (after clipping) $g^{(i)}_t$ computed on mini-batch $S_t^{(i)} \subseteq \D_i$, and $\tilde{g}^{(i)}_t$ the analogous quantities for $\Tilde{\D_i}$.
Note that $S_t^{(i)}$ and $\Tilde{S}_t^{(i)}$ differ by one element at most. Without loss of generality, let $x_* \in S_t^{(i)}, \Tilde{x}_* \in \Tilde{S}_t^{(i)}$ be the only two elements that differ from $S_t^{(i)}$ to $\Tilde{S}_t^{(i)}$.
Thanks to the triangle inequality, we have that
\begin{align*}
    \norm{g^{(i)}_t-\tilde{g}^{(i)}_t}
    &= \Big\|\frac{1}{b} \sum_{x \in S_t^{(i)}} \textbf{Clip}\left(\nabla \ell{\left( \theta_{t}, x \right)} ; C \right) - \frac{1}{b} \sum_{x \in \Tilde{S}_t^{(i)}} \textbf{Clip}\left(\nabla \ell{\left( \theta_{t}, x \right)} ; C \right)\Big\| \\
    &= \norm{\frac{1}{b} \textbf{Clip}\left(\nabla \ell{\left( \theta_{t}, x_* \right)} ; C \right) - \frac{1}{b} \textbf{Clip}\left(\nabla \ell{\left( \theta_{t}, \tilde{x}_* \right)} ; C \right)} \\
    &\leq \frac{1}{b}\norm{\textbf{Clip}\left(\nabla \ell{\left( \theta_{t}, x_* \right)} ; C \right)} + \frac{1}{b} \norm{ \textbf{Clip}\left(\nabla \ell{\left( \theta_{t}, \tilde{x}_* \right)} ; C \right)} \\
    &\leq \frac{2C}{b}.
\end{align*}
Thanks to the above, the sensitivity of computing the gradient $g^{(i)}_t$ when given a batch of $b$ point $S_t^{(i)}$ is upper bounded by $\Delta = \frac{2C}{b}$. Accordingly, invoking Lemma~\ref{lem:gm}, the Gaussian mechanism used in Line~6 of Algorithm~\ref{algo:robust-dpsgd} is $(\alpha,\frac{\alpha \Delta^2}{2\sigmadp^2 })$-RDP.

Furthermore, by Lemma~\ref{lem:rv-subsample}, for every $j \in \H$, the corresponding mechanism $\M_j$ taking the dataset $\D_j$ and $\theta_t$ as input and returning $\tilde{g}_t^{(j)}$ is $(\alpha,\varepsilon_1(\alpha))$-RDP with
\begin{align}
    \label{eq:exact_alpha}
    \varepsilon_1(\alpha) \coloneqq 
    (1-\alpha+\floor{\alpha})\frac{\floor{\alpha}-1}{\alpha-1} \varepsilon'(\floor{\alpha})
    + (\alpha-\floor{\alpha})\frac{\ceil{\alpha}-1}{\alpha-1}\varepsilon'(\ceil{\alpha}).
\end{align}
Where
\begin{align*}
    \varepsilon'(\alpha) &= \frac{1}{\alpha-1}\log
    \Big( 1 + r^2 \binom{\alpha}{2} \min{\left \{ 4(e^{\varepsilon(2)}-1), e^{\varepsilon(2)} \min{\{ 2, (e^{\varepsilon(\infty)}-1)^2\}}  \right \}} \nonumber\\
    &\quad + \sum_{j=3}^\alpha r^j \binom{\alpha}{j} e^{(j-1)\varepsilon(j)} \min{\{2,(e^{\varepsilon(\infty)}-1)^j\}} \Big),
\end{align*}
and $\varepsilon(\alpha) \coloneqq \frac{\alpha \Delta^2}{2\sigmadp^2} = \big(\frac{2C}{b}\big)^2 \frac{\alpha}{2\sigmadp^2}$. Furthermore, since $\varepsilon(\infty)=+\infty$, we get
\begin{align}
    \label{eq:exact_one}
    \varepsilon'(\alpha) &= \frac{1}{\alpha-1}\log
    \Big( 1 + r^2 \binom{\alpha}{2} \min{\left \{ 4(e^{\varepsilon(2)}-1), 2e^{\varepsilon(2)} \right \}} 
    + 2\sum_{j=3}^\alpha r^j \binom{\alpha}{j} e^{(j-1)\varepsilon(j)} \Big).
\end{align}

\underline{\textit{Step~(\textbf{II}):}} This step consists in computing the local momentums from the noisy gradients, and then aggregating the momentums and updating the model accordingly. As this process does not have direct access to the datasets $\D_i, i \in \H$, it should be considered as a post-processing operation for Step~(\textbf{I}). As RDP is preserved by post-processing~\cite{mironov2017renyi}, we conclude that a single iteration of Algorithm~\ref{algo:robust-dpsgd} is $\left(\alpha,\varepsilon_1{(\alpha)}\right)$-RDP with respect to each worker's data for any $\alpha >1$, with $\varepsilon_1{(\alpha)}$ as defined above.

\textbf{End-to-end privacy.} 
We can now compute the end-to-end DP of our algorithm. First, invoking Lemma~\ref{lem:composition} and the per-iteration RDP guarantee of Algorithm~\ref{algo:robust-dpsgd}, we obtain that Algorithm~\ref{algo:robust-dpsgd} is 
$\left(\alpha,T \varepsilon_1{(\alpha)}\right)$-RDP towards the server, for any $\alpha >1$. Next, by Lemma~\ref{lem:rdp-dp}, we deduce that Algorithm~\ref{algo:robust-dpsgd} is $(\varepsilon^*{(\alpha)},\delta)$-DP towards the server for every $\delta \in (0,1), \alpha > 1,$ with
\begin{align*}
    \varepsilon^*{(\alpha)} \coloneqq T \varepsilon_1{(\alpha)} + \frac{\log{(1/\delta)}}{\alpha-1}.
\end{align*}
This implies that, for any $\delta \in (0,1)$, Algorithm~\ref{algo:robust-dpsgd} is $(\varepsilon^*,\delta)$-DP with
\begin{align*}
    \varepsilon^* \coloneqq \inf_{\alpha > 1} \varepsilon^*{(\alpha)}
    = \inf_{\alpha > 1} \Big( T \varepsilon_1{(\alpha)} + \frac{\log{(1/\delta)}}{\alpha-1}\Big).
\end{align*}
The above concludes the proof.
\end{proof}

We now prove the (closed-form) approximate DP guarantees of \algoname{} in Theorem~\ref{thm:privacy}, as a corollary of Theorem~\ref{thm:exact-privacy}.

\privacy*
\begin{proof}
Suppose that $\frac{b}{m}$ is sufficiently small. 
Let $\varepsilon > 0$ and $\delta \in (0,1)$ be such that $\varepsilon \leq \log{(1/\delta)}$. Finally consider $\Delta, \epsilon^*(\cdot), \epsilon_1(\cdot), \epsilon'(\cdot)$, and $\epsilon(\cdot)$ as defined in the statement and the proof of Theorem~\ref{thm:exact-privacy}. 
Below, we show that there exists $k>0$ such that, when $\sigma_{\mathrm{DP}} \geq k \cdot \nicefrac{2C}{b} \max{\{1, \,  \nicefrac{b \sqrt{T \log{(1/\delta)}}}{m\varepsilon}\}}$, Algorithm~\ref{algo:robust-dpsgd} ensures $(\varepsilon,\delta)$-DP towards an honest-but-curious server. 
First note that, when $\sigmadp \geq \nicefrac{2C}{b}$, we have
\begin{equation*}
    \varepsilon(2) = \frac{\Delta^2}{\sigmadp^2} = \frac{(\nicefrac{2C}{b})^2}{\sigmadp^2} \leq 1.
\end{equation*}
Since $h \coloneqq x \mapsto \frac{1}{x}(e^x-1)$ is non-decreasing on $(0,+\infty)$, this also implies that $\frac{1}{\varepsilon(2)}(e^{\varepsilon(2)}-1) = h{(\varepsilon(2))} \leq h(1) = e-1 \leq 2$.
As a result, we have 
\begin{equation}
\label{eqn:bounddelta}
    \min{\left \{ 4(e^{\varepsilon(2)}-1), 2e^{\varepsilon(2)}\right\}} \leq 4(e^{\varepsilon(2)}-1) \leq 8\,  \varepsilon{(2)}.
\end{equation}
Recall that 
\begin{equation}
    \varepsilon'(\alpha) = \frac{1}{\alpha-1}\log
    \Big( 1 + r^2 \binom{\alpha}{2} \min{\left \{ 4(e^{\varepsilon(2)}-1), 2e^{\varepsilon(2)} \right \}} 
    + 2\sum_{j=3}^\alpha r^j \binom{\alpha}{j} e^{(j-1)\varepsilon(j)} \Big).
\end{equation}
Therefore, since we assume that $\frac{b}{m}$ is sufficiently small ($r \ll 1$), the dominating term inside the logarithm is the term in $r^2$.
Using $\log{(1+x)} \leq x$, there exists a constant $k'$ such that
\begin{align*}
    \varepsilon'{(\alpha)} 
    &\leq \frac{1}{\alpha-1}
    \left( r^2 \binom{\alpha}{2} \min{\left \{ 4(e^{\varepsilon(2)}-1), 2e^{\varepsilon(2)}\right\}}
    + 2\sum_{j=3}^\alpha r^j \binom{\alpha}{j} e^{(j-1)\varepsilon(j)} \right) \\
    &\leq \frac{k'}{\alpha-1}\left( r^2 \binom{\alpha}{2} \min{\left \{ 4(e^{\varepsilon(2)}-1), 2e^{\varepsilon(2)}\right\}} \right) \\
    &=\frac{k'}{\alpha-1}\mathcal{O}{\left( r^2 \alpha (\alpha -1) \min{\left \{ 4(e^{\varepsilon(2)}-1), 2e^{\varepsilon(2)}\right\}} \right)}.
\end{align*}
Hence substituting from~\eqref{eqn:bounddelta}, we get
\begin{align*}
    \varepsilon'{(\alpha)} 
    \leq 8k' r^2 \alpha \varepsilon(2)
    = 8k' r^2 \frac{\Delta^2}{\sigmadp^2} \alpha.
\end{align*}
This directly implies that
\begin{align}
\label{eq:epsilon-1}
\varepsilon_1(\alpha) &=
(1-\alpha+\floor{\alpha})\frac{\floor{\alpha}-1}{\alpha-1} \varepsilon'(\floor{\alpha})
+ (\alpha-\floor{\alpha})\frac{\ceil{\alpha}-1}{\alpha-1}\varepsilon'(\ceil{\alpha}) \nonumber\\
&\leq 8k' r^2 \frac{\Delta^2}{\sigmadp^2} \Big[ (1-\alpha+\floor{\alpha})\frac{\floor{\alpha}-1}{\alpha-1} \floor{\alpha}
+ (\alpha-\floor{\alpha})\frac{\ceil{\alpha}-1}{\alpha-1} \ceil{\alpha} \Big].
\end{align}
Now, recall that $\alpha-1 \leq \floor{\alpha} \leq \alpha$ and $\alpha \leq \ceil{\alpha} \leq \alpha+1$.
We will prove that $\varepsilon_1(\alpha) \leq 32k' r^2 \frac{\Delta^2}{\sigmadp^2}$ by distinguishing two cases:

\underline{\textit{Case $\alpha \in (1,2)$:}}
Since $\alpha>1$, we have $\floor{\alpha} \geq 1$ and therefore $\nicefrac{\alpha - \floor{\alpha}}{\alpha - 1} \leq 1$.
We therefore have from Equation~\eqref{eq:epsilon-1}
\begin{align*}
\varepsilon_1(\alpha)
&\leq 8k' r^2 \frac{\Delta^2}{\sigmadp^2} \Big[ (1-\alpha+\floor{\alpha})\frac{\floor{\alpha}-1}{\alpha-1} \floor{\alpha}
+ (\alpha-\floor{\alpha})\frac{\ceil{\alpha}-1}{\alpha-1} \ceil{\alpha} \Big] \\
&\leq 8k' r^2 \frac{\Delta^2}{\sigmadp^2} \Big[ \underbrace{(1-\alpha+\floor{\alpha})}_{\leq 1}\underbrace{\frac{\floor{\alpha}-1}{\alpha-1}}_{\leq 1}\floor{\alpha}
+ \underbrace{(\ceil{\alpha}-1)}_{\leq \alpha} \ceil{\alpha} \Big] \\
&\leq 8k' r^2 \frac{\Delta^2}{\sigmadp^2} \Big[\floor{\alpha}
+ \alpha \ceil{\alpha} \Big]
\underset{(i)}{\leq} 8k' r^2 \frac{\Delta^2}{\sigmadp^2} \Big[\alpha
+ 2 \alpha \Big]
= 24k' r^2 \frac{\Delta^2}{\sigmadp^2} \alpha,
\end{align*}
where $(i)$ is due to $\ceil{\alpha} \leq 2$ because $\alpha < 2$.

\underline{\textit{Case $\alpha \in [2,+\infty)$:}}

Since $\alpha \geq 2$, we have both $\floor{\alpha} \leq\ceil{\alpha} \leq \alpha +1 \leq 2 \alpha$ and $\floor{\alpha}-1 \leq \ceil{\alpha}-1 \leq 2(\alpha-1)$.
Therefore, we have from Equation~\eqref{eq:epsilon-1} that
\begin{align*}
\varepsilon_1(\alpha) 
&\leq 8k' r^2 \frac{\Delta^2}{\sigmadp^2} \Big[ (1-\alpha+\floor{\alpha})\frac{\floor{\alpha}-1}{\alpha-1} \floor{\alpha}
+ (\alpha-\floor{\alpha})\frac{\ceil{\alpha}-1}{\alpha-1} \ceil{\alpha} \Big] \\
&\leq 8k' r^2 \frac{\Delta^2}{\sigmadp^2} \Big[ (1-\alpha+\floor{\alpha})4\alpha
+ (\alpha-\floor{\alpha})4 \alpha \Big]
= 32k' r^2 \frac{\Delta^2}{\sigmadp^2}.
\end{align*}

We have now proved for every $\alpha > 1$ that $\varepsilon_1(\alpha) \leq 32k' r^2 \frac{\Delta^2}{\sigmadp^2}$.
This implies that
\begin{align*}
    \varepsilon^* 
    &= \inf_{\alpha > 1} \left( T \varepsilon_1{(\alpha)} + \frac{\log{(1/\delta)}}{\alpha-1} \right)
    \leq \inf_{\alpha>1} \left( 32k' r^2 \frac{\Delta^2}{\sigmadp^2} \alpha T + \frac{\log{(1/\delta)}}{\alpha-1} \right).
\end{align*}

The above (convex) optimization problem is solved for $\alpha = \alpha^* \coloneqq 1+\sigmadp \sqrt{\frac{\log{(1/\delta)}}{32k' r^2 \Delta^2 T}}$.
Remark that the constraint $\alpha > 1$ is satisfied at $\alpha^*$.
Additionally, the objective at $\alpha = \alpha^*$ is equal to
\begin{align*}
    32k' r^2 \frac{\Delta^2}{\sigmadp^2} \alpha^* T + \frac{\log{(1/\delta)}}{\alpha^*-1}
    &= 32k' r^2 \frac{\Delta^2}{\sigmadp^2} T + 2r\Delta\frac{\sqrt{32k' \, T \log{(1/\delta)}}}{\sigmadp}.
\end{align*}
Therefore, using the assumption $\varepsilon \leq \log{(1/\delta)}$, when $\sigmadp \geq \frac{6 C\sqrt{32k' T \log{(1/\delta)}}}{m \varepsilon} = 3 r \Delta \frac{\sqrt{32k' T \log{(1/\delta)}}}{\varepsilon}$, we have
\begin{align*}
    \varepsilon^* 
    &\leq 32k' r^2 \frac{\Delta^2}{\sigmadp^2} T + 2r\Delta\frac{\sqrt{32k' \, T \log{(1/\delta)}}}{\sigmadp} \\
    &\leq \frac{\varepsilon^2}{9 \log{(1/\delta)}} + \nicefrac{2}{3}\, \varepsilon \leq (\nicefrac{1}{9}+\nicefrac{2}{3})\varepsilon \leq \varepsilon.
\end{align*}
Recall that to derive this last inequality, we overall needed $\sigmadp \geq \nicefrac{2C}{b} = \Delta$ and $\sigmadp \geq \frac{6C\sqrt{32k' T \log{(1/\delta)}}}{m \varepsilon} = 3 r \Delta \frac{\sqrt{32k' T \log{(1/\delta)}}}{\varepsilon}$.
Therefore, by choosing $k \coloneqq \max{\{1,3\sqrt{32k'}\}}$, we can now conclude that, when $\sigma_{\mathrm{DP}} \geq k \cdot \nicefrac{2C}{b} \max{\{1, \,  \nicefrac{b \sqrt{T \log{(1/\delta)}}}{m\varepsilon}\}}$, Algorithm~\ref{algo:robust-dpsgd} is $(\varepsilon,\delta)$-DP.
\end{proof}

\clearpage
\section{Upper Bounds}
\label{app:ub}

\subsection{Proof Outline}
\label{sec:resultsskeleton}
Our analysis of \algoname{} (\cref{algo:robust-dpsgd}), inspired from \cite{farhadkhani2022byzantine},  consists of three elements:
(i) {\em Momentum drift} (Lemma~\ref{lem:drift}),
(ii) {\em Momentum deviation} (Lemma~\ref{lem:dev}), and
(iii) {\em Descent bound} (Lemma~\ref{lem:descent}).
We combine these elements to obtain the final convergence result stated in Theorem~\ref{thm:convergence}, and the matching upper bound stated in \cref{cor:tradeoff}.

\noindent \paragraph{Notation.} Recall that for each step $t$, for each honest worker $w_i$,
\begin{align}
    &m_t^{(i)} = \beta_{t-1} m_{t-1}^{(i)} + (1-\beta_{t-1}) \tilde{g}_t^{(i)} \label{eqn:mmt_i},\\
    &\tilde{g}_t^{(i)} = g_t^{(i)} + \xi_t^{(i)} ; \, \, \,  \xi_t^{(i)} \sim \N{(0,\sigmadp^2 I_d)},
    \label{eq:noise-gradient}
\end{align}
where we initialize $m_0^{(i)} = 0$. As we analyze \cref{algo:robust-dpsgd} with aggregation $F$, we denote 
\begin{align}
    &R_t \coloneqq F{\left(m_t^{(1)}, \ldots, m_t^{(n)} \right)}, \label{eqn:R}\\
    &\theta_{t+1} = \theta_{t} - \gamma_t R_t. \label{eqn:SGD}
\end{align}
Throughout, we denote the loss function over dataset $\D_i$ by $\loss_i = \loss{(\cdot~; \D_i)}$.
Also, we denote by $\mathcal{P}_t$ the history from steps $0$ to $t$. Specifically, 
\[\P_t \coloneqq \left\{\weight{0}, \ldots, \, \weight{t}; ~ \mmt{i}{1}, \ldots, \, \mmt{i}{t-1}; i \in [n] \right\}.\] 
By convention, $\P_1 = \{ \weight{0}\}$. We denote by $\condexpect{t}{\cdot}$ and $\expect{\cdot}$ the conditional expectation $\expect{\cdot ~ \vline ~ \P_t}$ and the total expectation, respectively. 
Thus, $\expect{\cdot} = \condexpect{1}{ \cdots \condexpect{T}{\cdot}}$.

\subsubsection{Momentum Drift}
\label{sec:drift}

Along the trajectory $\theta_0,\ldots,\theta_{t}$, the honest workers' local momentums may drift away from each other.
The drift has three distinct sources: 
(i) noise injected by the DP mechanism,
(ii) gradient dissimilarity induced by data heterogeneity,
and (iii) stochasticity of the mini-batch gradients.
The aforementioned drift of local momentums can be exploited by the Byzantine adversaries to maliciously bias the aggregation output.

In this section, we will control the growth of the drift $\Delta_t$ between momentums, which we define as
\begin{equation}
    \Delta_t \coloneqq 
    \lambda_{\max}{\left(\frac{1}{\card{\H}} \sum_{i \in \H}(m_t^{(i)} - \overline{m}_t)(m_t^{(i)} - \overline{m}_t)^\top\right)},
    \label{eq:defdrift}
\end{equation}
where $\lambda_{\max}$ denotes the maximum eigenvalue, and $\overline{m}_t \coloneqq \frac{1}{\card{\H}} \sum_{i \in \H} m^{(i)}_t$ denotes the average honest momentum.
We show in Lemma~\ref{lem:drift} below that the growth of the drift $\Delta_t$ of the momentums can be controlled by tuning the momentum coefficient $\beta_t$.
The full proof can be found in \cref{app:lem-drift}.

\begin{restatable}{lemma}{globaldrift}
\label{lem:drift}
Suppose that assumptions \ref{asp:bnd_var} and \ref{asp:bnd_norm} hold. Consider Algorithm~\ref{algo:robust-dpsgd}. For every $t \in \{0, \ldots, T-1\}$, we have
    \begin{align*}
        \expect{\Delta_{t+1}}
    \leq \beta_t \expect{\Delta_t}
    +2(1-\beta_t)^2\left(\sigma_b^2+36\sigmadp^2(1+\frac{d}{n-f})\right)
    + (1-\beta_t)\gcov^2,
    \end{align*}
    where $\overline{m}_t \coloneqq \frac{1}{\card{\H}} \sum_{i \in \H} m^{(i)}_t$,
$\sigma_b^2 \coloneqq 2(1-\frac{b}{m})\frac{\sigma^2}{b}$,
and $\gcov^2 \coloneqq \sup_{\theta \in \R^d} \sup_{\norm{v}\leq 1} \frac{1}{\card{\H}}\sum_{i \in \H}\iprod{v}{\nabla{\loss_i{(\theta)}} - \nabla{\loss_{\H}{(\theta)}}}^2$.
\end{restatable}

The dimension factor $d$ due to DP noise is divided by $n-f$, which would not have been possible without leveraging the Gaussian nature of the noise. This dependence will prove crucial to match our lower bound.
To leverage Gaussianity, we use a concentration argument on the empirical covariance matrix of Gaussian random variables, stated in \cref{lem:concentration}.

The remaining term $\gcov^2$ of the upper bound is only due to data heterogeneity.
An important distinction from \cite{karimireddy2022byzantinerobust} is that $\gcov^2$ is a tighter bound on heterogeneity, compared to $G^2$ the bound on the average squared distance from \cref{asp:hetero}.
This is because the drift $\Delta_t$ is not an average squared distance, but rather a bound on average squared distances of every projection on the unit ball.
Controlling this quantity requires a covering argument (stated in \cref{lem:cover}).

\subsubsection{Momentum Deviation}
\label{sec:deviation}

Next, we study the momentum deviation; i.e., the distance between the average honest momentum $\overline{m}_t$ and the true gradient $\nabla \loss_{\H}(\weight{t})$ in an arbitrary step $t$. Specifically, we define momentum {\em deviation} to be
\begin{align}
    \dev{t} \coloneqq \overline{m}_t - \nabla \loss_{\H}\left( \weight{t} \right). \label{eqn:dev}
\end{align}

Also, we introduce the error between the aggregate $R_t$  and $\overline{m}_t \coloneqq \frac{1}{\card{\H}} \sum_{i \in \H} m^{(i)}_t$ the average momentum of honest workers for the case. Specifically, when defining the error
\begin{align}
    \drift{t} \coloneqq R_t - \overline{m}_t, \label{eqn:drift}
\end{align}
we get the following bound on the momentum deviation in \cref{lem:dev}, proof of which can be found in \cref{app:lem-dev}.

\begin{restatable}{lemma}{deviation}
\label{lem:dev}
Suppose that assumptions~\ref{asp:bnd_var} and~\ref{asp:bnd_norm} hold and that $\loss_\H$ is $L$-smooth.
Consider \cref{algo:robust-dpsgd}. For all $t \in \{0, \ldots, T-1\}$, we have
\begin{align*}
    \expect{\norm{\dev{t+1}}^2} 
    &\leq \beta_t^2 (1 + \gamma_t L )(1 + 4 \gamma_t L) \expect{\norm{\dev{t}}^2} +
    4 \gamma_t L ( 1 + \gamma_t L) \beta_t^2  \expect{\norm{\nabla \loss_{\H}(\weight{t})}^2} \\
    &\quad +(1 - \beta_t)^2 \frac{\sigmadpbar^2}{(n-f)}
    + 2 \gamma_t L ( 1 + \gamma_t L)\beta_t^2  \expect{\norm{\drift{t}}^2},
\end{align*}
where $\sigmadpbar^2 \coloneqq 2\left(1-\frac{b}{m}\right) \frac{\sigma^2}{b} + d \cdot \sigmadp^2$.
\end{restatable}

\subsubsection{Descent Bound} 
\label{sec:growth}

Finally, we bound the progress made at each learning step in minimizing the loss $\loss_{\H}$ using \cref{algo:robust-dpsgd}. 
From~\eqref{eqn:SGD} and~\eqref{eqn:R}, we obtain that, for each step $t$,
\begin{align*}
    \weight{t+1} = \weight{t} - \gamma_t  R_t
    = \weight{t} - \gamma_t \overline{m}_t - \gamma_t  (R_t - \overline{m}_t),
\end{align*}
Furthermore, by~\eqref{eqn:drift}, $R_t -   \, \overline{m}_t = \drift{t}$. Thus, for all $t$,
\begin{align}
    \weight{t+1} = \weight{t} - \gamma_t \overline{m}_t - \gamma_t \drift{t}. \label{eqn:sgd_new}
\end{align}
This means that \cref{algo:robust-dpsgd} can actually be treated as distributed SGD with a momentum term that is subject to perturbation proportional to $\drift{t}$ at each step $t$. This perspective leads us to Lemma~\ref{lem:descent}, proof of which can be found in Appendix~\ref{app:descent}. 

\begin{restatable}{lemma}{descent} 
\label{lem:descent}
Assume that $\loss_\H$ is $L$-smooth.
Consider \cref{algo:robust-dpsgd}. 
For any $t \in [T]$, we have
\begin{align*}
    \expect{\loss_{\H}(\weight{t+1}) - \loss_{\H}(\weight{t})} \leq 
    & - \frac{\gamma_t}{2}    \left( 1 - 4 \gamma_t  L  \right) \expect{\norm{\nabla \loss_{\H}(\weight{t})}^2}  + 
    \gamma_t    \left( 1 + 2 \gamma_t  L   \right) \expect{ \norm{\dev{t}}^2}
    + \gamma_t  \left(  1 + \gamma_t  L \right) \expect{\norm{\drift{t}}^2}.
\end{align*}
\end{restatable}

Putting all of the previous lemmas together, we prove Theorem~\ref{thm:convergence} in Section~\ref{app:convergence}. We then prove Corollary~\ref{cor:tradeoff} in Section~\ref{app:cor-strongly-convex}, and its non-convex version in Corollary~\ref{cor:tradeoff-nonconvex} in Section~\ref{app:cor-nonconvex}.

\subsection{Proof of Theorem~\ref{thm:convergence}}
\label{app:convergence}
We recall the theorem statement below for convenience. Recall that
\begin{align*}
    \loss_* = \inf_{\weight{} \in \R^d} \loss_{\H}(\weight{}),
    \loss_0 = \loss_{\H}{(\theta_0)} - \loss_*,
    a_1= 240, a_2 = 480,
    a_3 = 5760,
    \text{ and } a_4 = 270. 
\end{align*}

\convergence*

We prove Theorem~\ref{thm:convergence} in the strongly convex case in Section~\ref{app:strongly-convex}, and in the non-convex case in Section~\ref{app:nonconvex}.

\subsubsection{Strongly Convex Case}
\label{app:strongly-convex}
\begin{proof}
Let \cref{asp:bnd_var} hold and assume that $\loss_\H$ is $L$-smooth and $\mu$-strongly convex, and that $F$ is a $(f,\kappa)$-robust averaging aggregation rule.
Let $t \in \{0, \ldots, T-1\}$.
We set the learning rate and momentum schedules to be
\begin{equation}
    \gamma_t = \frac{10}{\mu{(t+a_1\frac{L}{\mu})}},\
    \beta_t = 1-24L \gamma_t,
\end{equation}
where $a_1 \coloneqq 240$.
Note that we have
\begin{equation}
    \gamma_t \leq \gamma_0 = \frac{10}{\mu 240 \frac{L}{\mu}} = \frac{1}{24L}.
    \label{eq:gamma_t}
\end{equation}

To obtain the convergence result we 
define the Lyapunov function to be
\begin{align}
    V_t \coloneqq \left(t+a_1\frac{L}{\mu}\right)^2\expect{\loss_{\H}(\weight{t}) - \loss_* + \frac{z_1}{L} \norm{\dev{t}}^2 + \kappa \cdot \frac{z_2}{L} \Delta_t}, \label{eqn:lyap_func}
\end{align}
where $a_1 = 240, z_1 = \frac{1}{16}$, and $z_2 = 2$.
Throughout the proof, we denote $\hat{t} \coloneqq t+a_1\frac{L}{\mu}$.
Therefore, we have $\gamma_t = \frac{10}{\mu \hat{t}}$.
Consider also the auxiliary sequence $W_t$ defined as
\begin{equation}
    W_t \coloneqq \expect{\loss_{\H}(\weight{t}) - \loss_* + \frac{z_1}{L} \norm{\dev{t}}^2 + \kappa \cdot \frac{z_2}{L} \Delta_t}.
    \label{eqn:lyap_func_aux}
\end{equation}
Therefore, we have 
\begin{align}
    V_{t+1}-V_t
    &= (\hat{t}+1)^2 W_{t+1} - \hat{t}^2 W_t
    = (\hat{t}+1)^2 W_{t+1} - (\hat{t}^2 + 2\hat{t} + 1) W_t + (2\hat{t}+1)W_t \nonumber\\
    &= (\hat{t}+1)^2 (W_{t+1} - W_{t}) + (2\hat{t}+1)W_t.
    \label{eq:lyap-to-aux}
\end{align}
We now bound the quantity $W_{t+1}-W_t$ below.

{\bf Invoking Lemma~\ref{lem:drift}.}
Upon substituting from Lemma~\ref{lem:drift}, we obtain
\begin{align}
    \expect{ \kappa \cdot \frac{z_2}{L} \Delta_{t+1} - \kappa \cdot \frac{z_2}{L} \Delta_t} 
    &\leq \kappa \cdot \frac{z_2}{L} \beta_t \expect{\Delta_t}
    +2 \kappa \cdot \frac{z_2}{L}(1-\beta_t)^2\left(\sigma_b^2+36\sigmadp^2(1+\frac{d}{n-f})\right)
    + \kappa \cdot \frac{z_2}{L} (1-\beta_t)G_{cov}^2 \nonumber \\
    &\quad - \kappa \cdot \frac{z_2}{L} \expect{\Delta_t}. \label{eqn:drift_gamma}
\end{align}

{\bf Invoking Lemma~\ref{lem:dev}.} Upon substituting from Lemma~\ref{lem:dev}, we obtain 
\begin{align}
    \expect{ \frac{z_1}{L} \norm{\dev{t+1}}^2 - \frac{z_1}{L} \norm{\dev{t}}^2} 
    &\leq \frac{z_1}{L} \beta_t^2 c_t \expect{\norm{\dev{t}}^2} +  4 z_1 \gamma_t ( 1 + \gamma_t L) \beta_t^2   \expect{\norm{\nabla \loss_{\H}(\weight{t})}^2} + \frac{z_1}{L} (1 - \beta_t)^2 \frac{\sigmadpbar^2}{n-f} \nonumber \\
    &\quad + 2 z_1 \gamma_t ( 1 + \gamma_t L)\beta_t^2 \expect{\norm{\drift{t}}^2} - \frac{z_1}{L} \expect{\norm{\dev{t}}^2}, \label{eqn:dev_gamma}
\end{align}
where we introduced the following quantity for simplicity
\begin{align}
    c_t = (1 + \gamma_t L) \left(1 + 4 \gamma_t  L \right) = 1 + 5 \gamma_t L + 4 \gamma_t^2  L^2. \label{eqn:zeta_expand}
\end{align}

{\bf Invoking Lemma~\ref{lem:descent}.} Substituting from Lemma~\ref{lem:descent}, we obtain
\begin{align}
    \expect{\loss_{\H}(\weight{t+1}) - \loss_{\H}(\weight{t})}
    &\leq - \frac{\gamma_t}{2}\left( 1 - 4 \gamma_t  L \right) \expect{\norm{\nabla \loss_{\H}(\weight{t})}^2} 
    + \gamma_t   \left( 1 + 2 \gamma_t  L  \right) \expect{\norm{\dev{t}}^2}
    + \gamma_t  \left(1 + \gamma_t  L \right) \expect{\norm{\drift{t}}^2}. \label{eqn:growth_gamma}
\end{align}
Substituting from \eqref{eqn:drift_gamma}, ~\eqref{eqn:dev_gamma} and~\eqref{eqn:growth_gamma} in~\eqref{eqn:lyap_func_aux}, we obtain
\begin{align}
    W_{t+1} - W_t &= \expect{\loss_{\H}(\weight{t+1}) - \loss_{\H}(\weight{t})} + \expect{ \frac{z_1}{L} \norm{\dev{t+1}}^2 - \frac{z_1}{L} \norm{\dev{t}}^2}
    + \expect{ \kappa \cdot \frac{z_2}{L} \Delta_{t+1} - \kappa \cdot \frac{z_2}{L} \Delta_t}\nonumber \\
    &\leq - \frac{\gamma_t}{2}\left( 1 - 4 \gamma_t  L \right) \expect{\norm{\nabla \loss_{\H}(\weight{t})}^2} 
    + \gamma_t   \left( 1 + 2 \gamma_t  L  \right) \expect{\norm{\dev{t}}^2}
    + \gamma_t  \left(  1 + \gamma_t  L \right) \expect{\norm{\drift{t}}^2} \nonumber \\
    &\quad +  \frac{z_1}{L} \beta_t^2 c_t \expect{\norm{\dev{t}}^2} +  4 z_1 \gamma_t ( 1 + \gamma_t L) \beta_t^2   \expect{\norm{\nabla \loss_{\H}(\weight{t})}^2} + \frac{z_1}{L} (1 - \beta_t)^2 \frac{\sigmadpbar^2}{n-f} \nonumber \\
    &\quad + 2 z_1 \gamma_t ( 1 + \gamma_t L)\beta_t^2 \expect{\norm{\drift{t}}^2} - \frac{z_1}{L} \expect{\norm{\dev{t}}^2}\nonumber\\
    &\quad +  \kappa \cdot \frac{z_2}{L} \beta_t \expect{\Delta_t}
    +2 \kappa \cdot \frac{z_2}{L}(1-\beta_t)^2\left(\sigma_b^2+36\sigmadp^2(1+\frac{d}{n-f})\right)
    + \kappa \cdot \frac{z_2}{L} (1-\beta_t)G_{cov}^2 \nonumber \\
    &\quad - \kappa \cdot \frac{z_2}{L} \expect{\Delta_t}.
    \label{eqn:Vt-t}
\end{align}
Upon rearranging the R.H.S.~in~\eqref{eqn:Vt-t} we obtain that
\begin{align}
    W_{t+1} - W_t &\leq - \frac{\gamma_t}{2} \left(  \left( 1 - 4 \gamma_t L \right) - 8 z_1( 1 + \gamma_t L) \beta_t^2   \right) \expect{\norm{\nabla \loss_{\H}(\weight{t})}^2} +  \frac{z_1}{L} (1 - \beta_t)^2 \frac{\sigmadpbar^2}{n-f} \nonumber \\
    &\quad - z_1 \gamma_t\left(-\frac{1}{z_1}\left( 1 + 2 \gamma_t L  \right) -  \frac{1}{\gamma_t L} \beta_t^2 c_t + \frac{1}{\gamma_t L} \right)  \expect{\norm{\dev{t}}^2}  + \gamma_t \left( 1 + \gamma_t L + 2z_1 (1 + \gamma_t L) \beta_t^2 \right) \expect{\norm{\drift{t}}^2}\nonumber\\
    &\quad - \kappa \cdot \frac{z_2}{L}(1-\beta_t) \expect{\Delta_t} +2 \kappa \cdot \frac{z_2}{L}(1-\beta_t)^2\left(\sigma_b^2+36\sigmadp^2(1+\frac{d}{n-f})\right)
    + \kappa \cdot \frac{z_2}{L} (1-\beta_t)G_{cov}^2.
    \label{eq:before-kappa}
\end{align}
Since we assume $F$ to be $(f,\kappa)$-robust averaging, we can bound $\expect{\norm{\drift{t}}^2}$ as follows.
Starting from the definition of $\drift{t}$, we have
\begin{align*}
    \norm{\drift{t}}^2 
    &= \norm{R_t - \overline{m}_t}^2
    = \norm{F{(m_t^{(1)}, \ldots, m_t^{(n)})} - \overline{m}_t}^2
    \leq \kappa \cdot \lambda_{\max}{\left(\frac{1}{\card{\H}} \sum_{i \in \H}(m_t^{(i)} - \overline{m}_t)(m_t^{(i)} - \overline{m}_t)^\top\right)}
    = \kappa \cdot \Delta_t.
\end{align*}
Then taking total expectations above gives the bound 
\begin{equation*}
    \expect{\norm{\drift{t}}^2} \leq \kappa \cdot \expect{\Delta_t}.
\end{equation*}
Using the bound above in \cref{eq:before-kappa}, and then rearranging terms, yields
\begin{align*}
    W_{t+1} - W_t &\leq - \frac{\gamma_t}{2} \left(  \left( 1 - 4 \gamma_t L \right) - 8 z_1( 1 + \gamma_t L) \beta_t^2   \right) \expect{\norm{\nabla \loss_{\H}(\weight{t})}^2} +  \frac{z_1}{L} (1 - \beta_t)^2 \frac{\sigmadpbar^2}{n-f} \nonumber \\
    &\quad - z_1 \gamma_t\left(-\frac{1}{z_1}\left( 1 + 2 \gamma_t L  \right) -  \frac{1}{\gamma_t L} \beta_t^2 c_t + \frac{1}{\gamma_t L} \right)  \expect{\norm{\dev{t}}^2}  + \kappa \gamma_t \left( 1 +  \gamma_t L + 2z_1 (1 + \gamma_t L) \beta_t^2 \right) \expect{\Delta_t}\nonumber\\
    &\quad - \kappa \cdot \frac{z_2}{L}(1-\beta_t) \expect{\Delta_t} +2 \kappa \cdot \frac{z_2}{L}(1-\beta_t)^2\left(\sigma_b^2+36\sigmadp^2(1+\frac{d}{n-f})\right)
    + \kappa \cdot \frac{z_2}{L} (1-\beta_t)G_{cov}^2 \nonumber\\
    &= - \frac{\gamma_t}{2} \left(  \left( 1 - 4 \gamma_t L \right) - 8 z_1( 1 + \gamma_t L) \beta_t^2   \right) \expect{\norm{\nabla \loss_{\H}(\weight{t})}^2} +  \frac{z_1}{L} (1 - \beta_t)^2 \frac{\sigmadpbar^2}{n-f} \nonumber \\
    &\quad - z_1 \gamma_t\left(-\frac{1}{z_1}\left( 1 + 2 \gamma_t L  \right) -  \frac{1}{\gamma_t L} \beta_t^2 c_t + \frac{1}{\gamma_t L} \right)  \expect{\norm{\dev{t}}^2} \nonumber\\
    &\quad - \kappa z_2 \gamma_t \left(\frac{1}{\gamma_t L}(1-\beta_t) - \frac{1}{z_2}\left( 1 +  \gamma_t L + 2z_1 (1 + \gamma_t L) \beta_t^2 \right) \right)\expect{\Delta_t}\nonumber\\ 
    &\quad +2 \kappa \cdot \frac{z_2}{L}(1-\beta_t)^2\left(\sigma_b^2+36\sigmadp^2(1+\frac{d}{n-f})\right)
    + \kappa \cdot \frac{z_2}{L} (1-\beta_t)G_{cov}^2.
\end{align*}

For simplicity, we define
\begin{align}
    A \coloneqq \frac{1}{2}\left( 1 - 4 \gamma_t L \right) - 8 z_1( 1 + \gamma_t L) \beta_t^2 , \label{eqn:def_At}
\end{align}
\begin{align}
    B \coloneqq -\frac{1}{z_1}\left( 1 + 2 \gamma_t L  \right) -  \frac{1}{\gamma_t L} \beta_t^2 c_t + \frac{1}{\gamma_t L}, \label{eqn:def_Bt}
\end{align}
and
\begin{align}
    C \coloneqq \frac{1}{\gamma_t L}(1-\beta_t) - \frac{1}{z_2}\left( 1 +  \gamma_t L + 2z_1 (1 + \gamma_t L) \beta_t^2 \right), \label{eqn:def_Ct}
\end{align}
Denote also
\begin{align*}
    \sigmabar^2 \coloneqq \frac{\sigma_b^2 + d \sigmadp^2}{n-f}
    +4 \kappa\left(\sigma_b^2+36\sigmadp^2(1+\frac{d}{n-f})\right).
\end{align*}
Recall that, as $z_1 = \frac{1}{16}$ and $z_2 = 2$, and $\sigmadpbar^2 = \sigma_b^2 + d \sigmadp^2$, we have
\begin{align*}
    \sigmabar^2 
    \geq z_1 \frac{\sigmadpbar^2}{n-f}
    +2 \kappa \cdot z_2\left(\sigma_b^2+36\sigmadp^2(1+\frac{d}{n-f})\right).
\end{align*}
Thus, substituting the above variables, we obtain
\begin{align}
    W_{t+1} - W_t 
    &\leq - A \gamma_t \expect{\norm{\nabla \loss_{\H}(\weight{t})}^2} - z_1 B \gamma_t \expect{\norm{\dev{t}}^2}  - \kappa \cdot z_2 C \gamma_t  \expect{\Delta_t} \nonumber\\ 
    &\quad +  \frac{1}{L} (1-\beta_t)^2 \sigmabar^2
    + \kappa \cdot \frac{z_2}{L} (1-\beta_t)G_{cov}^2. \label{eqn:after_At}
\end{align}
We now analyze below the terms $A$, $B$ and $C$ on the RHS of \eqref{eqn:after_At}.

{\bf Term $A$.} Recall from~\eqref{eq:gamma_t} that $ \gamma_t \leq \frac{1}{24L} $. Upon using this in~\eqref{eqn:def_At}, and the facts that $z_1 = \frac{1}{16}$ and $\beta_t^2 \leq 1$,  we obtain that
\begin{align}
    A \geq \frac{1}{2}\left( 1 - 4 \gamma_t L \right) - 8 z_1( 1 + \gamma_t L) 
    \geq \frac{1}{2}(1 - 4 \times \frac{1}{24})  - \frac{8}{16} ( 1 + \frac{1}{24} ) 
    \geq \frac{1}{10}. \label{eqn:At_3}
\end{align}

{\bf Term $B$.} 
Substituting $c_t$ from~\eqref{eqn:zeta_expand} in~\eqref{eqn:def_Bt} we obtain that
\begin{align*}
    B &= -\frac{1}{z_1}\left( 1 + 2\gamma_t L \right) - \frac{1}{\gamma_t L} \beta_t^2 \left( 1 + 5 \gamma_t L + 4 \gamma_t^2  L^2 \right) + \frac{1}{\gamma_t L} \\
    & =  \frac{1}{\gamma_t L}\left(1 - \beta^2 \right) 
    -\frac{1}{z_1} \left( 1 + 2\gamma_t L +5 z_1 \beta_t^2 + 4 z_1 \beta_t^2\gamma_t L \right).
\end{align*}
Using the facts that $\beta_t \leq 1$ and $\gamma_t \leq \frac{1}{24L}$, and then substituting  $z_1 = \frac{1}{16}$ we obtain
\begin{align}
    B &\geq \frac{1}{\gamma_t L}(1-\beta_t^2) - 16 \left(1 + \frac{2}{24} + \frac{5}{16} + \frac{4}{24 \times 16} \right)  
    \geq \frac{1}{\gamma_t L}(1-\beta_t^2) - 23
    \geq  \frac{1}{\gamma_t L}(1-\beta_t) - 23 = 1.
    \label{eqn:Bt_2} 
\end{align}
where the last equality follows from the fact that $1 - \beta_t = 24 \gamma_t L$.

{\bf Term $C$.} Substituting $z_1 = \frac{1}{16}, z_2 = 2$ in~\eqref{eqn:def_Ct}, and then using the facts that $\beta_t \leq 1$ and $\gamma_t \leq \frac{1}{24L}$, we obtain
\begin{align}
    C &= \frac{1}{\gamma_t L}(1-\beta_t) - \frac{1}{2}\left( 1 +  \gamma_t L + (2 \times 16) (1 + \gamma_t L) \beta_t^2 \right)
    \geq  \frac{1}{\gamma_t L}(1-\beta_t) - \frac{1}{2}\left( 1 +  \frac{1}{24} + 32 (1 + \frac{1}{24}) \right)\nonumber\\
    &\geq \frac{1}{\gamma_t L}(1-\beta_t) - 18
    = 6,
    \label{eqn:Ct_2}
\end{align}
where the last equality follows from the fact that $1 - \beta_t = 24 \gamma_t L$.

{\bf Combining terms $A$, $B$, and $C$.} Finally, substituting from~\eqref{eqn:At_3},~\eqref{eqn:Bt_2}, and ~\eqref{eqn:Ct_2} in~\eqref{eqn:after_At} (and recalling that $z_2 = 2$) we obtain that
\begin{align}
    W_{t+1} - W_t &\leq - \frac{\gamma_t}{10}\expect{\norm{\nabla \loss_{\H}(\weight{t})}^2} - z_1 \gamma_t \expect{\norm{\dev{t}}^2}  - 6\kappa z_2 \gamma_t  \expect{\Delta_t} \nonumber\\ 
    &\quad +  \frac{1}{L} (1-\beta_t)^2 \sigmabar^2
    + \kappa \cdot \frac{2}{L} (1-\beta_t)G_{cov}^2.
    \label{eq:bound-lyap-aux}
\end{align}
Since $\loss_{\H}$ is $\mu$-strongly convex, we have~\cite{karimi2016linear} for any $\theta \in \R^d$ that
\begin{equation}
    \norm{\nabla \loss_{\H}(\theta)}^2 \geq 2\mu(\loss{(\theta)} - \loss_*).
    \label{eq:pl-inequality}
\end{equation}
Plugging \eqref{eq:pl-inequality} in \eqref{eq:bound-lyap-aux} above, and then recalling that $L \geq \mu$, yields
\begin{align*}
    W_{t+1} - W_t &\leq - \frac{\mu \gamma_t}{5}\expect{\loss_{\H}(\weight{t}) - \loss_*} - z_1 \gamma_t \expect{\norm{\dev{t}}^2}  - 6\kappa z_2 \gamma_t  \expect{\Delta_t} \nonumber\\ 
    &\quad +  \frac{1}{L} (1-\beta_t)^2 \sigmabar^2
    + \kappa \cdot \frac{2}{L} (1-\beta_t)G_{cov}^2 \nonumber\\
    &\leq - \frac{\mu \gamma_t}{5}\expect{\loss_{\H}(\weight{t}) - \loss_* + \frac{z_1}{\mu} \norm{\dev{t}}^2 + \kappa \cdot \frac{z_2}{\mu} \Delta_t}
    +  \frac{1}{L} (1-\beta_t)^2 \sigmabar^2
    + \kappa \cdot \frac{2}{L} (1-\beta_t)G_{cov}^2
    \nonumber\\
    & \leq - \frac{\mu \gamma_t}{5}\expect{\loss_{\H}(\weight{t}) - \loss_* + \frac{z_1}{L} \norm{\dev{t}}^2 + \kappa \cdot \frac{z_2}{L} \Delta_t}
    +  \frac{1}{L} (1-\beta_t)^2 \sigmabar^2
    + \kappa \cdot \frac{2}{L} (1-\beta_t)G_{cov}^2
    \nonumber\\
    &=- \frac{\mu \gamma_t}{5} W_t
    +  \frac{1}{L} (1-\beta_t)^2 \sigmabar^2
    + \kappa \cdot \frac{2}{L} (1-\beta_t)G_{cov}^2.
\end{align*}
Upon plugging the above bound back in \cref{eq:lyap-to-aux}, rearranging terms and substituting $1-\beta_t = 24L\gamma_t$, we obtain
\begin{align*}
    V_{t+1} - V_t 
    &\leq (\hat{t}+1)^2 \left[ - \frac{\mu \gamma_t}{5} W_t
    +  \frac{1}{L} (1-\beta_t)^2 \sigmabar^2
    + \kappa \cdot \frac{2}{L} (1-\beta_t)G_{cov}^2 \right] + (2\hat{t}+1)W_t \nonumber\\
    &= -\left[ (\hat{t}+1)^2 \frac{\mu \gamma_t}{5} - (2\hat{t}+1) \right]W_t + \frac{(\hat{t}+1)^2}{L}(24L\gamma_t)^2\sigmabar^2
    + \kappa \cdot \frac{2(\hat{t}+1)^2}{L} (24L\gamma_t)G_{cov}^2.
\end{align*}
Recall however that $\gamma_t = \frac{10}{\mu \hat{t}}$ as $\hat{t} = t+a_1\frac{L}{\mu}$.
Recall that we denote $a_1 = 24 \times 10 = 240$.
Substituting $\gamma_t$ above yields
\begin{align*}
    V_{t+1} - V_t 
    &\leq (\hat{t}+1)^2 \left[ - \frac{\mu \gamma_t}{5} W_t
    +  \frac{1}{L} (1-\beta_t)^2 \sigmabar^2
    + \kappa \cdot \frac{2}{L} (1-\beta_t)G_{cov}^2 \right] + (2\hat{t}+1)W_t \nonumber\\
    &= -\left[ 2\frac{(\hat{t}+1)^2}{\hat{t}} - (2\hat{t}+1) \right]W_t + a_1^2 L\frac{(\hat{t}+1)^2}{\mu^2 \hat{t}^2}\sigmabar^2
    + 2a_1 \kappa \cdot \frac{(\hat{t}+1)^2}{\mu \hat{t}}G_{cov}^2.
\end{align*}
Observe that $2\frac{(\hat{t}+1)^2}{\hat{t}} \geq 2(\hat{t}+1) > 2\hat{t}+1$, implying that the first term above is negative:
\begin{align*}
    V_{t+1} - V_t 
    &\leq a_1^2 L\frac{(\hat{t}+1)^2}{\mu^2 \hat{t}^2}\sigmabar^2
    + 2a_1 \kappa \cdot \frac{(\hat{t}+1)^2}{\mu \hat{t}}G_{cov}^2.
\end{align*}
Observe now that, as $\hat{t} = t+a_1\frac{L}{\mu} \geq a_1 = 240$ (because $L \geq \mu$), we have $(\hat{t}+1)^2 \leq (1+\frac{1}{240})^2 \hat{t}^2 \leq 2 \hat{t}^2$.
Plugging this bound in the inequality above gives
\begin{align*}
    V_{t+1} - V_t 
    &\leq \frac{2 a_1^2 L}{\mu^2}\sigmabar^2
    + 4a_1 \kappa \cdot \frac{\hat{t}}{\mu}G_{cov}^2.
\end{align*}
Therefore, we have for every $t \in \{0,\ldots,T-1\}$ that
\begin{align*}
    V_{t+1} - V_0 = \sum_{k=0}^{t} (V_{k+1}-V_k) \leq (t+1) \frac{2 a_1^2 L}{\mu^2}\sigmabar^2 + \left(\sum_{k=0}^{t}\hat{k}\right) \frac{4a_1 \kappa}{\mu}G_{cov}^2.
\end{align*}
Since $\sum_{k=0}^{t}\hat{k} = \sum_{k=0}^{t}(k+a_1\frac{L}{\mu}) = \sum_{k=0}^{t}k + a_1(t+1)\frac{L}{\mu} = \frac{t(t+1)}{2}+ a_1(t+1)\frac{L}{\mu}$, we obtain
\begin{align*}
    V_{t+1} - V_0 &= \sum_{k=0}^{t} (V_{k+1}-V_k) \leq (t+1) \frac{2 a_1^2 L}{\mu^2}\sigmabar^2 + \left(\frac{t(t+1)}{2}+ a_1(t+1)\frac{L}{\mu}\right) \frac{4a_1 \kappa}{\mu}G_{cov}^2 \nonumber\\
    &= (t+1) \frac{2 a_1^2 L}{\mu^2}\sigmabar^2 + (t+1)\left(\frac{t}{2}+ a_1\frac{L}{\mu}\right) \frac{4a_1 \kappa}{\mu}G_{cov}^2.
\end{align*}
However, recalling the definition \eqref{eqn:lyap_func} of $V_t$, we obtain
\begin{align*}
    (t+1+a_1\frac{L}{\mu})^2 \expect{\loss_{\H}{(\theta_{t+1})}-\loss_*} 
    \leq V_{t+1} \leq V_0 + (t+1) \frac{2 a_1^2 L}{\mu^2}\sigmabar^2 + (t+1)\left(\frac{t}{2}+ a_1\frac{L}{\mu}\right) \frac{4a_1 \kappa}{\mu}G_{cov}^2.
\end{align*}
By rearranging terms, and using the fact that $\frac{L}{\mu} \geq 1$, we then get
\begin{align}
    \expect{\loss_{\H}{(\theta_{t+1})}-\loss_*} 
    &\leq \frac{V_0}{(t+1+a_1\frac{L}{\mu})^2} + \frac{t+1}{(t+1+a_1\frac{L}{\mu})^2 } \frac{2 a_1^2 L \sigmabar^2}{\mu^2}  
    + \frac{(t+1)\left(\frac{t}{2}+ a_1\frac{L}{\mu}\right)}{(t+1+a_1\frac{L}{\mu})^2 }\frac{4a_1 \kappa}{\mu}G_{cov}^2 \nonumber\\
    &\leq \frac{V_0}{(t+1+a_1\frac{L}{\mu})^2} + \frac{1}{t+1+a_1\frac{L}{\mu}} \frac{2 a_1^2 L \sigmabar^2}{\mu^2}  
    + \frac{4a_1 \kappa}{\mu}G_{cov}^2.
    \label{eq:penultimate}
\end{align}
It remains to bound $V_0$.
By definition, we have
\begin{align*}
    V_0 = \left(a_1\frac{L}{\mu}\right)^2 \left[\loss_{\H}(\weight{0}) - \loss_* + \frac{z_1}{L}\norm{\dev{0}}^2 + \frac{z_2}{L}\Delta_0\right].
\end{align*}
 By definition of $\AvgMmt{t} = \frac{1}{\card{\H}} \sum_{i \in \H} m^{(i)}_t$ and the initializations $m_0^{(i)} = 0$ for all $i \in \H$, we have $\Delta_0 = \lambda_{\max}{\left(\frac{1}{\card{\H}} \sum_{i \in \H}(m_0^{(i)} - \overline{m}_0)(m_0^{(i)} - \overline{m}_0)^\top\right)} = 0$.
 Therefore, we have
 \begin{equation*}
     V_0 = \left(a_1\frac{L}{\mu}\right)^2 \left[\loss_{\H}(\weight{0}) - \loss_* + \frac{z_1}{L}\norm{\dev{0}}^2\right].
 \end{equation*}
 Moreover, by definition of $\dev{t}$ in~\eqref{eqn:dev},
 we obtain that
\begin{align*}
    \norm{\dev{0}}^2 = \norm{\AvgMmt{0} - \nabla \loss_{\H}(\weight{0})}^2
    = \norm{\nabla \loss_{\H}(\weight{0})}^2. 
\end{align*}
Recall that $\loss_\H$ is $L$-smooth. Thus, $\norm{\nabla \loss_{\H}(\weight{0})}^2 \leq 2L(\loss_{\H}{(\weight{0})}-\loss^*)$ (see~\cite{nesterov2018lectures}, Theorem 2.1.5).
Therefore, substituting $z_1 = \frac{1}{16}$, we have
\begin{equation*}
    V_0 \leq \left(a_1\frac{L}{\mu}\right)^2 \left[\loss_{\H}(\weight{0}) - \loss_* + \frac{2L}{16L}(\loss_{\H}(\weight{0}) - \loss_*)\right] = \leq \left(a_1\frac{L}{\mu}\right)^2\frac{9}{8}(\loss_{\H}(\weight{0}) - \loss_*) 
    \leq 2\left(a_1\frac{L}{\mu}\right)^2(\loss_{\H}(\weight{0}) - \loss_*).
\end{equation*}
Plugging the above bound back in \cref{eq:penultimate}, rearranging terms, and then recalling that $a_1 \frac{L}{\mu} \geq 0$, yields 
\begin{align*}
    \expect{\loss_{\H}{(\theta_{t+1})}-\loss_*}
    &\leq \frac{4a_1}{\mu} \kappa G_{cov}^2 +
    \frac{2 a_1^2 L \sigmabar^2}{\mu^2(t+1+a_1\frac{L}{\mu})} +
    \frac{2a_1 L^2(\loss_{\H}(\weight{0}) - \loss_*)}{\mu^2(t+1+a_1\frac{L}{\mu})^2}\\
    &\leq  \frac{4a_1}{\mu} \kappa G_{cov}^2 +
    \frac{2 a_1^2 L \sigmabar^2}{\mu^2(t+1)} +
    \frac{2a_1 L^2(\loss_{\H}(\weight{0}) - \loss_*)}{\mu^2(t+1)^2}.
\end{align*}

Specializing the inequality above for $t=T-1$ and denoting $\loss_0 \coloneqq \loss_{\H}(\weight{0}) - \loss_*$ proves the theorem:
\begin{align*}
    \expect{\loss_{\H}{(\theta_{T})}-\loss_*}
    &\leq \frac{4a_1}{\mu} \kappa G_{cov}^2 +
    \frac{2 a_1^2 L \sigmabar^2}{\mu^2 T} +
    \frac{2a_1^2 L^2 \loss_0}{\mu^2 T^2}.
\end{align*}
\end{proof}

\begin{remark}
In the proof of the strongly convex case of Theorem~\ref{thm:convergence} above, we do not need the function $\loss_\H$ to be $\mu$-strongly convex. In fact, it is sufficient for $\loss_\H$ to satisfy the $\mu$-PL inequality stated in \eqref{eq:pl-inequality}.
Accordingly, our results not only apply to smooth $\mu$-strongly convex functions, but more generally to the class of smooth $\mu$-PL functions, which may be non-convex~\cite{karimi2016linear}.
\end{remark}

\subsubsection{Non-convex Case}
\label{app:nonconvex}
\begin{proof}
Let \cref{asp:bnd_var} hold and assume that $\loss_\H$ is $L$-smooth, and that $F$ is a $(f,\kappa)$-robust averaging aggregation rule.
Let $t \in \{0, \ldots, T-1\}$.
We set the learning rate and momentum to constant as follows:
\begin{equation}
    \gamma_t = \gamma \coloneqq \min{\left\{\frac{1}{24L}, ~ \frac{\sqrt{a_4 \loss_0}}{2\sigmabar \sqrt{a_3 L T}}\right\}},~
    \beta_t = \beta \coloneqq 1-24L \gamma,
\end{equation}
where $a_1 \coloneqq 240$.
Note that we have
\begin{equation}
    \gamma_t = \gamma \leq \frac{1}{24L}.
    \label{eq:gamma_t_nonconvex}
\end{equation}

To obtain the convergence result we 
define the Lyapunov function to be
\begin{align}
    V_t \coloneqq \expect{\loss_{\H}(\weight{t}) - \loss_* + \frac{z_1}{L} \norm{\dev{t}}^2 + \kappa \cdot \frac{z_2}{L} \Delta_t}, \label{eqn:lyap_func_nonconvex}
\end{align}
where $z_1 = \frac{1}{16}$, and $z_2 = 2$.
Note that $V_t$ corresponds to the sequence $W_t$ defined in \cref{eqn:lyap_func_aux}, and analyzed in \cref{app:strongly-convex} under the assumption that $\gamma_t \leq \frac{1}{24L}$.
Since the latter holds by \cref{eq:gamma_t_nonconvex}, we directly apply the bound obtained in \cref{eq:bound-lyap-aux}:
\begin{align*}
    V_{t+1} - V_t &\leq - \frac{\gamma_t}{10}\expect{\norm{\nabla \loss_{\H}(\weight{t})}^2} - z_1 \gamma_t \expect{\norm{\dev{t}}^2}  - 6\kappa z_2 \gamma_t  \expect{\Delta_t} \nonumber\\ 
    &\quad +  \frac{1}{L} (1-\beta_t)^2 \sigmabar^2
    + \kappa \cdot \frac{2}{L} (1-\beta_t)G_{cov}^2.
\end{align*}
In turn, substituting $\gamma_t = \gamma, \beta_t = \beta$ and bounding the second and third terms on the RHS by zero, this implies that
\begin{align*}
    V_{t+1} - V_t &\leq - \frac{\gamma}{10}\expect{\norm{\nabla \loss_{\H}(\weight{t})}^2}
    +  \frac{1}{L} (1-\beta)^2 \sigmabar^2
    + \kappa \cdot \frac{2}{L} (1-\beta)G_{cov}^2.
\end{align*}
By rearranging terms and then averaging over $t \in \{0, \ldots, T-1\}$, we obtain
\begin{align*}
    \frac{1}{T} \sum_{t=0}^{T-1}\expect{\norm{\nabla \loss_{\H}(\weight{t})}^2} \leq \frac{10}{\gamma T} \sum_{t=0}^{T-1}(V_t-V_{t+1})+ \frac{10}{\gamma L} (1-\beta)^2 \sigmabar^2 + \kappa \cdot \frac{20}{\gamma L}(1-\beta)G_{cov}^2.
\end{align*}
We now substitute $\beta = 1-24\gamma L$.
Denoting $a_3 = 10 \times 24^2  = 5760, a_2 = 20 \times 24 = 480 $, we obtain
\begin{align}
    \frac{1}{T} \sum_{t=0}^{T-1}\expect{\norm{\nabla \loss_{\H}(\weight{t})}^2} 
    &\leq \frac{10}{\gamma T} \sum_{t=0}^{T-1}(V_t-V_{t+1})+ \frac{(10 \times 24^2)}{\gamma L} (\gamma L)^2 \sigmabar^2 + \kappa \cdot \frac{(20 \times 24)}{\gamma L}(\gamma L)G_{cov}^2\nonumber\\
    &= \frac{10}{\gamma T} (V_0-V_T)+ a_3 \gamma L \sigmabar^2 + a_2 \kappa G_{cov}^2.
    \label{eq:penultimate-nonconvex}
\end{align}
We now bound $V_0 - V_T$.
First recall that $V_T \geq 0$ as a sum of non-negative terms (see \eqref{eqn:lyap_func_nonconvex}).
Therefore, we have
\begin{align*}
    V_0 - V_T \leq V_0 = \loss_{\H}(\weight{0}) - \loss_* + \frac{z_1}{L}\norm{\dev{0}}^2 + \frac{z_2}{L}\Delta_0.
\end{align*}
 By definition of $\AvgMmt{t} = \frac{1}{\card{\H}} \sum_{i \in \H} m^{(i)}_t$ and the initializations $m_0^{(i)} = 0$ for all $i \in \H$, we have $\Delta_0 = \lambda_{\max}{\left(\frac{1}{\card{\H}} \sum_{i \in \H}(m_0^{(i)} - \overline{m}_0)(m_0^{(i)} - \overline{m}_0)^\top\right)} = 0$.
 Therefore, we have
 \begin{equation*}
     V_0 = \loss_{\H}(\weight{0}) - \loss_* + \frac{z_1}{L}\norm{\dev{0}}^2.
 \end{equation*}
 Moreover, by definition of $\dev{t}$ in~\eqref{eqn:dev},
 we obtain that
\begin{align*}
    \norm{\dev{0}}^2 = \norm{\AvgMmt{0} - \nabla \loss_{\H}(\weight{0})}^2
    = \norm{\nabla \loss_{\H}(\weight{0})}^2. 
\end{align*}
Recall that $\loss_\H$ is $L$-smooth. Thus, $\norm{\nabla \loss_{\H}(\weight{0})}^2 \leq 2L(\loss_{\H}{(\weight{0})}-\loss^*)$ (see~\cite{nesterov2018lectures}, Theorem 2.1.5).
Therefore, substituting $z_1 = \frac{1}{16}$, we have
\begin{equation*}
    V_0 - V_T \leq V_0 \leq \loss_{\H}(\weight{0}) - \loss_* + \frac{2L}{16L}(\loss_{\H}(\weight{0}) - \loss_*) 
    = \frac{9}{8}(\loss_{\H}(\weight{0}) - \loss_*).
\end{equation*} 
By plugging this bound back in \eqref{eq:penultimate-nonconvex}, and denoting $a_4 \coloneqq 24 \times 10 \times (\frac{9}{8}) = 270$ and $\loss_0 \coloneqq \loss_{\H}(\weight{0}) - \loss_*$, we obtain
\begin{align}
    \frac{1}{T} \sum_{t=0}^{T-1}\expect{\norm{\nabla \loss_{\H}(\weight{t})}^2} 
    &\leq \frac{10 \times (\frac{9}{8})}{\gamma T} (\loss_{\H}(\weight{0}) - \loss_*)+ a_3 \gamma L \sigmabar^2 + a_2 \kappa G_{cov}^2 \nonumber\\
    &= \frac{a_4 \loss_0}{24\gamma T} + a_3 \gamma L \sigmabar^2 + a_2 \kappa G_{cov}^2.
    \label{eq:penultimate2-nonconvex}
\end{align}

Recall that by definition
\begin{align*}
    \gamma = \min{\left\{\frac{1}{24L}, ~ \frac{\sqrt{a_4 \loss_0}}{2\sigmabar \sqrt{a_3 L T}}\right\}},
\end{align*}
and thus $\frac{1}{\gamma} = \max{\left\{24L, \frac{2}{\sqrt{a_4 \loss_0}}\sigmabar\sqrt{a_3LT}\right\}} \leq 24L + \frac{2}{\sqrt{a_4 \loss_0}}\sigmabar\sqrt{a_3LT}$.
Therefore, we have
\begin{equation*}
    \frac{a_4 \loss_0}{24\gamma T} 
    \leq \frac{a_4 \loss_0}{24T}\left(24L + \frac{2}{\sqrt{a_4 \loss_0}}\sigmabar\sqrt{a_3LT}\right)
    = \frac{a_4 L \loss_0}{T} + \frac{\sqrt{a_3 a_4 L \loss_0} \sigmabar}{12\sqrt{T}}.
\end{equation*}

Upon using the above, and that $\gamma \leq \frac{\sqrt{a_4 \loss_0}}{2\sigmabar \sqrt{a_3 L T}}$, in~\eqref{eq:penultimate2-nonconvex}, we obtain that
\begin{align*}
    \frac{1}{T}\sum_{t = 0}^{T-1} \expect{\norm{\nabla \loss_{\H}(\weight{t})}^2} 
    &\leq \frac{a_4 L \loss_0}{T} + \frac{\sqrt{a_3 a_4 L \loss_0} \sigmabar}{12\sqrt{T}} + \frac{\sqrt{a_3 a_4 L \loss_0} \sigmabar}{2\sqrt{T}} + a_2 \kappa G_{cov}^2 
    \leq a_2 \kappa G_{cov}^2 +  \frac{\sqrt{a_3 a_4 L \loss_0} \sigmabar}{\sqrt{T}} + \frac{a_4 L \loss_0}{T}.
\end{align*}
Finally, recall from Algorithm~\ref{algo:robust-dpsgd} that $\hat{\weight{}}$ is chosen randomly from the set of parameter vectors $\left(\weight{0}, \ldots, \, \weight{T-1} \right)$. Thus, $\expect{\norm{\nabla \loss_{\H}\left(\hat{\weight{}} \right)}^2} = \frac{1}{T}\sum_{t = 0}^{T-1} \expect{\norm{\nabla \loss_{\H}(\weight{t})}^2}$. Substituting this above proves the theorem.
\end{proof}

\subsection{Proof of \cref{cor:tradeoff}}
\label{app:cor-strongly-convex}
We now state the proof of Corollary~\ref{cor:tradeoff} below.

\tradeoff*
\begin{proof}
Assume that $\loss_\H$ is $L$-smooth and $\mu$-strongly convex.
Consider \cref{algo:robust-dpsgd} with aggregation $F = \mathrm{SMEA}$, learning rate $\gamma_t= \frac{10}{\mu(t+a_1 \frac{L}{\mu})}$, and momentum coefficient $\beta_t = 1 - 24L \gamma_t$. 
By Theorem~\ref{thm:privacy}, the condition on $\sigmadp$ ensures that Algorithm~\ref{algo:robust-dpsgd} is $(\varepsilon,\delta)$-DP.
In the remaining, we prove that Algorithm~\ref{algo:robust-dpsgd} is $(f,\varrho)$-robust as stated in the corollary. 

First, note that, by \cref{prop:smea}, SMEA is $(f,\kappa)$-robust averaging with $\kappa = \frac{4f}{n-f}(1+\frac{f}{n-2f})^2$.
In fact, as we assume $n \geq (2+\eta)f$ where $\eta>0$ is an absolute constant, we have 
\begin{equation}
    \kappa \leq \frac{4f}{n-f}(1+\frac{1}{\eta})^2 = \mathcal{O}{(\frac{f}{n-f})}.
    \label{eq:kappa-smea}
\end{equation}
Therefore, thanks to Theorem~\ref{thm:convergence}, we have
\begin{align}
\label{eq:cor-eq0}
   \expect{\loss_{\H}{(\theta_{T})}-\loss_*}
    &\leq 4a_1\frac{\kappa G_{cov}^2}{\mu}  +
    \frac{2 a_1^2 L \sigmabar^2}{\mu^2 T} +  \frac{2a_1^2 L^2 \loss_0}{\mu^2 T^2},
\end{align}
where the constant $a_1$ is defined as in~\eqref{eqn:a1a2}, and
\begin{equation*}
    \sigmabar^2 \coloneqq \frac{\sigma_b^2 + d \sigmadp^2}{n-f}
    +4 \kappa\left(\sigma_b^2+36\sigmadp^2(1+\frac{d}{n-f})\right), \quad
    \sigma_b^2 \coloneqq 2(1-\frac{b}{m})\frac{\sigma^2}{b}.
\end{equation*}
We now analyze independently the terms of~\eqref{eq:cor-eq0} that depend on $T$, i.e. the last two terms on the RHS of \eqref{eq:cor-eq0}. 
Recall that, asymptotically in $T$, the condition on $\sigmadp$ implies
\begin{align}
    \label{eq:condition}
    \sigmadp &= k \cdot \frac{2C}{b} \max{\{1, \,  \nicefrac{b \sqrt{T \log{(1/\delta)}}}{m\varepsilon}\}}
    = \mathcal{O}{\left(\frac{C\sqrt{T \log{(1/\delta)}}}{m\varepsilon}\right)}.
\end{align}

\textbf{Term $\frac{2 a_1^2 L \sigmabar^2}{\mu^2 T}$.}
Recalling the expression of $\sigmabar^2$, and using \eqref{eq:condition} and \eqref{eq:kappa-smea} and the facts that $\sigma_b^2$ is independent of $T$ and $f < n-f$, we obtain
\begin{align*}
    \sigmabar^2 &= \frac{\sigma_b^2 + d \sigmadp^2}{n-f}
    +4 \kappa\left(\sigma_b^2+36\sigmadp^2(1+\frac{d}{n-f})\right)
    =\mathcal{O}{\left(\frac{d \sigmadp^2}{n-f} + \frac{f}{n-f} \cdot \sigmadp^2(1+\frac{d}{n-f}) \right)}\nonumber\\
    &=\mathcal{O}{\left(\frac{d \sigmadp^2}{n-f} + \frac{f}{n-f} \cdot \sigmadp^2 \right)}
    = \mathcal{O}{\left( \frac{C^2  d\, T \log{(1/\delta)}}{m^2 (n-f) \varepsilon^2}+
        \frac{f}{n-f} \cdot \frac{C^2 T \log{(1/\delta)}}{m^2\varepsilon^2} \right)}.
\end{align*}
As a result, we obtain
\begin{equation}
    \frac{2 a_1^2 L \sigmabar^2}{\mu^2 T} = \mathcal{O}{\left( \frac{C^2 d \,\log{(1/\delta)}}{m^2 (n-f) \varepsilon^2}+
        \frac{f}{n-f} \cdot \frac{C^2 \log{(1/\delta)}}{m^2\varepsilon^2} \right)}.
\end{equation}

\textbf{Term $\frac{2a_1^2 L^2 \loss_0}{\mu^2 T^2}$.}
This term is independent of $\sigmadp$ and vanishes with $T$.

Going back to \eqref{eq:cor-eq0}, and ignoring terms vanishing in $T$, and using \eqref{eq:kappa-smea}, we obtain
\begin{equation*}
    \expect{\loss_{\H}{(\theta_{T})}-\loss_*} = \mathcal{O} \left( \frac{C^2 d\,\log{(1/\delta)}}{m^2(n-f)\varepsilon^2}+
        \frac{f}{n-f}\frac{C^2 \log{(1/\delta)}}{m^2\varepsilon^2}+
        \frac{f}{n-f}G_{cov}^2 \right).
\end{equation*}
Finally, note that $G_{cov}^2 \leq G^2$.
Indeed, using the definition of $G_{cov}^2$ and \cref{asp:hetero}, together with Cauchy-Schwartz, we have
\begin{align*}
    G_{cov}^2 =& \sup_{\theta \in \R^d} \sup_{\norm{v}\leq 1} \frac{1}{\card{\H}}\sum_{i \in \H}\iprod{v}{\nabla{\loss_i{(\theta)}} - \nabla{\loss_{\H}{(\theta)}}}^2 \leq \sup_{\theta \in \R^d} \frac{1}{\card{\H}}\sum_{i \in \H}\norm{\nabla{\loss_i{(\theta)}} - \nabla{\loss_{\H}{(\theta)}}}^2 \leq G^2.
\end{align*}
Using the fact above in the last inequality, together with the fact that $n-f \geq \frac{n}{2}$ (as $n >2f$), we conclude
\begin{align*}
    \expect{\loss_{\H}{(\theta_{T})}-\loss_*} = \mathcal{O} \left( \frac{C^2 d\,\log{(1/\delta)}}{m^2 n\varepsilon^2}+
        \frac{f}{n}\frac{C^2 \log{(1/\delta)}}{m^2\varepsilon^2}+
        \frac{f}{n}G^2 \right).
\end{align*}
Ignoring the constant $C$ above concludes the proof.
\end{proof}
\subsection{Proof of \cref{cor:tradeoff-nonconvex}}
\label{app:cor-nonconvex}

We now state the robustness and DP guarantees of \algoname{} with SMEA in the non-convex case in Corollary~\ref{cor:tradeoff-nonconvex} below.

\begin{corollary}
\label{cor:tradeoff-nonconvex}
Consider Algorithm~\ref{algo:robust-dpsgd} with aggregation $F = \mathrm{SMEA}$, under the non-convex setting of Theorem~\ref{thm:convergence}.
Suppose that assumptions~\ref{asp:hetero}, \ref{asp:bnd_var}, \ref{asp:bnd_norm} hold, that $\loss_\H$ is $L$-smooth, and that $n \geq (2+\eta)f$, for some absolute constant $\eta>0$.
Let $\varepsilon>0, \delta \in (0,1)$ be such that $\varepsilon \leq \log{(1/\delta)}$.
Then, there exists a constant $k>0$ such that, if $ \sigmadp = k \cdot \nicefrac{2C}{b} \max{\{1, \,  \nicefrac{b \sqrt{T \log{(1/\delta)}}}{\varepsilon m}\}}$, then Algorithm~\ref{algo:robust-dpsgd} is $(\varepsilon,\delta)$-DP and $(f, \varrho)$-robust, where
\begin{align*}
   \varrho = \mathcal{O} \left( \frac{\sqrt{d\,\log{(1/\delta)}}}{\varepsilon \sqrt{n} m}+
        \sqrt{\frac{f}{n}} \cdot \frac{\sqrt{\log{(1/\delta)}}}{\varepsilon m}+
        \frac{f}{n}G^2 \right).
\end{align*}
\end{corollary}
\begin{proof}
Assume that $\loss_\H$ is $L$-smooth.
Consider \cref{algo:robust-dpsgd} with aggregation $F = \mathrm{SMEA}$, learning rate $\gamma_t = \gamma = \min{\left\{\frac{1}{24L}, ~ \frac{\sqrt{a_4 \loss_0}}{2\sigmabar \sqrt{a_3 L T}}\right\}}$, and momentum coefficient $\beta_t = \beta = 1 - 24 L \gamma$. 
By Theorem~\ref{thm:privacy}, the condition on $\sigmadp$ ensures that Algorithm~\ref{algo:robust-dpsgd} is $(\varepsilon,\delta)$-DP.
In the remaining, we prove that Algorithm~\ref{algo:robust-dpsgd} is $(f,\varrho)$-robust as stated in the corollary. 

First, note that, by \cref{prop:smea}, SMEA is $(f,\kappa)$-robust averaging with $\kappa = \frac{4f}{n-f}(1+\frac{f}{n-2f})^2$.
In fact, as we assume $n \geq (2+\eta)f$ where $\eta>0$ is an absolute constant, we have 
\begin{equation}
    \kappa \leq \frac{4f}{n-f}(1+\frac{1}{\eta})^2 = \mathcal{O}{(\frac{f}{n-f})}.
    \label{eq:kappa-smea-nonconvex}
\end{equation}
Therefore, thanks to Theorem~\ref{thm:convergence}, we have
\begin{align}
\label{eq:cor-eq0-nonconvex}
   \expect{\|\nabla \loss_{\H} {(\hat{\weight{}})}\|^2} 
    \leq  a_2 \kappa G_{cov}^2 +  \frac{\sqrt{a_3 a_4 L \loss_0} \sigmabar}{\sqrt{T}} + \frac{a_4 L \loss_0}{T},
\end{align}
where the constants $a_1, a_2, a_3, a_4$ are defined as in~\eqref{eqn:a1a2}, and
\begin{equation*}
    \sigmabar^2 \coloneqq \frac{\sigma_b^2 + d \sigmadp^2}{n-f}
    +4 \kappa\left(\sigma_b^2+36\sigmadp^2(1+\frac{d}{n-f})\right), \quad
    \sigma_b^2 \coloneqq 2(1-\frac{b}{m})\frac{\sigma^2}{b}.
\end{equation*}
We now analyze independently the terms of~\eqref{eq:cor-eq0} that depend on $T$, i.e. the last two terms on the RHS of \eqref{eq:cor-eq0}. 
Recall that, asymptotically in $T$, the condition on $\sigmadp$ implies
\begin{align}
    \label{eq:condition-nonconvex}
    \sigmadp &= k \cdot \frac{2C}{b} \max{\{1, \,  \nicefrac{b \sqrt{T \log{(1/\delta)}}}{m\varepsilon}\}}
    = \mathcal{O}{\left(\frac{C\sqrt{T \log{(1/\delta)}}}{m\varepsilon}\right)}.
\end{align}

\textbf{Term $\frac{\sqrt{a_3 a_4 L \loss_0} \sigmabar}{\sqrt{T}}$.}
Recalling the expression of $\sigmabar^2$, and using \eqref{eq:condition} and \eqref{eq:kappa-smea} and the facts that $\sigma_b^2$ is independent of $T$ and $f < n-f$, we obtain
\begin{align*}
    \sigmabar^2 &= \frac{\sigma_b^2 + d \sigmadp^2}{n-f}
    +4 \kappa\left(\sigma_b^2+36\sigmadp^2(1+\frac{d}{n-f})\right)
    =\mathcal{O}{\left(\frac{d \sigmadp^2}{n-f} + \frac{f}{n-f} \cdot \sigmadp^2(1+\frac{d}{n-f}) \right)}\nonumber\\
    &=\mathcal{O}{\left(\frac{d \sigmadp^2}{n-f} + \frac{f}{n-f} \cdot \sigmadp^2 \right)}
    = \mathcal{O}{\left( \frac{C^2  d\, T \log{(1/\delta)}}{m^2 (n-f) \varepsilon^2}+
        \frac{f}{n-f} \cdot \frac{C^2 T \log{(1/\delta)}}{m^2\varepsilon^2} \right)}.
\end{align*}
Therefore, using $\sqrt{x+y} \leq \sqrt{x}+\sqrt{y}$, we obtain
\begin{align*}
    \sigmabar &= \mathcal{O}{\left( \frac{C \sqrt{d\, T \log{(1/\delta)}}}{m \sqrt{n-f} \varepsilon}+
        \sqrt{\frac{f}{n-f}} \cdot \frac{C \sqrt{T \log{(1/\delta)}}}{m \varepsilon} \right)}.
\end{align*}
As a result, we obtain
\begin{equation}
    \frac{\sqrt{a_3 a_4 L \loss_0} \sigmabar}{\sqrt{T}} = \mathcal{O}{\left( \frac{C \sqrt{d\, \log{(1/\delta)}}}{m \sqrt{n-f} \varepsilon}+
        \sqrt{\frac{f}{n-f}} \cdot \frac{C \sqrt{\log{(1/\delta)}}}{m \varepsilon} \right)}.
\end{equation}

\textbf{Term $\frac{a_4 L \loss_0}{T}$.}
This term is independent of $\sigmadp$ and vanishes with $T$.

Going back to \eqref{eq:cor-eq0-nonconvex}, ignoring terms vanishing in $T$, and using \eqref{eq:kappa-smea-nonconvex}, we obtain
\begin{equation*}
    \expect{\|\nabla \loss_{\H} {(\hat{\weight{}})}\|^2} = \mathcal{O} \left( \frac{C \sqrt{d\, \log{(1/\delta)}}}{m \sqrt{n-f} \varepsilon}+
        \sqrt{\frac{f}{n-f}} \cdot \frac{C \sqrt{\log{(1/\delta)}}}{m \varepsilon}+
        \frac{f}{n-f}G_{cov}^2 \right).
\end{equation*}
Finally, note that $G_{cov}^2 \leq G^2$.
Indeed, using the definition of $G_{cov}^2$ and \cref{asp:hetero}, together with Cauchy-Schwartz, we have
\begin{align*}
    G_{cov}^2 =& \sup_{\theta \in \R^d} \sup_{\norm{v}\leq 1} \frac{1}{\card{\H}}\sum_{i \in \H}\iprod{v}{\nabla{\loss_i{(\theta)}} - \nabla{\loss_{\H}{(\theta)}}}^2 \leq \sup_{\theta \in \R^d} \frac{1}{\card{\H}}\sum_{i \in \H}\norm{\nabla{\loss_i{(\theta)}} - \nabla{\loss_{\H}{(\theta)}}}^2 \leq G^2.
\end{align*}
Using the fact above in the last inequality, together with the fact that $n-f \geq \frac{n}{2}$ (as $n >2f$), we conclude
\begin{align*}
    \expect{\|\nabla \loss_{\H} {(\hat{\weight{}})}\|^2} = \mathcal{O} \left( \frac{C \sqrt{d\, \log{(1/\delta)}}}{m \sqrt{n} \varepsilon}+
        \sqrt{\frac{f}{n}} \cdot \frac{C \sqrt{\log{(1/\delta)}}}{m \varepsilon}+
        \frac{f}{n}G^2 \right).
\end{align*}
Ignoring the constant $C$ above concludes the proof.
\end{proof}

\paragraph{Discussion.}
We conjecture that the non-convex upper bound can be improved as observed recently in the centralized DP setting using other variance reduction  techniques~\cite{arora2022faster}.
Nevertheless, both in the centralized and distributed settings, it remains an open question to derive tight lower bounds for non-convex problems.

\clearpage
\subsection{Proof of Supporting Lemmas}
\label{app:lemmas}
Before proving \cref{lem:drift,lem:dev,lem:descent} in Sections~\ref{app:lem-drift} to \ref{app:descent} respectively, we first show some additional technical lemmas in Section~\ref{app:technical} below.

\subsubsection{Technical Lemmas}
\label{app:technical}

\begin{lemma}
    \label{lem:cover}
    Let $M \in \R^{d \times d}$ be a random real symmetric matrix and $g \colon \R \to \R$ an increasing function.
    It holds that
\begin{equation*}
    \expect{\sup_{\norm{v}\leq1}g{\left(\iprod{v}{Mv}\right)}}
    \leq 9^d \cdot \sup_{\norm{v} \leq 1} \expect{g{\left(2\iprod{v}{Mv}\right)}}.
\end{equation*}
\end{lemma}
\begin{proof}
Let $M \in \R^{d \times d}$ be a random real symmetric matrix and $g \colon \R \to \R$ a increasing function.

The proof follows the construction of (Section 5.2, \cite{vershynin2010introduction}).
Recall from standard covering net results~\cite{vershynin2010introduction} that we can construct $\N_{1/4}$ a finite $1/4$-net of the unit ball, i.e., for any vector $v$ in the unit ball, there exists $u_v \in \N_{1/4}$ such that $\norm{u_v-v} \leq 1/4$.
Moreover, we have the bound $\card{\N_{1/4}} \leq (1+2/(1/4))^d = 9^d$.
Denote by $\norm{M} \coloneqq \sup_{\norm{v}\leq1}\norm{Mv}$ the operator norm of $M$.
By recalling that $M$ is symmetric, we obtain for any $v$ in the unit ball
\begin{align*}
    \absv{\iprod{v}{Mv} - \iprod{u_v}{Mu_v}}
    &= \absv{\iprod{v+u_v}{M(v-u_v)}}
    \leq \norm{v+u_v} \norm{M(v-u_v)}
    \leq (\norm{v}+\norm{u_v})\norm{M(v-u_v)}\\
    &\leq 2\norm{M(v-u_v)}
    \leq 2\norm{M}\norm{v-u_v}
    \leq 2\norm{M}/4 = \norm{M}/2.
\end{align*}
Therefore, we have $\iprod{v}{Mv} - \iprod{u_v}{Mu_v} \leq \norm{M}/2$, and $\iprod{v}{Mv} - \norm{M}/2 \leq \iprod{u_v}{Mu_v} \leq \sup_{u \in \N_{1/4}} \iprod{u}{Mu}$.
Recall that since $M$ is symmetric, its operator norm coincides with its maximum eigenvalue: $\norm{M} = \sup_{\norm{v}\leq1}\iprod{v}{Mv}$.
We therefore deduce that
\begin{equation*}
    \sup_{\norm{v}\leq1}\iprod{v}{Mv}
    \leq 2\cdot \sup_{v \in \N_{1/4}} \iprod{v}{Mv}.
\end{equation*}
Upon composing with $g$, which is increasing, we get
\begin{equation*}
    \sup_{\norm{v}\leq 1} g{\left( \iprod{v}{Mv} \right)} 
    = g{\left( \sup_{\norm{v}\leq 1} \iprod{v}{Mv} \right)} 
    \leq g{\left( 2\cdot \sup_{v \in \N_{1/4}} \iprod{v}{Mv} \right)}
    =  \sup_{v \in \N_{1/4}} g{\left( 2\iprod{v}{Mv} \right)}.
\end{equation*}
Upon taking expectations and applying union bound, we finally conclude
\begin{align*}
    \expect{\sup_{\norm{v}\leq 1} g{\left( \iprod{v}{Mv} \right)}}
    &\leq \expect{\sup_{v \in \N_{1/4}} g{\left( 2\iprod{v}{Mv} \right)}}
    \leq \card{\N_{1/4}} \cdot \sup_{v \in \N_{1/4}} \expect{g{\left( 2\iprod{v}{Mv} \right)}} 
    \leq 9^d \cdot \sup_{\norm{v} \leq 1} \expect{g{\left( 2\iprod{v}{Mv} \right)}}.
\end{align*}
\end{proof}

\begin{lemma}
\label{lem:grad-subsample}
Suppose assumptions \ref{asp:bnd_var} and \ref{asp:bnd_norm} hold.
For any $t \in \{0,\ldots, T-1\}$ and $i \in \H$, we have
\begin{equation*}
    \expect{\norm{\tilde{g}_t^{(i)} - \nabla{\loss_i{(\theta_{t})}}}^2} \leq 2\left(1-\frac{b}{m}\right) \frac{\sigma^2}{b} + d \cdot \sigmadp^2.
\end{equation*}
\end{lemma}
\begin{proof}
Suppose assumptions \ref{asp:bnd_var} and \ref{asp:bnd_norm} hold.
Let $i \in \H$ and $t \in \{0,\ldots, T-1\}$.

First recall from \eqref{eq:noise-gradient} that, since $\tilde{g}_t^{(i)} = g_t^{(i)} + \xi^{(i)}_t, \xi_t^{(i)} \overset{\mathrm{i.i.d.}}{\sim} \N{(0,\sigmadp^2 I_d)}$, we have
\begin{equation*}
    \condexpect{\xi^{(i)}_t}{\norm{\tilde{g}_t^{(i)} - g_t^{(i)}}^2} 
    = \expect{\norm{\xi^{(i)}_t}^2}
    =d \cdot \sigmadp^2.
\end{equation*}

Next, we have
\begin{align*}
    \norm{\tilde{g}_t^{(i)} - \nabla{\loss_i{(\theta_{t})}}}^2
    &= \norm{\tilde{g}_t^{(i)} - g_t^{(i)} + g_t^{(i)} - \nabla{\loss_i{(\theta_{t})}}}^2 \\
    &= \norm{\tilde{g}_t^{(i)} - g_t^{(i)}}^2 + 
    \norm{g_t^{(i)} - \nabla{\loss_i{(\theta_{t})}}}^2 +
    2\iprod{\tilde{g}_t^{(i)} - g_t^{(i)}}{g_t^{(i)} - \nabla{\loss_i{(\theta_{t})}}}.
\end{align*}
Now taking expectation on the randomness of $\xi^{(i)}_{t}$ (independent of all other random variables), and since $\expect{\xi^{(i)}_{t}} = 0$, we get
\begin{align*}
    \condexpect{\xi^{(i)}_{t}}{\norm{\tilde{g}_t^{(i)} - \nabla{\loss_i{(\theta_{t})}}}^2}
    &= \condexpect{\xi^{(i)}_{t}}{\norm{\tilde{g}_t^{(i)} - g_t^{(i)}}^2}
    + \norm{g_t^{(i)} - \nabla{\loss_i{(\theta_{t})}}}^2
    + 2\iprod{\underbrace{\condexpect{\xi^{(i)}_{t}}{\tilde{g}_t^{(i)} - g_t^{(i)}}}_{= \expect{\xi^{(i)}_{t}}=0}}{g_t^{(i)} - \nabla{\loss_i{(\theta_{t})}}} \\
    &= \condexpect{\xi^{(i)}_{t}}{\norm{\tilde{g}_t^{(i)} - g_t^{(i)}}^2}
    + \norm{g_t^{(i)} - \nabla{\loss_i{(\theta_{t})}}}^2.
\end{align*}
Upon taking total expectation, we obtain
\begin{align}
    \expect{\norm{\tilde{g}_t^{(i)} - \nabla{\loss_i{(\theta_{t})}}}^2}
    &= \expect{\norm{\tilde{g}_t^{(i)} - g_t^{(i)}}^2}
    + \expect{\norm{g_t^{(i)} - \nabla{\loss_i{(\theta_{t})}}}^2}\nonumber \\
    & = \expect{\norm{g_t^{(i)} - \nabla{\loss_i{(\theta_{t})}}}^2} + d \cdot \sigmadp^2.
    \label{eq:bound-decomp}
\end{align}

First observe that when $m=1$, as $b \in [m]$, we must have $b=m$. Thus, the gradient is deterministic, i.e.,  $g^{(i)}_t = \nabla \loss_i{(\theta_t)}$.
Thus, the first term in the equation above is zero, and the claimed bound holds.

Else, when $m \geq 2$, recall that from Assumption~\ref{asp:bnd_var}, we have 
$\condexpect{x \sim \U{(\D_i)}}{\norm{\nabla_{\theta{}}{\ell{(\theta_{t};x)}} - \nabla{\loss_i{(\theta_{t})}}}^2} \leq \sigma^2$.
From \cite{rice2006mathematical}, the variance reduction due to subsampling without replacement gives
\begin{equation*}
    \expect{\norm{g_t^{(i)} - \nabla{\loss_i{(\theta_{t})}}}^2} \leq \left(1-\frac{b-1}{m-1}\right)\frac{\sigma^2}{b}.
\end{equation*}
Plugging this bound back in \cref{eq:bound-decomp} yields
\begin{align*}
    \expect{\norm{\tilde{g}_t^{(i)} - \nabla{\loss_i{(\theta_{t})}}}^2}
    &\leq \left(1-\frac{b-1}{m-1}\right)\frac{\sigma^2}{b}
    + d \cdot \sigmadp^2.
\end{align*}
By observing, as $m \geq 2$, that $1-\frac{b-1}{m-1} = \frac{m-b}{m-1} = \frac{m}{m-1} \cdot \frac{m-b}{m} = (1+\frac{1}{m-1})(1-\frac{b}{m}) \leq 2(1-\frac{b}{m})$, we obtain the final result:
\begin{align*}
    \expect{\norm{\tilde{g}_t^{(i)} - \nabla{\loss_i{(\theta_{t})}}}^2}
    &\leq 2\left(1-\frac{b}{m}\right)\frac{\sigma^2}{b}
    + d \cdot \sigmadp^2.
\end{align*}
\end{proof}

\begin{lemma}
\label{lem:concentration}
Let $\sigmadp \geq 0$ and $d, n \geq 1$. 
Consider $\xi^{(1)}, \ldots, \xi^{(n)}$ to be i.i.d. random variables drawn from the Gaussian distribution $\N{(0, \sigmadp^2 I_d)}$. We have
\begin{align*}
    \expect{\sup_{\norm{v} \leq 1} \frac{1}{n} \sum_{i=1}^n \iprod{v}{\xi^{(i)}}^2}
    \leq 36\sigmadp^2 \left(1+\frac{d}{n}\right).
\end{align*}
\end{lemma}
\begin{proof}
Let $\sigmadp \geq 0$ and $d, n \geq 1$. 
Consider $\xi^{(1)}, \ldots, \xi^{(n)}$ to be i.i.d. random variables drawn from the Gaussian distribution $\N{(0, \sigmadp^2 I_d)}$.

If $\sigmadp = 0$, then $\xi^{(i)} = 0$ almost surely for every $i \in [n]$, and the remainder of the proof holds with $\sigmadp=0$.
Else, we assume $\sigmadp > 0$ in the remaining.

Thus, the law of the random variable $\xi^{(i)} / \sigmadp$ is $\N{\left(0, I_d\right)}$ for every $i \in [n]$.
Thus, for every vector of the unit ball $v$, the random variable $\iprod{v}{\xi^{(i)} / \sigmadp}$ is sub-Gaussian with variance proxy equal to $1$ (see~(Chapter 1, \cite{rigollet2015high})).
Therefore, for every $i \in [n]$ and every vector $v$ of the unit ball, applying (Theorem~2.1.1, \cite{pauwels2020lecture}), we have
\begin{align*}
    \expect{\exp{\left(\iprod{v}{\xi^{(i)} /\sigmadp}^2/8\right)}} \leq 2.
\end{align*}
As a result, by the independence of $\xi^{(i)}$'s, we obtain
\begin{align*}
    \sup_{\norm{v} \leq 1} \expect{\exp{\left(\frac{1}{8\sigmadp^2} \sum_{i=1}^n\iprod{v}{\xi^{(i)}}^2\right)}}
    &= \sup_{\norm{v} \leq 1} \prod_{i =1}^n \expect{\exp{\left(\iprod{v}{\xi^{(i)} / \sigmadp}^2/8\right)}}
    \leq 2^{n}.
\end{align*}
Now, observe that we can write $\sum_{i=1}^n\iprod{v}{\xi^{(i)}}^2$ as the quadratic form $\iprod{v}{Mv}$, where $M \coloneqq \sum_{i=1}^n \xi^{(i)} \cdot {\xi^{(i)}}^\top$ is a random real symmetric matrix.
Thus, applying \cref{lem:cover} with the increasing function $g = \exp{(\frac{1}{16\sigmadp^2} \times \cdot)}$, we have
\begin{align*}
    \expect{\sup_{\norm{v} \leq 1} \exp{\left(\frac{1}{16\sigmadp^2} \sum_{i=1}^n\iprod{v}{\xi^{(i)}}^2\right)}}
    &=\expect{\sup_{\norm{v} \leq 1} g(\iprod{v}{Mv})}
    \leq 9^d \cdot \sup_{\norm{v} \leq 1} \expect{g(2\iprod{v}{Mv})}\\
    &= 9^d \cdot \sup_{\norm{v} \leq 1} \expect{\exp{\left(\frac{1}{8\sigmadp^2} \sum_{i=1}^n\iprod{v}{\xi^{(i)}}^2\right)}}
    \leq 9^d \cdot 2^{n}.
\end{align*}
We can now use this inequality to bound the term of interest.
We apply Jensen's inequality thanks to $\exp$ being convex, and we also interchange $\exp$ and $\sup$ thanks to the former being increasing:
\begin{align*}
    \exp{\left(\frac{1}{16\sigmadp^2}\expect{\sup_{\norm{v} \leq 1} \sum_{i=1}^n\iprod{v}{\xi^{(i)}}^2}\right)}
    &\leq \expect{\exp{\left(\frac{1}{16\sigmadp^2} \sup_{\norm{v} \leq 1} \sum_{i=1}^n\iprod{v}{\xi^{(i)}}^2 \right)}} \\
    &=  \expect{\sup_{\norm{v} \leq 1} \exp{\left( \frac{1}{16\sigmadp^2} \sum_{i=1}^n\iprod{v}{\xi^{(i)}}^2\right)}}
    \leq 9^d \cdot 2^{n}.
\end{align*}
Upon applying $\ln$ and multiplying by $16 \sigmadp^2/n$, we obtain
\begin{align*}
    \expect{\sup_{\norm{v} \leq 1} \frac{1}{n} \sum_{i=1}^n\iprod{v}{\xi^{(i)}}^2}
    &\leq 16\frac{\sigmadp^2}{n}(d \ln{9} + n \ln{2})\alpha^2
    \leq 36\frac{\sigmadp^2}{n}(d + n)
    = 36\sigmadp^2 \left(1+\frac{d}{n}\right).
\end{align*}
The above concludes the proof
\end{proof}

\subsubsection{Proof of \cref{lem:drift}}
\label{app:lem-drift}
\globaldrift*
\begin{proof}
Let $t \in \{0, \ldots, T-1\}$.
Suppose that \cref{asp:bnd_var} holds.
Recall that the alternate definition of maximum eigenvalue implies, following the definition of $\Delta_t$ in \cref{eq:defdrift}, that
\begin{align*}
    \Delta_t = \lambda_{\max}{\left(\frac{1}{\card{\H}} \sum_{i \in \H}(m_t^{(i)} - \overline{m}_t)(m_t^{(i)} - \overline{m}_t)^\top\right)}
    = \sup_{\norm{v} \leq 1} \frac{1}{\card{\H}} \sum_{i \in \H}\iprod{v}{m^{(i)}_{t}-\overline{m}_{t}}^2.
\end{align*}
We will use the latter expression above for $\Delta_t$ throughout this lemma.

For every $i \in \H$, by definition of $m^{(i)}_t$, given in \cref{eqn:mmt_i}, we have
\begin{align*}
    m^{(i)}_{t+1} = \beta_t m^{(i)}_{t} + (1-\beta_t)\tilde{g}^{(i)}_{t+1}.
\end{align*}
We also denote $\overline{m}_t \coloneqq \frac{1}{\card{\H}}\sum_{i \in \H}m^{(i)}_t$ and $\overline{\widetilde{g}}_{t+1} \coloneqq \frac{1}{\card{\H}}\sum_{i \in \H}\tilde{g}^{(i)}_{t+1}$.
Therefore, we have $\overline{m}_{t+1} = \beta_t \overline{m}_{t} + (1-\beta_t)\overline{\widetilde{g}}_{t+1}$.
As a result, we can write for every $i \in \H$
\begin{align*}
    m^{(i)}_{t+1}-\overline{m}_{t+1}
    &= \beta_t (m^{(i)}_{t}-\overline{m}_{t})
    + (1-\beta_t)(\tilde{g}_{t+1}^{(i)}-\overline{\widetilde{g}}_{t+1}) \\
    &= \beta_t (m^{(i)}_{t}-\overline{m}_{t})
    + (1-\beta_t)(\nabla \loss_i{(\theta_{t+1})}-\nabla \loss_{\H}{(\theta_{t+1})})\\
    &\quad +(1-\beta_t)(\tilde{g}_{t+1}^{(i)} - \nabla \loss_i{(\theta_{t+1})} - \overline{\widetilde{g}}_{t+1} + \nabla \loss_{\H}{(\theta_{t+1})}).
\end{align*}
By projecting the above expression on an arbitrary vector $v$ and then taking squares, we obtain
\begin{align*}
    \iprod{v}{m^{(i)}_{t+1}-\overline{m}_{t+1}}^2
    &= \Big[\beta_t\iprod{v}{m^{(i)}_{t}-\overline{m}_{t}}
    + (1-\beta_t)\iprod{v}{\nabla \loss_i{(\theta_{t+1})}-\nabla \loss_{\H}{(\theta_{t+1})}}\\
    &\qquad +(1-\beta_t)\iprod{v}{\tilde{g}_{t+1}^{(i)} - \nabla \loss_i{(\theta_{t+1})} - \overline{\widetilde{g}}_{t+1} + \nabla \loss_{\H}{(\theta_{t+1})}}\Big ]^2\\
    &=\beta_t^2 \iprod{v}{m^{(i)}_{t}-\overline{m}_{t}}^2
    + (1-\beta_t)^2\iprod{v}{\nabla \loss_i{(\theta_{t+1})}-\nabla \loss_{\H}{(\theta_{t+1})}}^2\\
    &\qquad +(1-\beta_t)^2\iprod{v}{\tilde{g}_{t+1}^{(i)} - \nabla \loss_i{(\theta_{t+1})} - \overline{\widetilde{g}}_{t+1} + \nabla \loss_{\H}{(\theta_{t+1})}}^2\\
    &\qquad +2\beta_t(1-\beta_t)\iprod{v}{m^{(i)}_{t}-\overline{m}_{t}}\iprod{v}{\nabla \loss_i{(\theta_{t+1})}-\nabla \loss_{\H}{(\theta_{t+1})}}\\
    &\qquad + 2\beta_t(1-\beta_t)\iprod{v}{m^{(i)}_{t}-\overline{m}_{t}}\iprod{v}{\tilde{g}_{t+1}^{(i)} - \nabla \loss_i{(\theta_{t+1})} - \overline{\widetilde{g}}_{t+1} + \nabla \loss_{\H}{(\theta_{t+1})}}\\
    &\qquad + 2\beta_t(1-\beta_t)\iprod{v}{\nabla \loss_i{(\theta_{t+1})}-\nabla \loss_{\H}{(\theta_{t+1})}}\iprod{v}{\tilde{g}_{t+1}^{(i)} - \nabla \loss_i{(\theta_{t+1})} - \overline{\widetilde{g}}_{t+1} + \nabla \loss_{\H}{(\theta_{t+1})}}.
\end{align*}
Upon averaging over $i \in \H$, taking the supremum over the unit ball, and then total expectations, we get
\begin{align}
    &\expect{\sup_{\norm{v} \leq 1} \frac{1}{\card{\H}} \sum_{i \in \H} \iprod{v}{m^{(i)}_{t+1}-\overline{m}_{t+1}}^2}
    =\beta_t^2 \expect{\sup_{\norm{v} \leq 1} \frac{1}{\card{\H}} \sum_{i \in \H}\iprod{v}{m^{(i)}_{t}-\overline{m}_{t}}^2}\nonumber\\
    &\qquad+ (1-\beta_t)^2 \expect{\sup_{\norm{v} \leq 1} \frac{1}{\card{\H}} \sum_{i \in \H}\iprod{v}{\nabla \loss_i{(\theta_{t+1})}-\nabla \loss_{\H}{(\theta_{t+1})}}^2}\nonumber\\
    &\qquad +(1-\beta_t)^2 \expect{\sup_{\norm{v} \leq 1} \frac{1}{\card{\H}} \sum_{i \in \H}\iprod{v}{\tilde{g}_{t+1}^{(i)} - \nabla \loss_i{(\theta_{t+1})} - \overline{\widetilde{g}}_{t+1} + \nabla \loss_{\H}{(\theta_{t+1})}}^2}\nonumber\\
    &\qquad +2\beta_t(1-\beta_t) \expect{\sup_{\norm{v} \leq 1} \frac{1}{\card{\H}} \sum_{i \in \H}\iprod{v}{m^{(i)}_{t}-\overline{m}_{t}}\iprod{v}{\nabla \loss_i{(\theta_{t+1})}-\nabla \loss_{\H}{(\theta_{t+1})}}}\nonumber\\
    &\qquad + 2\beta_t(1-\beta_t) \expect{\sup_{\norm{v} \leq 1} \frac{1}{\card{\H}} \sum_{i \in \H}\iprod{v}{m^{(i)}_{t}-\overline{m}_{t}}\iprod{v}{\tilde{g}_{t+1}^{(i)} - \nabla \loss_i{(\theta_{t+1})} - \overline{\widetilde{g}}_{t+1} + \nabla \loss_{\H}{(\theta_{t+1})}}}\nonumber\\
    &\qquad + 2\beta_t(1-\beta_t) \expect{\sup_{\norm{v} \leq 1} \frac{1}{\card{\H}} \sum_{i \in \H}\iprod{v}{\nabla \loss_i{(\theta_{t+1})}-\nabla \loss_{\H}{(\theta_{t+1})}}
    \iprod{v}{\tilde{g}_{t+1}^{(i)} - \nabla \loss_i{(\theta_{t+1})} - \overline{\widetilde{g}}_{t+1} + \nabla \loss_{\H}{(\theta_{t+1})}}}.
    \label{eq:heterodrift1}
\end{align}
We now show that the last two terms on the RHS of \cref{eq:heterodrift1} are non-positive.
We show it for the first one, as the second one can be shown to be non-positive in the same way.

First, note that we can write the inner expression as a quadratic form.
Precisely, we have for any vector $v$ and any $i \in \H$ that
\begin{align*}
    2\sum_{i \in \H}\iprod{v}{m^{(i)}_{t}-\overline{m}_{t}}\iprod{v}{\tilde{g}_{t+1}^{(i)} - \nabla \loss_i{(\theta_{t+1})} - \overline{\widetilde{g}}_{t+1} + \nabla \loss_{\H}{(\theta_{t+1})}}
    = \iprod{v}{Mv},
\end{align*}
where we have introduced the $d \times d$ matrix $M \coloneqq N + N^\top$, such that $N \coloneqq \sum_{i \in \H}(m^{(i)}_{t}-\overline{m}_{t})(\tilde{g}_{t+1}^{(i)} - \nabla \loss_i{(\theta_{t+1})} - \overline{\widetilde{g}}_{t+1} + \nabla \loss_{\H}{(\theta_{t+1})})^\top$.
By observing that $M$ is symmetric, we can apply \cref{lem:cover} with $g$ being the identity mapping:
\begin{align}
    \expect{\sup_{\norm{v}\leq 1}2\sum_{i \in \H}\iprod{v}{m^{(i)}_{t}-\overline{m}_{t}}\iprod{v}{\tilde{g}_{t+1}^{(i)} - \nabla \loss_i{(\theta_{t+1})} - \overline{\widetilde{g}}_{t+1} + \nabla \loss_{\H}{(\theta_{t+1})}}}
    &= \expect{\sup_{\norm{v} \leq 1} \iprod{v}{Mv}} \nonumber\\
    &\leq 9^d \cdot \sup_{\norm{v} \leq 1} \expect{2\iprod{v}{Mv}}.
    \label{eq:heterodrift2}
\end{align}
However, the last term is zero by the total law of expectation.
Indeed, recall that stochastic gradients are unbiased (Assumption~\ref{asp:bnd_var}) and that $\theta_{t+1}$ and $m^{(i)}_t$ are deterministic when given history $\P_{t+1}$.
This gives
\begin{align*}
    \expect{\iprod{v}{Mv}}
    &=\expect{2\sum_{i \in \H}\iprod{v}{m^{(i)}_{t}-\overline{m}_{t}}\iprod{v}{\tilde{g}_{t+1}^{(i)} - \nabla \loss_i{(\theta_{t+1})} - \overline{\widetilde{g}}_{t+1} + \nabla \loss_{\H}{(\theta_{t+1})}}}\\
    &= \expect{\condexpect{t+1}{2\sum_{i \in \H}\iprod{v}{m^{(i)}_{t}-\overline{m}_{t}}\iprod{v}{\tilde{g}_{t+1}^{(i)} - \nabla \loss_i{(\theta_{t+1})} - \overline{\widetilde{g}}_{t+1} + \nabla \loss_{\H}{(\theta_{t+1})}}}} \\
    &= \expect{2\sum_{i \in \H}\iprod{v}{m^{(i)}_{t}-\overline{m}_{t}}\iprod{v}{\underbrace{\condexpect{t+1}{\tilde{g}_{t+1}^{(i)} - \nabla \loss_i{(\theta_{t+1})}}}_{=0} - \underbrace{\condexpect{t+1}{\overline{\widetilde{g}}_{t+1} - \nabla \loss_{\H}{(\theta_{t+1})}}}_{=0}}}
    =0.
\end{align*}
Moreover, going back to \cref{eq:heterodrift2}, we obtain
\begin{equation*}
\expect{\sup_{\norm{v}\leq 1}2\sum_{i \in \H}\iprod{v}{m^{(i)}_{t}-\overline{m}_{t}}\iprod{v}{\tilde{g}_{t+1}^{(i)} - \nabla \loss_i{(\theta_{t+1})} - \overline{\widetilde{g}}_{t+1} + \nabla \loss_{\H}{(\theta_{t+1})}}}
    \leq 9^d \cdot \sup_{\norm{v}\leq 1}\expect{2\iprod{v}{Mv}} = 0.
\end{equation*}
As mentioned previously, we can prove in the same way that
\begin{align*}
    \expect{\sup_{\norm{v} \leq 1} 2\sum_{i \in \H}\iprod{v}{\nabla \loss_i{(\theta_{t+1})}-\nabla \loss_{\H}{(\theta_{t+1})}}
    \iprod{v}{\tilde{g}_{t+1}^{(i)} - \nabla \loss_i{(\theta_{t+1})} - \overline{\widetilde{g}}_{t+1} + \nabla \loss_{\H}{(\theta_{t+1})}}} \leq 0.
\end{align*}

Plugging the two previous bounds back in \cref{eq:heterodrift1}, we have thus proved that
\begin{align}
    &\expect{\sup_{\norm{v} \leq 1} \frac{1}{\card{\H}} \sum_{i \in \H} \iprod{v}{m^{(i)}_{t+1}-\overline{m}_{t+1}}^2}
    =\beta_t^2 \expect{\sup_{\norm{v} \leq 1} \frac{1}{\card{\H}} \sum_{i \in \H}\iprod{v}{m^{(i)}_{t}-\overline{m}_{t}}^2}\nonumber\\
    &\qquad+ (1-\beta_t)^2 \expect{\sup_{\norm{v} \leq 1} \frac{1}{\card{\H}} \sum_{i \in \H}\iprod{v}{\nabla \loss_i{(\theta_{t+1})}-\nabla \loss_{\H}{(\theta_{t+1})}}^2}\nonumber\\
    &\qquad +(1-\beta_t)^2 \expect{\sup_{\norm{v} \leq 1} \frac{1}{\card{\H}} \sum_{i \in \H}\iprod{v}{\tilde{g}_{t+1}^{(i)} - \nabla \loss_i{(\theta_{t+1})} - \overline{\widetilde{g}}_{t+1} + \nabla \loss_{\H}{(\theta_{t+1})}}^2}\nonumber\\
    &\qquad +2\beta_t(1-\beta_t) \expect{\sup_{\norm{v} \leq 1} \frac{1}{\card{\H}} \sum_{i \in \H}\iprod{v}{m^{(i)}_{t}-\overline{m}_{t}}\iprod{v}{\nabla \loss_i{(\theta_{t+1})}-\nabla \loss_{\H}{(\theta_{t+1})}}}.
    \label{eq:heterodrift4}
\end{align}
We now bound the two last terms on the RHS of \cref{eq:heterodrift4}.

First, by using the fact that $2ab \leq a^2 + b^2$, we have for any vector $v$ that
\begin{align}
    &\frac{2}{\card{\H}} \sum_{i \in \H}\iprod{v}{m^{(i)}_{t}-\overline{m}_{t}}\iprod{v}{\nabla \loss_i{(\theta_{t+1})}-\nabla \loss_{\H}{(\theta_{t+1})}}
    \leq \frac{1}{\card{\H}} \sum_{i \in \H}\left[\iprod{v}{m^{(i)}_{t}-\overline{m}_{t}}^2 + \iprod{v}{\nabla \loss_i{(\theta_{t+1})}-\nabla \loss_{\H}{(\theta_{t+1})}}^2\right]\nonumber\\
    &\qquad = \frac{1}{\card{\H}} \sum_{i \in \H}\iprod{v}{m^{(i)}_{t}-\overline{m}_{t}}^2 + \frac{1}{\card{\H}} \sum_{i \in \H}\iprod{v}{\nabla \loss_i{(\theta_{t+1})}-\nabla \loss_{\H}{(\theta_{t+1})}}^2.
    \label{eq:heterodrift5.5}
\end{align}
Taking the supremum over the unit ball and then total expectations yields
\begin{align}
    &2\expect{ \sup_{\norm{v} \leq 1} \frac{1}{\card{\H}}  \sum_{i \in \H}\iprod{v}{m^{(i)}_{t}-\overline{m}_{t}}\iprod{v}{\nabla \loss_i{(\theta_{t+1})}-\nabla \loss_{\H}{(\theta_{t+1})}}} \nonumber\\
    &\qquad \leq \expect{ \sup_{\norm{v} \leq 1}\frac{1}{\card{\H}} \sum_{i \in \H}\iprod{v}{m^{(i)}_{t}-\overline{m}_{t}}^2} 
    + \expect{ \sup_{\norm{v} \leq 1}\frac{1}{\card{\H}} \sum_{i \in \H}\iprod{v}{\nabla \loss_i{(\theta_{t+1})}-\nabla \loss_{\H}{(\theta_{t+1})}}^2}.
    \label{eq:heterodrift6}
\end{align}

Second, recall that $\tilde{g}_{t+1}^{(i)} = g_{t+1}^{(i)} + \xi_{t+1}^{(i)}$, where $\xi_{t+1}^{(i)} \sim \N{(0, \sigmadp^2 I_d)}$.
Denote $\overline{\xi}_{t+1} \coloneqq \frac{1}{\card{\H}} \sum_{i \in \H} \xi_{t+1}^{(i)}$.
Therefore, by applying Jensen's inequality, we have
\begin{align*}
    &\expect{\sup_{\norm{v} \leq 1} \frac{1}{\card{\H}} \sum_{i \in \H}\iprod{v}{\tilde{g}_{t+1}^{(i)} - \nabla \loss_i{(\theta_{t+1})} - \overline{\widetilde{g}}_{t+1} + \nabla \loss_{\H}{(\theta_{t+1})}}^2}\\
    &\quad = \expect{\sup_{\norm{v} \leq 1} \frac{1}{\card{\H}} \sum_{i \in \H}\iprod{v}{g_{t+1}^{(i)} - \nabla \loss_i{(\theta_{t+1})} - \overline{g}_{t+1} + \nabla \loss_{\H}{(\theta_{t+1})} + \xi_{t+1}^{(i)} - \overline{\xi}_{t+1}}^2}\\
    &\leq 2\expect{\sup_{\norm{v} \leq 1} \frac{1}{\card{\H}} \sum_{i \in \H} \left[\iprod{v}{g_{t+1}^{(i)} - \nabla \loss_i{(\theta_{t+1})}-\overline{g}_{t+1} + \nabla \loss_{\H}{(\theta_{t+1})}}^2 + \iprod{v}{\xi_{t+1}^{(i)} - \overline{\xi}_{t+1}}^2 \right]}
\end{align*}
Now, recall the following bias-variance decomposition: for any $x_1, \ldots, x_n \in \R$ we have $\frac{1}{n} \sum_{i=1}^n (x_i - \overline{x})^2 = \frac{1}{n} \sum_{i=1}^n x_i^2 - \overline{x}^2 \leq \sum_{i=1}^n x_i^2$, where we denoted $\overline{x} \coloneqq \frac{1}{n} \sum_{i=1}^n x_i$.
Applying this fact above yields
\begin{align}
    &\expect{\sup_{\norm{v} \leq 1} \frac{1}{\card{\H}} \sum_{i \in \H}\iprod{v}{\tilde{g}_{t+1}^{(i)} - \nabla \loss_i{(\theta_{t+1})} - \overline{\widetilde{g}}_{t+1} + \nabla \loss_{\H}{(\theta_{t+1})}}^2}\nonumber\\
    &\leq 2\expect{\sup_{\norm{v} \leq 1} \frac{1}{\card{\H}} \sum_{i \in \H} \left[\iprod{v}{g_{t+1}^{(i)} - \nabla \loss_i{(\theta_{t+1})}}^2 + \iprod{v}{\xi_{t+1}^{(i)}}^2 \right]}\nonumber\\
    &\leq 2\expect{\frac{1}{\card{\H}} \sum_{i \in \H} \norm{g_{t+1}^{(i)} - \nabla \loss_i{(\theta_{t+1})}}^2}
    +2\expect{\sup_{\norm{v} \leq 1} \frac{1}{\card{\H}} \sum_{i \in \H}\iprod{v}{\xi_{t+1}^{(i)}}^2},
    \label{eq:heterodrift7}
\end{align}
where the last inequality is due to the Cauchy-Schwartz inequality.
Recall that, by \cref{asp:bnd_var} and \cref{lem:grad-subsample} applied with zero privacy noise, we have for every $i \in \H$ that $\expect{ \norm{g_{t+1}^{(i)} - \nabla \loss_i{(\theta_{t+1})}}^2} \leq  2(1-\frac{b}{m})\frac{\sigma^2}{b} \eqqcolon \sigma_b^2$.
Therefore, upon averaging over $i \in \H$, we have
\begin{align}
    \expect{\frac{1}{\card{\H}} \sum_{i \in \H}\norm{g_{t+1}^{(i)} - \nabla \loss_i{(\theta_{t+1})}}^2}
    \leq \sigma_b^2.
    \label{eq:heterodrift5}
\end{align}

We now bound the remaining (last) term on the RHS of \cref{eq:heterodrift7}.
By applying \cref{lem:concentration} to the random variables $(\xi^{(i)}_{t+1})_{i \in \H}$ which are drawn i.i.d. from $\N{(0,\sigmadp^2 I_d)}$, we obtain
\begin{align}
    \expect{\sup_{\norm{v} \leq 1} \frac{1}{\card{\H}} \sum_{i \in \H}\iprod{v}{\xi^{(i)}_{t+1}}^2}
    &\leq 36\sigmadp^2 \left(1+\frac{d}{n-f}\right).
    \label{eq:heterodrift8}
\end{align}

Plugging the bounds obtained in \cref{eq:heterodrift5,eq:heterodrift8} back in \cref{eq:heterodrift7}, we get
\begin{align}
    \expect{\sup_{\norm{v} \leq 1} \frac{1}{\card{\H}} \sum_{i \in \H}\iprod{v}{\tilde{g}_{t+1}^{(i)} - \nabla \loss_i{(\theta_{t+1})} - \overline{\widetilde{g}}_{t+1} + \nabla \loss_{\H}{(\theta_{t+1})}}^2}
    \leq 2\left(\sigma_b^2+36\sigmadp^2(1+\frac{d}{n-f})\right).
    \label{eq:heterodrift9}
\end{align}
We can now use the above bound of \cref{eq:heterodrift9} and that of \cref{eq:heterodrift6} to bound the RHS of \cref{eq:heterodrift4}, which yields
\begin{align*}
    &\expect{\sup_{\norm{v} \leq 1} \frac{1}{\card{\H}} \sum_{i \in \H} \iprod{v}{m^{(i)}_{t+1}-\overline{m}_{t+1}}^2}
    \leq \beta_t^2 \expect{\sup_{\norm{v} \leq 1} \frac{1}{\card{\H}} \sum_{i \in \H}\iprod{v}{m^{(i)}_{t}-\overline{m}_{t}}^2}\nonumber\\
    &\qquad+ (1-\beta_t)^2 \expect{\sup_{\norm{v} \leq 1} \frac{1}{\card{\H}} \sum_{i \in \H}\iprod{v}{\nabla \loss_i{(\theta_{t+1})}-\nabla \loss_{\H}{(\theta_{t+1})}}^2}
    +2(1-\beta_t)^2\left(\sigma_b^2+36\sigmadp^2(1+\frac{d}{n-f})\right)\nonumber\\
    &\qquad +\beta_t(1-\beta_t) \expect{\sup_{\norm{v} \leq 1} \frac{1}{\card{\H}} \sum_{i \in \H}\iprod{v}{m^{(i)}_{t}-\overline{m}_{t}}^2 + \frac{1}{\card{\H}} \sum_{i \in \H}\iprod{v}{\nabla \loss_i{(\theta_{t+1})}-\nabla \loss_{\H}{(\theta_{t+1})}}^2}.
\end{align*}
By rearranging terms, and noticing that $\beta_t^2 + \beta_t(1-\beta_t) = \beta_t$ and $(1-\beta_t)^2 + \beta_t(1-\beta_t) = 1-\beta_t$, we obtain
\begin{align*}
    &\expect{\sup_{\norm{v} \leq 1} \frac{1}{\card{\H}} \sum_{i \in \H} \iprod{v}{m^{(i)}_{t+1}-\overline{m}_{t+1}}^2}
    \leq \beta_t \expect{\sup_{\norm{v} \leq 1} \frac{1}{\card{\H}} \sum_{i \in \H}\iprod{v}{m^{(i)}_{t}-\overline{m}_{t}}^2}\nonumber\\
    &\qquad+ (1-\beta_t) \expect{\sup_{\norm{v} \leq 1} \frac{1}{\card{\H}} \sum_{i \in \H}\iprod{v}{\nabla \loss_i{(\theta_{t+1})}-\nabla \loss_{\H}{(\theta_{t+1})}}^2}
    +2(1-\beta_t)^2\left(\sigma_b^2+36\sigmadp^2(1+\frac{d}{n-f})\right).
\end{align*}
Denote $ \gcov^2 \coloneqq \sup_{\theta \in \R^d} \sup_{\norm{v}\leq 1} \frac{1}{\card{\H}}\sum_{i \in \H}\iprod{v}{\nabla{\loss_i{(\theta)}} - \nabla{\loss_{\H}{(\theta)}}}^2$.
Then, the above bound implies
\begin{align*}
    \expect{\sup_{\norm{v} \leq 1} \frac{1}{\card{\H}} \sum_{i \in \H} \iprod{v}{m^{(i)}_{t+1}-\overline{m}_{t+1}}^2}
    &\leq \beta_t \expect{\sup_{\norm{v} \leq 1} \frac{1}{\card{\H}} \sum_{i \in \H}\iprod{v}{m^{(i)}_{t}-\overline{m}_{t}}^2}\\
    &\quad +2(1-\beta_t)^2\left(\sigma_b^2+36\sigmadp^2(1+\frac{d}{n-f})\right)
    + (1-\beta_t)\gcov^2.
\end{align*}
The above inequality concludes the proof.
\end{proof}

\subsubsection{Proof of \cref{lem:dev}}
\label{app:lem-dev}
\deviation*
\begin{proof}
Let $t \in \{0, \ldots, T-1\}$.
Suppose that assumptions~\ref{asp:bnd_var} and~\ref{asp:bnd_norm} hold and that $\loss_\H$ is $L$-smooth.

Recall from~\eqref{eqn:dev} that 
\begin{align*}
    \dev{t+1} \coloneqq \AvgMmt{t+1} - \nabla \loss_{\H}\left( \weight{t+1} \right).
\end{align*}
Denote $\overline{\widetilde{g}}_t \coloneqq \frac{1}{\card{\H}}\sum_{i \in \H} \tilde{g}^{(i)}_t$.
Substituting from~\eqref{eqn:mmt_i} and recalling that $\overline{m}_t = \frac{1}{\card{\H}} \sum_{i \in \H} m^{(i)}_t$, we obtain
\begin{align*}
    \dev{t+1} = \beta_t \, \AvgMmt{t} + (1 - \beta_t) \, \overline{\widetilde{g}}_{t+1} - \nabla \loss_{\H}\left( \weight{t+1} \right).
\end{align*}
Upon adding and subtracting $\beta_t \nabla \loss_{\H}(\weight{t})$ and $\beta_t \nabla \loss_{\H}(\weight{t+1})$ on the R.H.S.~above we obtain that
\begin{align*}
    \dev{t+1} & = \beta_t \, \AvgMmt{t} - \beta_t \nabla \loss_{\H}(\weight{t}) + (1 - \beta_t) \, \overline{\widetilde{g}}_{t+1} - \nabla \loss_{\H}\left( \weight{t+1} \right) + \beta_t \nabla \loss_{\H}(\weight{t+1}) + \beta_t \nabla \loss_{\H}(\weight{t}) - \beta_t \nabla \loss_{\H}(\weight{t+1}) \\
    & = \beta_t \left( \AvgMmt{t} - \nabla \loss_{\H}(\weight{t}) \right) + (1 - \beta_t) \, \overline{\widetilde{g}}_{t+1} - (1 - \beta_t) \nabla \loss_{\H}\left( \weight{t+1} \right) + \beta_t \left( \nabla \loss_{\H}(\weight{t}) - \nabla \loss_{\H}(\weight{t+1})  \right).
\end{align*}
As $\AvgMmt{t} - \nabla \loss_{\H}(\weight{t}) = \dev{t}$ (by~\eqref{eqn:dev}), from above we obtain that
\begin{align*}
    \dev{t+1} = \beta_t \dev{t} + (1 - \beta_t) \, \left( \overline{\widetilde{g}}_{t+1} - \nabla \loss_{\H}\left( \weight{t+1} \right) \right) + \beta_t \left( \nabla \loss_{\H}(\weight{t}) - \nabla \loss_{\H}(\weight{t+1}) \right).
\end{align*}
Therefore, 
\begin{align*}
    \norm{\dev{t+1}}^2 = & \beta_t^2 \norm{\dev{t}}^2 + (1 - \beta_t)^2 \norm{ \overline{\widetilde{g}}_{t+1} - \nabla \loss_{\H}\left( \weight{t+1} \right)}^2 \\
    &+ \beta_t^2 \norm{\nabla \loss_{\H}(\weight{t}) - \nabla \loss_{\H}(\weight{t+1}) }^2 + 2 \beta_t (1 - \beta_t) \iprod{\dev{t}}{\overline{\widetilde{g}}_{t+1} - \nabla \loss_{\H}\left( \weight{t+1} \right)} \\
    & + 2 \beta_t^2 \iprod{\dev{t}}{\nabla \loss_{\H}(\weight{t}) - \nabla \loss_{\H}(\weight{t+1})} + 2 \beta_t ( 1- \beta_t) \iprod{\overline{\widetilde{g}}_{t+1} - \nabla \loss_{\H}\left( \weight{t+1} \right)}{\nabla \loss_{\H}(\weight{t}) - \nabla \loss_{\H}(\weight{t+1})}.
\end{align*}
By taking conditional expectation $\condexpect{t+1}{\cdot}$ on both sides, and recalling that $\dev{t}$, $\weight{t+1}$ and $\weight{t}$ are deterministic values when the history $\P_{t+1}$ is given, we obtain that 
\begin{align*}
    \condexpect{t+1}{\norm{\dev{t+1}}^2} = & \beta_t^2 \norm{\dev{t}}^2 + (1 - \beta_t)^2 \condexpect{t+1}{\norm{ \overline{\widetilde{g}}_{t+1} - \nabla \loss_{\H}\left( \weight{t+1} \right)}^2} + \beta_t^2 \norm{\nabla \loss_{\H}(\weight{t}) - \nabla \loss_{\H}(\weight{t+1}) }^2 + \\
    & 2 \beta_t (1 - \beta_t) \iprod{\dev{t}}{\condexpect{t+1}{\overline{\widetilde{g}}_{t+1}} - \nabla \loss_{\H}\left( \weight{t+1} \right)} + 2 \beta_t^2 \iprod{\dev{t}}{\nabla \loss_{\H}(\weight{t}) - \nabla \loss_{\H}(\weight{t+1})}\\
    & + 2 \beta_t ( 1- \beta_t) \iprod{\condexpect{t+1}{\overline{\widetilde{g}}_{t+1}} - \nabla \loss_{\H}\left( \weight{t+1} \right)}{\nabla \loss_{\H}(\weight{t}) - \nabla \loss_{\H}(\weight{t+1})}.
\end{align*}
Recall that $\overline{\widetilde{g}}_{t+1} \coloneqq \frac{1}{(n-f)} \sum_{j \in \H}\tilde{g}^{(i)}_{t+1}$. Thus, as we ignore clipping by Assumption~\ref{asp:bnd_norm}, we have $\condexpect{t+1}{\overline{\widetilde{g}}_{t+1}} = \nabla \loss_{\H}(\weight{t+1})$. Using this above we obtain that
\begin{align*}
    \condexpect{t+1}{\norm{\dev{t+1}}^2} = & \beta_t^2 \norm{\dev{t}}^2 + (1 - \beta_t)^2 \condexpect{t+1}{\norm{ \overline{\widetilde{g}}_{t+1} - \nabla \loss_{\H}\left( \weight{t+1} \right)}^2} + \beta_t^2 \norm{\nabla \loss_{\H}(\weight{t}) - \nabla \loss_{\H}(\weight{t+1}) }^2 \\
    & + 2 \beta_t^2 \iprod{\dev{t}}{\nabla \loss_{\H}(\weight{t}) - \nabla \loss_{\H}(\weight{t+1})}.
\end{align*}
Now, denote $\sigmadpbar^2 \coloneqq 2\left(1-\frac{b}{m}\right) \frac{\sigma^2}{b} + d \cdot \sigmadp^2$.
By assumptions~\ref{asp:bnd_var} and \ref{asp:bnd_norm}, we can invoke \cref{lem:grad-subsample} which implies, together with the fact that $\gradient{j}{t+1}$'s for $j \in \H$ are independent, that $\condexpect{t+1}{\norm{ \overline{\widetilde{g}}_{t+1} - \nabla \loss_{\H}\left( \weight{t+1} \right)}^2} \leq \frac{\sigmadpbar^2}{n-f}$. Thus,
\begin{align*}
    \condexpect{t+1}{\norm{\dev{t+1}}^2} \leq \beta_t^2 \norm{\dev{t}}^2 + (1 - \beta_t)^2 \frac{\sigmadpbar^2}{(n-f)} + \beta_t^2 \norm{\nabla \loss_{\H}(\weight{t}) - \nabla \loss_{\H}(\weight{t+1}) }^2 + 2 \beta_t^2 \iprod{\dev{t}}{\nabla \loss_{\H}(\weight{t}) - \nabla \loss_{\H}(\weight{t+1})}.
\end{align*}
By the Cauchy-Schwartz inequality, $\iprod{\dev{t}}{\nabla \loss_{\H}(\weight{t}) - \nabla \loss_{\H}(\weight{t+1})} \leq \norm{\dev{t}} \norm{\nabla \loss_{\H}(\weight{t}) - \nabla \loss_{\H}(\weight{t+1})}$. 
Since $\loss_\H$ is $L$-smooth, we have $\norm{ \nabla \loss_{\H}(\weight{t}) - \nabla \loss_{\H}(\weight{t+1})} \leq L \norm{\weight{t+1} - \weight{t}}$. 
Recall from~\eqref{eqn:SGD} that $\weight{t+1} = \weight{t} - \gamma_t  R_t$. Thus,$\norm{\nabla \loss_{\H}(\weight{t}) - \nabla \loss_{\H}(\weight{t+1})} \leq \gamma_t  L \norm{R_t}$. 
Using this above we obtain that
\begin{align*}
    \condexpect{t+1}{\norm{\dev{t+1}}^2} \leq \beta_t^2 \norm{\dev{t}}^2 + (1 - \beta_t)^2 \frac{\sigmadpbar^2}{(n-f)} + \gamma_t^2 \beta_t^2 L^2 \norm{R_t}^2 + 2 \gamma_t  \beta_t^2 L \norm{\dev{t}} \norm{R_t}.
\end{align*}
As $2 ab \leq a^2 + b^2$, from above we obtain that
\begin{align}
    \condexpect{t+1}{\norm{\dev{t+1}}^2} & \leq \beta_t^2 \norm{\dev{t}}^2 + (1 - \beta_t)^2 \frac{\sigmadpbar^2}{(n-f)} + \gamma_t^2 \beta_t^2 L^2 \norm{R_t}^2 + \gamma_t  L \beta_t^2 \left( \norm{\dev{t}}^2 +  \norm{R_t}^2\right) \nonumber \\
    & = (1 + \gamma_t L ) \beta_t^2 \norm{\dev{t}}^2 + (1 - \beta_t)^2 \frac{\sigmadpbar^2}{(n-f)} + \gamma_t L (1 + \gamma_t L) \beta_t^2  \norm{R_t}^2. \label{eqn:dev_before_ab}
\end{align}
By definition of $\drift{t}$ in~\eqref{eqn:drift}, we have $R_t = \drift{t} +  \AvgMmt{t}$. 
Thus, owing to the triangle inequality and the fact that $2 ab \leq a^2 + b^2$, we have $\norm{R_t}^2 \leq 2 \norm{\drift{t}}^2 + 2  \norm{\AvgMmt{t}}^2$. 
Similarly, by definition of $\dev{t}$ in~\eqref{eqn:dev}, we have $\norm{\AvgMmt{t}}^2 \leq 2 \norm{\dev{t}}^2 + 2 \norm{\nabla \loss_{\H}(\weight{t})}^2$. 
Thus, $\norm{R_t}^2 \leq 2 \norm{\drift{t}}^2 + 4  \norm{\dev{t}}^2 + 4  \norm{\nabla \loss_{\H}(\weight{t})}^2$. 
Using this in~\eqref{eqn:dev_before_ab} we obtain that
\begin{align*}
    \condexpect{t+1}{\norm{\dev{t+1}}^2} &\leq (1 + \gamma_t L ) \beta_t^2 \norm{\dev{t}}^2 + (1 - \beta_t)^2 \frac{\sigmadpbar^2}{(n-f)} \\
    &\quad + 2 \gamma_t L( 1 + \gamma_t L)\beta_t^2  \left( \norm{\drift{t}}^2 + 2  \norm{\dev{t}}^2 + 2  \norm{\nabla \loss_{\H}(\weight{t})}^2 \right).
\end{align*}
By rearranging the terms on the R.H.S., we get
\begin{align*}
    \condexpect{t+1}{\norm{\dev{t+1}}^2} \leq & \beta_t^2 (1 + \gamma_t L ) \left(1 + 4 \gamma_t   L \right) \norm{\dev{t}}^2 +  4 \gamma_t  L( 1 + \gamma_t L) \beta_t^2   \norm{\nabla \loss_{\H}(\weight{t})}^2 +(1 - \beta_t)^2 \frac{\sigmadpbar^2}{(n-f)} \\
    & + 2 \gamma_t  L( 1 + \gamma_t L)\beta_t^2 \norm{\drift{t}}^2.
\end{align*}
The proof concludes upon taking total expectation on both sides.
\end{proof}

\subsubsection{Proof of \cref{lem:descent}}
\label{app:descent}
\descent*
\begin{proof}
Let $t \in \{0, \ldots, T-1\}$.
Assuming $\loss_\H$ is $L$-smooth, we have (see~\cite{bottou2018optimization})
\begin{align*}
    \loss_{\H}(\weight{t+1}) - \loss_{\H}(\weight{t}) \leq \iprod{\weight{t+1} - \weight{t}}{\nabla \loss_{\H}(\weight{t})} + \frac{L}{2} \norm{\weight{t+1} - \weight{t}}^2.
\end{align*}
Substituting from~\eqref{eqn:sgd_new}, i.e., $\weight{t+1} = \weight{t} - \gamma_t    \, \AvgMmt{t} - \gamma_t  \drift{t}$, we obtain that
\begin{align*}
    \loss_{\H}(\weight{t+1}) - \loss_{\H}(\weight{t}) &\leq - \gamma_t   \iprod{\AvgMmt{t}}{\nabla \loss_{\H}(\weight{t})} - \gamma_t  \iprod{\drift{t}}{\nabla \loss_{\H}(\weight{t})} + \gamma_t^2 \frac{L}{2} \norm{ \, \AvgMmt{t} + \drift{t}}^2 \\
    & = - \gamma_t   \iprod{\AvgMmt{t} - \nabla \loss_{\H}(\weight{t}) + \nabla \loss_{\H}(\weight{t})}{\nabla \loss_{\H}(\weight{t})} - \gamma_t  \iprod{\drift{t}}{\nabla \loss_{\H}(\weight{t})} + \gamma_t^2 \frac{L}{2} \norm{ \, \AvgMmt{t} + \drift{t}}^2.
\end{align*}
By Definition~\eqref{eqn:dev}, $\AvgMmt{t} - \nabla \loss_{\H}(\weight{t}) = \dev{t}$. Thus, from above we obtain 
\begin{align}
     \loss_{\H}(\weight{t+1}) - \loss_{\H}(\weight{t}) \leq -  \gamma_t   \norm{\nabla \loss_{\H}(\weight{t})}^2 -  \gamma_t   \iprod{\dev{t}}{\nabla \loss_{\H}(\weight{t})} -  \gamma_t  \iprod{\drift{t}}{\nabla \loss_{\H}(\weight{t})} + \frac{1}{2}\gamma_t^2 L \norm{ \, \AvgMmt{t} + \drift{t}}^2. \label{eqn:norm_1}
\end{align}
Now, we consider the last three terms on the R.H.S.~separately. Using Cauchy-Schwartz inequality, and the fact that $2 ab \leq \frac{1}{c} a^2 + c b^2$ for any $c > 0$, we obtain that (by substituting $c = 2$)
\begin{align}
    2 \mnorm{\iprod{\dev{t}}{\nabla \loss_{\H}(\weight{t})}} \leq 2 \norm{\dev{t}} \norm{\nabla \loss_{\H}(\weight{t})} \leq \frac{2}{1} \norm{\dev{t}}^2 + \frac{1}{2} \norm{\nabla \loss_{\H}(\weight{t})}^2 . \label{eqn:rho_1}
\end{align}
Similarly, 
\begin{align}
    2 \mnorm{\iprod{\drift{t}}{\nabla \loss_{\H}(\weight{t})}} \leq 2 \norm{\drift{t}} \norm{\nabla \loss_{\H}(\weight{t})} \leq \frac{ 2}{ 1} \norm{\drift{t}}^2 +  \frac{1}{2} \norm{\nabla \loss_{\H}(\weight{t})}^2. \label{eqn:rho_2}
\end{align}
Finally, using triangle inequality and the fact that $2ab \leq a^2 + b^2$ we have
\begin{align}
    \norm{ \, \AvgMmt{t} + \drift{t}}^2 & \leq 2  \, \norm{\AvgMmt{t}}^2 + 2 \norm{\drift{t}}^2 = 2  \, \norm{\AvgMmt{t} - \nabla \loss_{\H}(\weight{t+1}) + \nabla \loss_{\H}(\weight{t})}^2 + 2 \norm{\drift{t}}^2 \nonumber \\
    & \leq 4  \, \norm{\dev{t}}^2 + 4  \, \norm{\nabla \loss_{\H}(\weight{t})}^2 + 2 \norm{\drift{t}}^2. \quad \quad [\text{since} ~ ~ \AvgMmt{t} - \nabla \loss_{\H}(\weight{t}) = \dev{t}] \label{eqn:last_term}
\end{align}
Substituting from~\eqref{eqn:rho_1},~\eqref{eqn:rho_2} and~\eqref{eqn:last_term} in~\eqref{eqn:norm_1} we obtain that
\begin{align*}
    \loss_{\H}(\weight{t+1}) - \loss_{\H}(\weight{t}) \leq & -  \gamma_t   \norm{\nabla \loss_{\H}(\weight{t})}^2 
    + \frac{1}{2} \gamma_t   \left( 2 \norm{\dev{t}}^2 + \frac{1}{2}\norm{\nabla \loss_{\H}(\weight{t})}^2 \right) 
    + \frac{1}{2}\gamma_t  \left( 2 \norm{\drift{t}}^2 + \frac{1}{2} \norm{\nabla \loss_{\H}(\weight{t})}^2 \right) \nonumber \\
    & + \frac{1}{2}\gamma_t^2 L \left( 4  \, \norm{\dev{t}}^2 + 4  \, \norm{\nabla \loss_{\H}(\weight{t})}^2 + 2 \norm{\drift{t}}^2 \right).
\end{align*}
Upon rearranging the terms in the R.H.S., we obtain that
\begin{align*}
    \loss_{\H}(\weight{t+1}) - \loss_{\H}(\weight{t}) \leq - \frac{\gamma_t}{2}\left( 1 - 4 \gamma_t  L \right) \norm{\nabla \loss_{\H}(\weight{t})}^2  +
    \gamma_t   \left( 1 + 2 \gamma_t  L  \right) \norm{\dev{t}}^2 +
    \gamma_t  \left(  1 + \gamma_t  L \right) \norm{\drift{t}}^2.
\end{align*}
This concludes the proof.
\end{proof}

\clearpage
\section{Experimental Evaluation}
\label{app:experiments}
In Section~\ref{app:exp-setup}, we present our experimental setup.
In Section~\ref{app:exp-results}, we report our empirical results.

\subsection{Experimental Setup}
\label{app:exp-setup}
In our experiments, we test the performance of \algoname{} using SMEA and Filter~\cite{diakonikolas2017being, data2021byzantine} in the server-based architecture and in three privacy regimes.

\paragraph{Dataset, model architecture, and hyperparameters.}
We train a logistic regression model of $d = 69$ parameters on the academic \textit{Phishing}\footnote{\url{https://www.csie.ntu.edu.tw/~cjlin/libsvmtools/datasets/}} dataset. We employ the \textit{binary cross entropy} (bce) loss as well as L2-regularization of parameter $\lambda = 10^{-4}$, making the underlying learning problem strongly convex. We train the model using a fixed learning rate $\gamma = 1$ over a total of $T = 400$ learning steps. We set the clipping threshold $C = 1$ and the batch size $b = 25$. We run all algorithms, except DSGD, with momentum $\beta = 0.99$.

\paragraph{Distributed setup, and privacy accounting.}
We consider a server-based architecture composed of $n = 7$ workers, among which $f = 3$ are adversarial. The honest workers inject a privacy noise $\sigmadp = \frac{2C}{b} \times \sigma_{\mathrm{NM}}$ to their gradients, where $\sigma_{\mathrm{NM}}$ is referred to as the noise multiplier. We consider three privacy regimes in our experiments; namely \textit{low} privacy where $\sigma_{\mathrm{NM}} = 1$, \textit{moderate} privacy where $\sigma_{\mathrm{NM}}=2$, and \textit{high} privacy where $\sigma_{\mathrm{NM}} = 3$.
In order to estimate the privacy budgets achieved at the end of the learning, we use Opacus~\cite{opacus}, a DP library for deep learning in PyTorch~\cite{pytorch}. Using Opacus, the aggregate privacy budgets after $T = 400$ steps of learning are $(\epsilon, \delta) = (1.14, 10^{-4})$ in the \textit{low} privacy regime, $(\epsilon, \delta) = (0.32, 10^{-4})$ in the \textit{moderate} privacy regime, and $(\epsilon, \delta) = (0.19, 10^{-4})$ in the \textit{high} privacy regime.

\paragraph{Evaluation details and reproducibility.}
As a benchmark, we compare the performance of \algoname{} against the DP-DSGD algorithm, i.e., the private version of the adversary-free DSGD. We test \algoname{} using SMEA and Filter. These algorithms are obtained by running Algorithm~\ref{algo:robust-dpsgd} while replacing the aggregation method $F$ with the robust algorithm in question, namely SMEA and Filter. Note that we run Filter with spectral norm bound $\sigma_0^2 = 0$ (see Section~\ref{app:filter}) because it provides the best empirical results, and it cannot be set to its theoretical value since the values of data heterogeneity $G^2$ and stochastic gradient noise $\sigma^2$ are unknown.
We run each experiment with five seeds from 1 to 5 for reproducibility. The code we use to launch the different experiments will be made available.

\paragraph{Adversarial attacks.}
In our experiments, the adversarial workers execute four state-of-the-art attacks from the robust distributed ML literature, namely A Little is Enough (ALIE)~\cite{little}, Fall of Empires (FOE)~\cite{empire}, Sign-flipping (SF)~\cite{allen2020byzantine}, and Label-flipping (LF)~\cite{allen2020byzantine}.\\
The first three attacks rely on the same attack primitive that we explain below, while LF is executed differently.\\
Let $b_t$ be the attack vector in step $t$ and $\tau \geq 0$ a fixed real number. In every step $t$, the adversarial workers send to the server the gradient $B_t = \overline{g}_t + \tau_t b_t$, where $\overline{g}_t$ is an estimation of the true gradient at step $t$. Experimentally, we set $\overline{g}_t = \frac{1}{|\mathcal{H}|} \sum\limits_{i \in \mathcal{H}} g_t^{(i)}$.

\begin{itemize}
    \item \textbf{ALIE:} In this attack, $b_t = \sigma_t$, where $\sigma_t$ is coordinate-wise standard deviation of $\overline{g}_t$. In our experiments on ALIE, $\tau_t$ is chosen through an extensive grid search. Essentially, in each step $t$, we choose the value that results in the worst adversarial vector, i.e, the vector for which the  distance to $\overline{g}_t$ is the largest.
    \item \textbf{FOE:} In this attack, $b_t = - \overline{g}_t$. All adversarial workers thus send $(1 - \tau_t) \overline{g}_t$ in step $t$. Similar to \textit{ALIE}, $\tau_t$ for \textit{FoE} is also estimated through grid searching.
    \item \textbf{SF:} In this attack, $b_t = - \overline{g}_t$, and $\tau_t = 2$. All adversarial workers thus send $B_t = b_t = - \overline{g}_t$ in step $t$.
    \item \textbf{LF:} Every adversarial worker computes its gradient on flipped labels. Since the labels $l$ for Phishing are in $\{0, 1\}$, the adversarial workers flip the labels by computing $l' = 1 - l$ on the batch, where $l'$ is the flipped/modified label.
\end{itemize}
\subsection{Experimental Results}\label{app:exp-results}
We present our results in the \textit{low} privacy regime in Figures~\ref{fig:plots_phishing_1} and~\ref{fig:plots_phishing_2}, in the \textit{mid} privacy regime in Figures~\ref{fig:plots_phishing_3} and~\ref{fig:plots_phishing_4}, and finally in the \textit{high} privacy regime in Figures~\ref{fig:plots_phishing_5} and~\ref{fig:plots_phishing_6}. We then comment on the results below.

\clearpage
\paragraph{Low Privacy Regime ($\sigma_{\mathrm{NM}} = 1$).}
\begin{figure*}[ht!]
    \centering
    \includegraphics[width=0.5\textwidth]{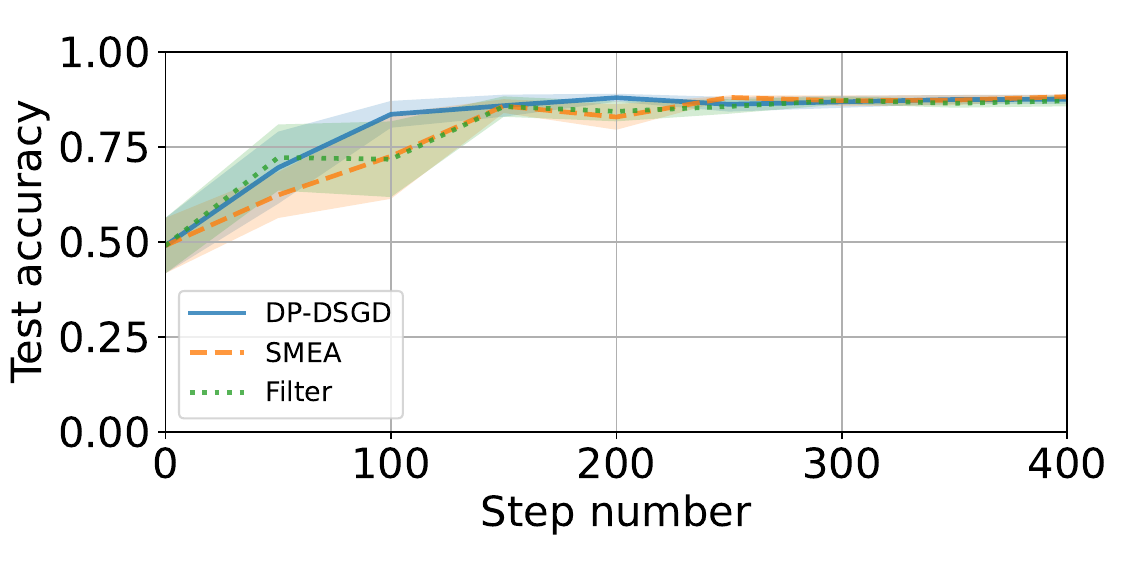}%
    \includegraphics[width=0.5\textwidth]{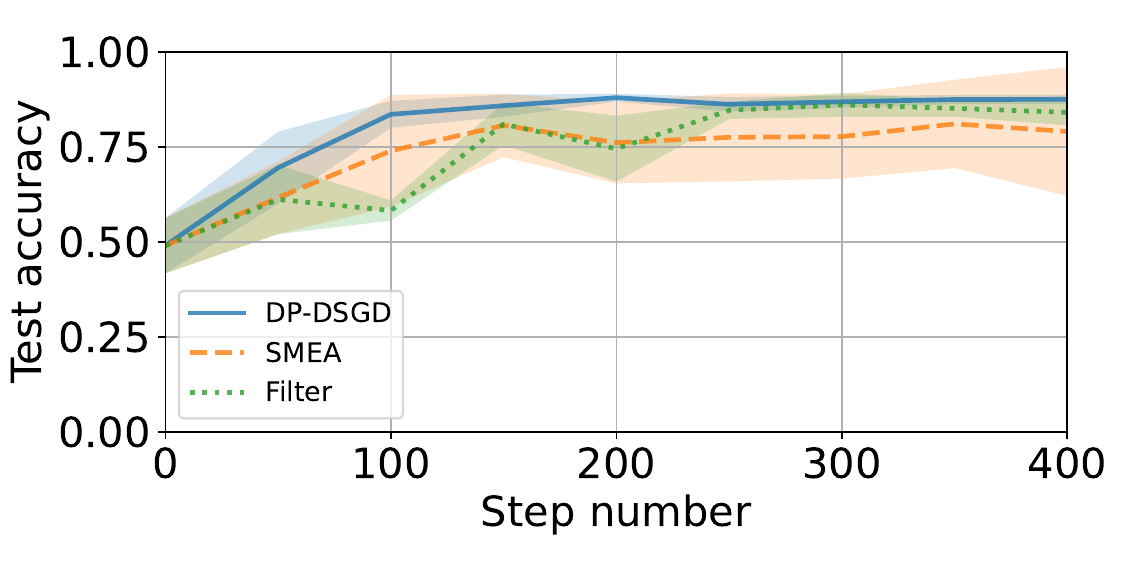}\\%
    \includegraphics[width=0.5\textwidth]{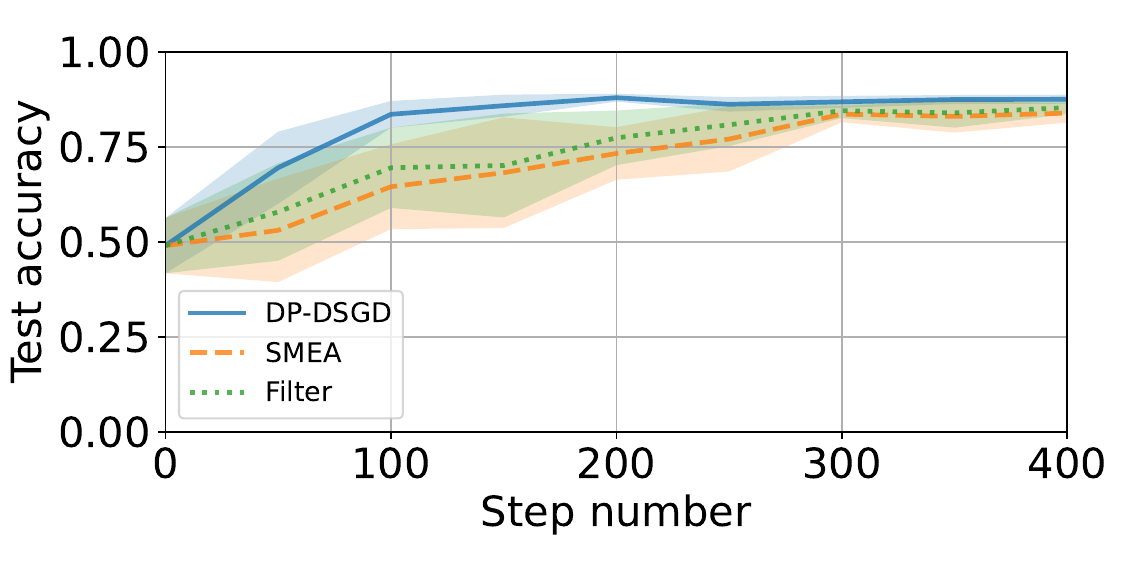}%
    \includegraphics[width=0.5\textwidth]{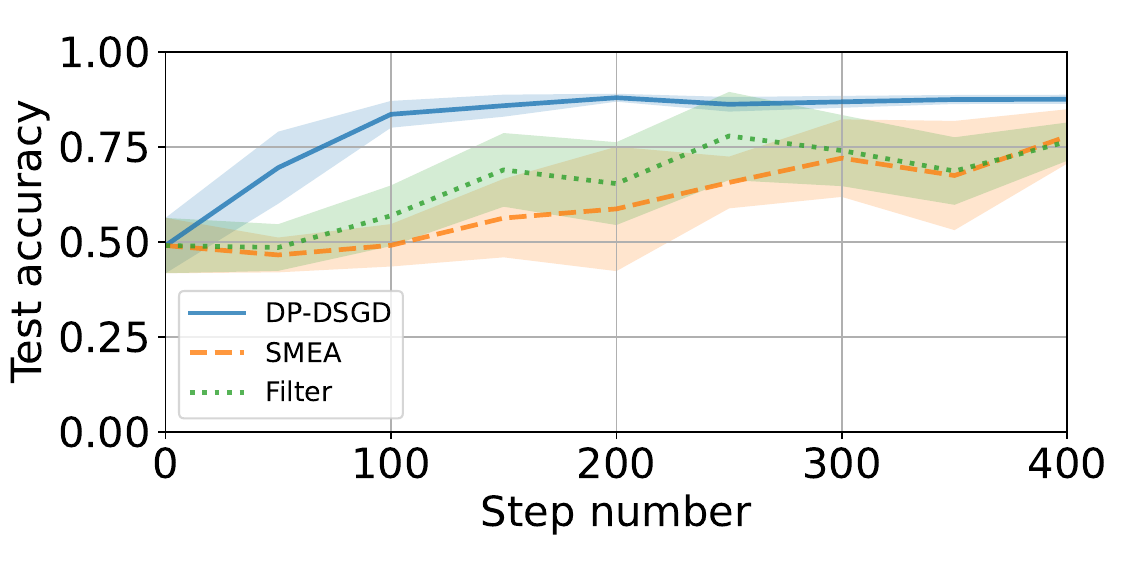}\\%
    \caption{Test accuracy on Phishing with $f = 3$ adversarial workers among $n = 7$ workers, with $\beta = 0.99$. The adversarial workers execute the LF (\textit{row 1, left}), SF (\textit{row 1, right}), ALIE (\textit{row 2, left}), and FOE (\textit{row 2, right}) attacks. Privacy budget after $T = 400$ steps is $(\epsilon, \delta) = (1.14, 10^{-4})$.}
\label{fig:plots_phishing_1}
\end{figure*}

\begin{figure*}[ht!]
    \centering
    \includegraphics[width=0.5\textwidth]{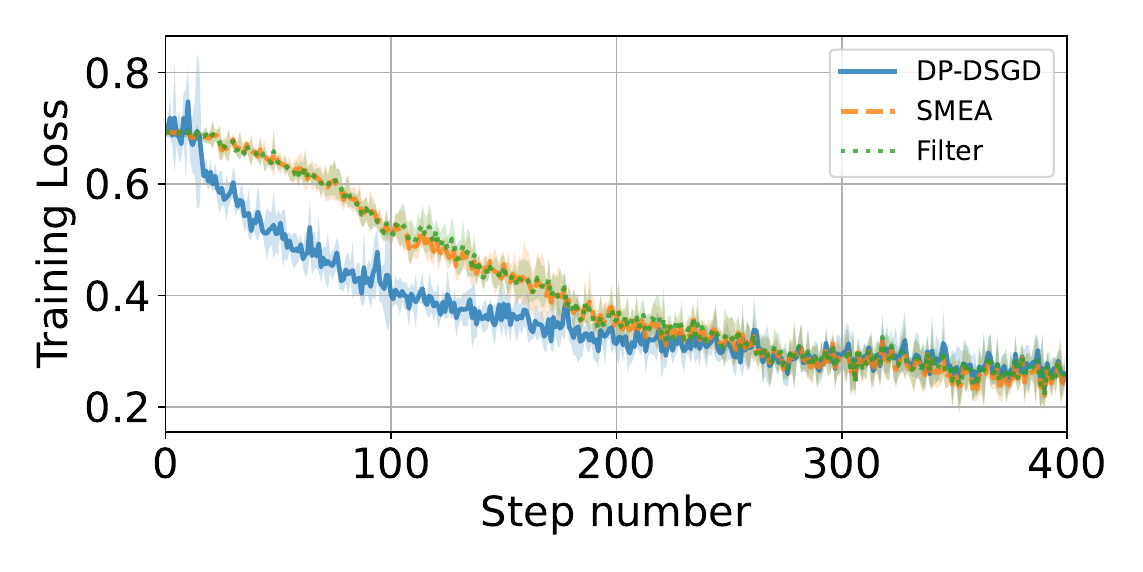}%
    \includegraphics[width=0.5\textwidth]{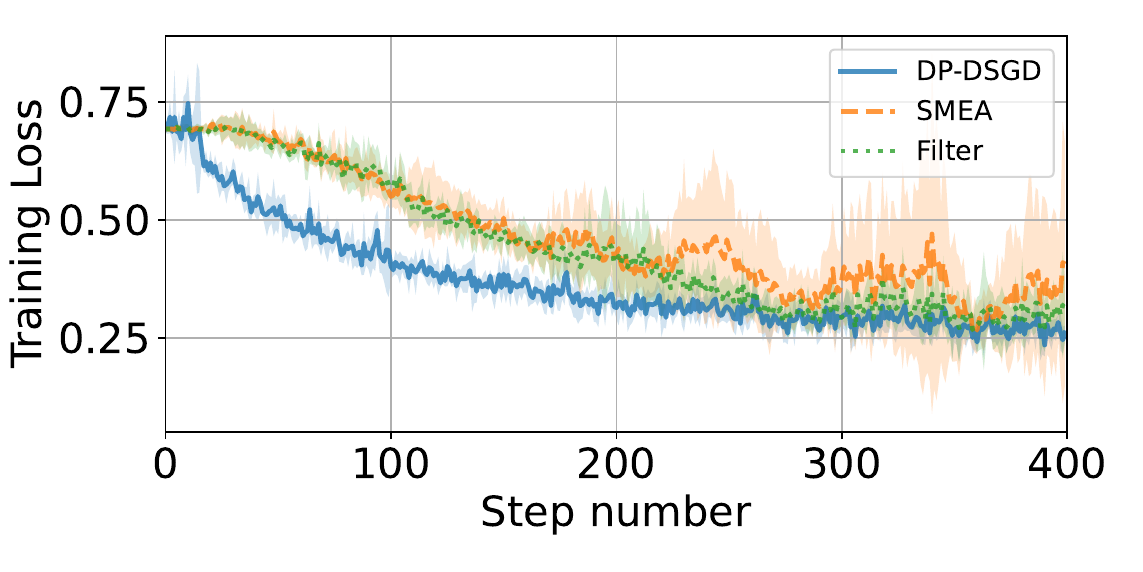}\\%
    \includegraphics[width=0.5\textwidth]{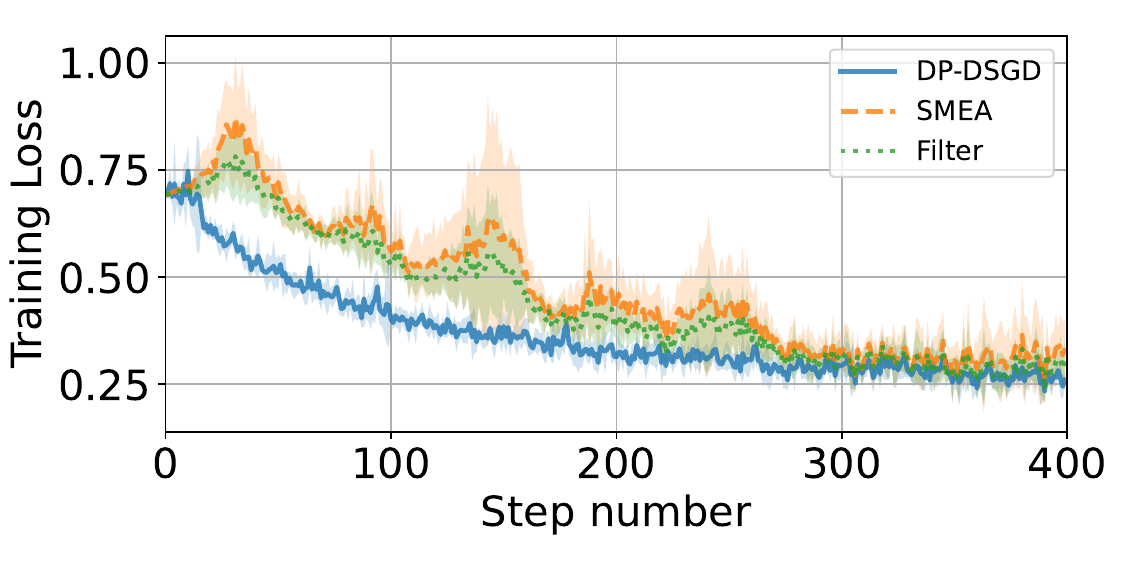}%
    \includegraphics[width=0.5\textwidth]{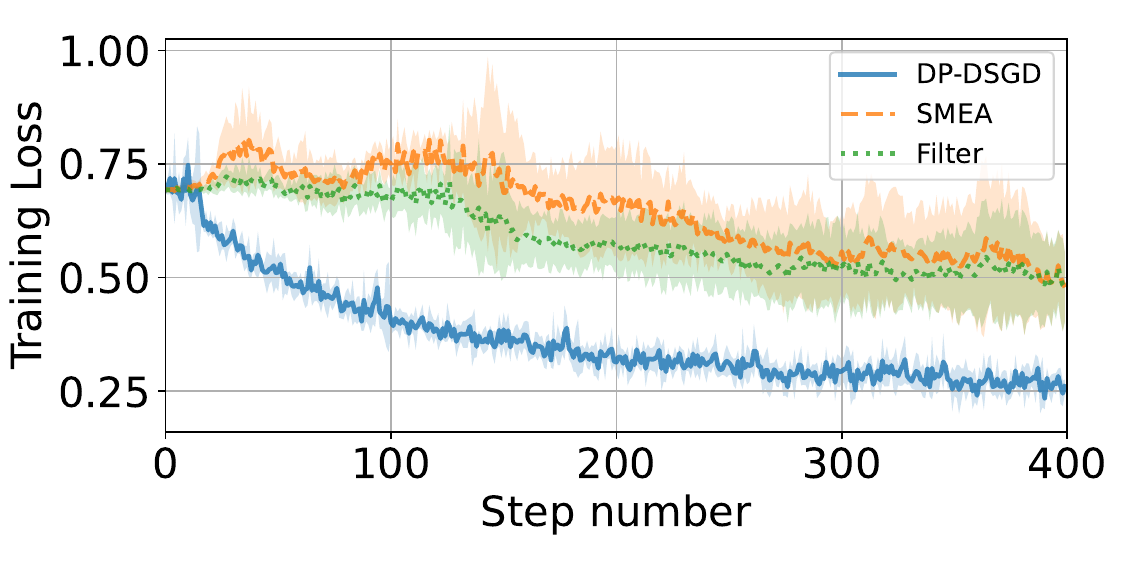}\\%
    \caption{Training loss on Phishing with $f = 3$ adversarial workers among $n = 7$ workers, with $\beta = 0.99$. The adversarial workers execute the LF (\textit{row 1, left}), SF (\textit{row 1, right}), ALIE (\textit{row 2, left}), and FOE (\textit{row 2, right}) attacks. Privacy budget after $T = 400$ steps is $(\epsilon, \delta) = (1.14, 10^{-4})$.}
\label{fig:plots_phishing_2}
\end{figure*}

\clearpage
\paragraph{Moderate Privacy Regime ($\sigma_{\mathrm{NM}} = 2$).}
\begin{figure*}[ht!]
    \centering
    \includegraphics[width=0.5\textwidth]{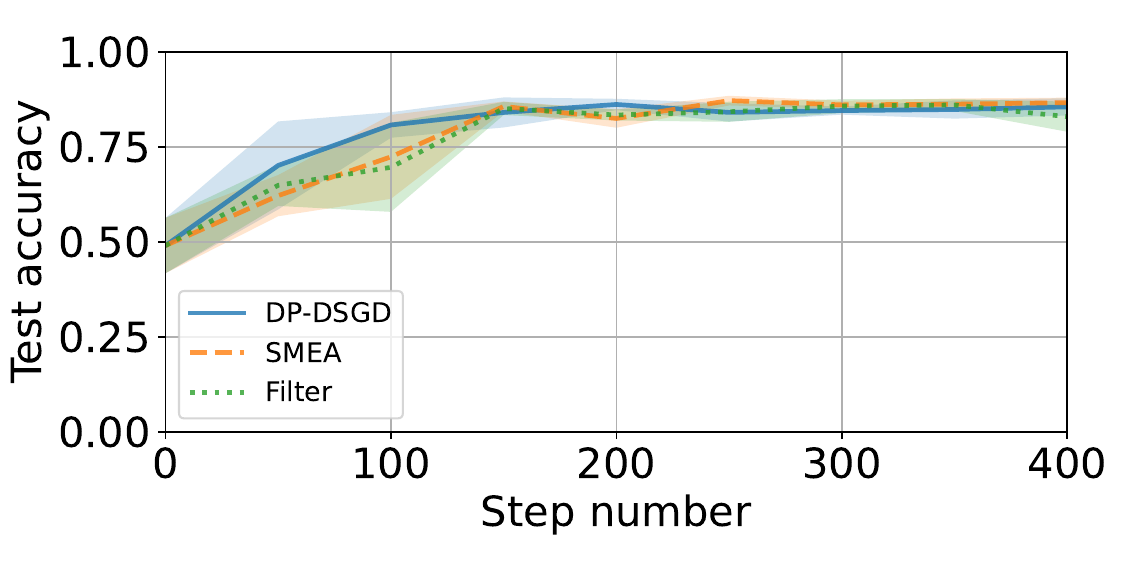}%
    \includegraphics[width=0.5\textwidth]{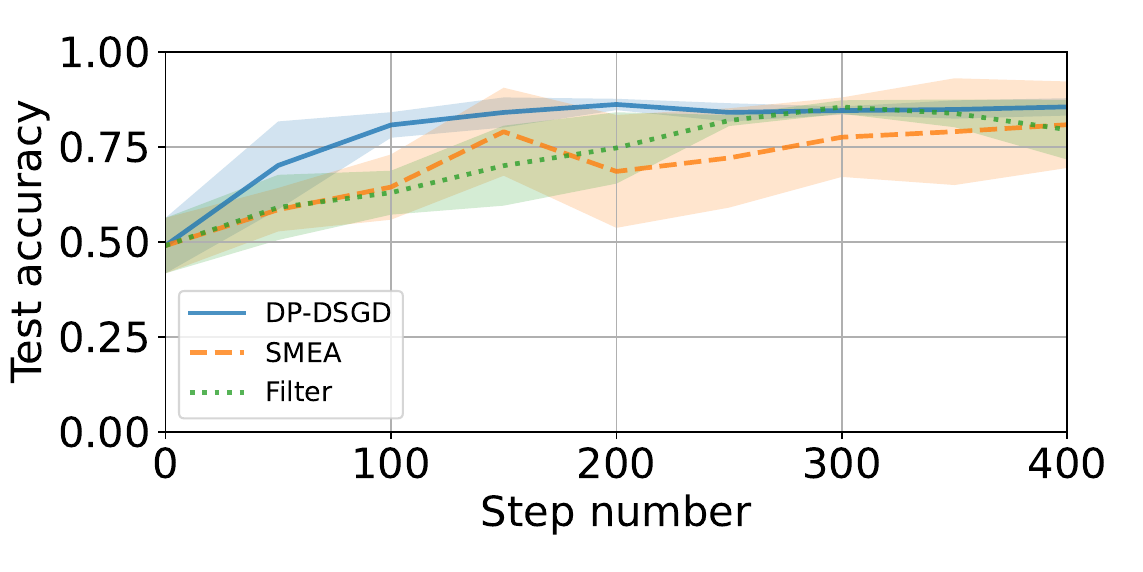}\\%
    \includegraphics[width=0.5\textwidth]{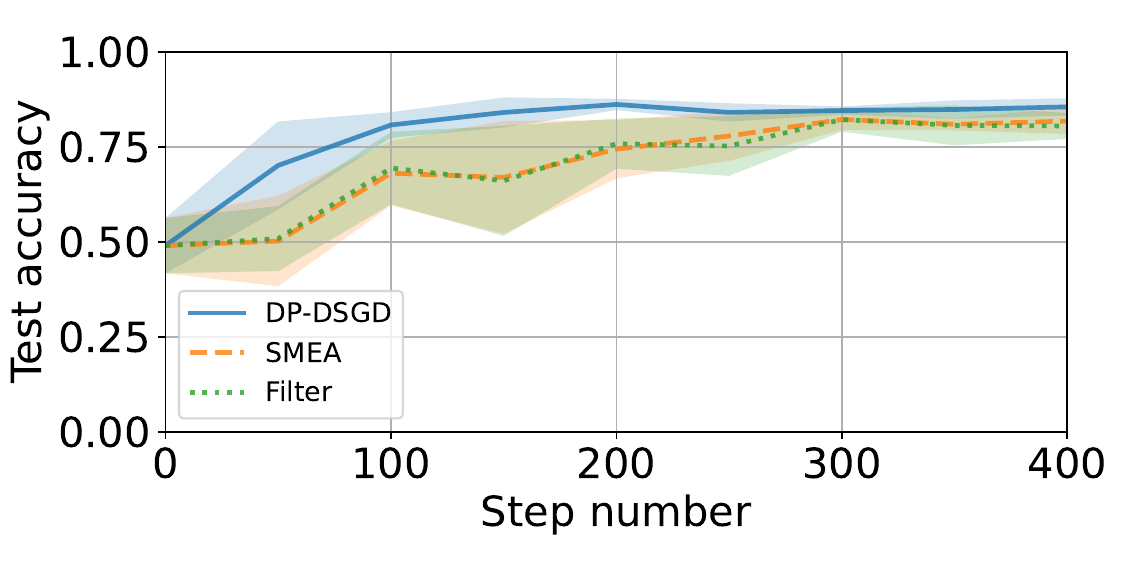}%
    \includegraphics[width=0.5\textwidth]{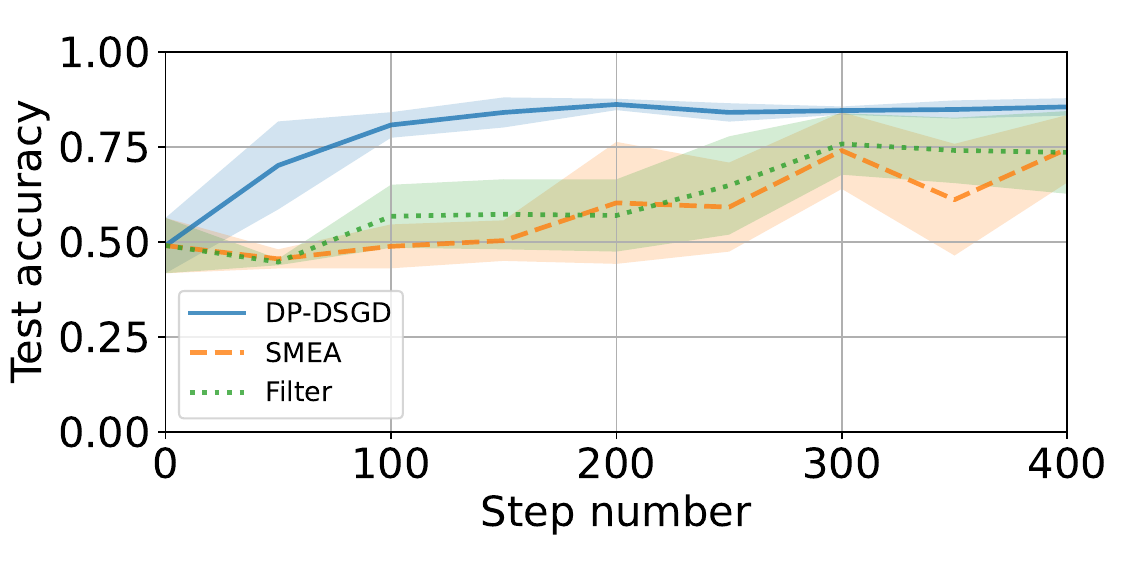}\\%
    \caption{Test accuracy on Phishing with $f = 3$ adversarial workers among $n = 7$ workers, with $\beta = 0.99$. The adversarial workers execute the LF (\textit{row 1, left}), SF (\textit{row 1, right}), ALIE (\textit{row 2, left}), and FOE (\textit{row 2, right}) attacks. Privacy budget after $T = 400$ steps is $(\epsilon, \delta) = (0.32, 10^{-4})$.}
\label{fig:plots_phishing_3}
\end{figure*}

\begin{figure*}[ht!]
    \centering
    \includegraphics[width=0.5\textwidth]{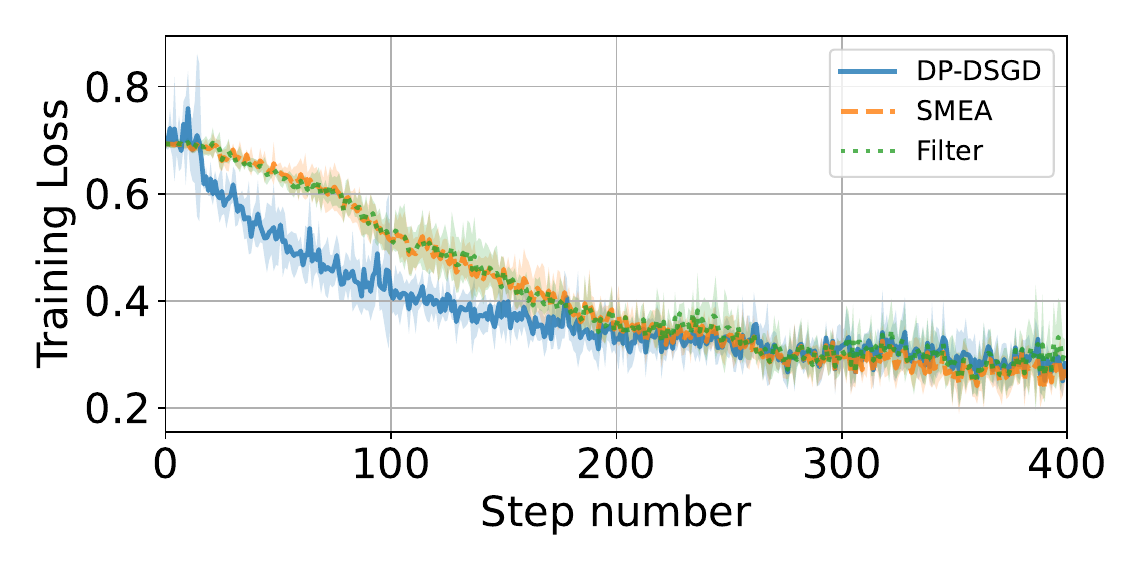}%
    \includegraphics[width=0.5\textwidth]{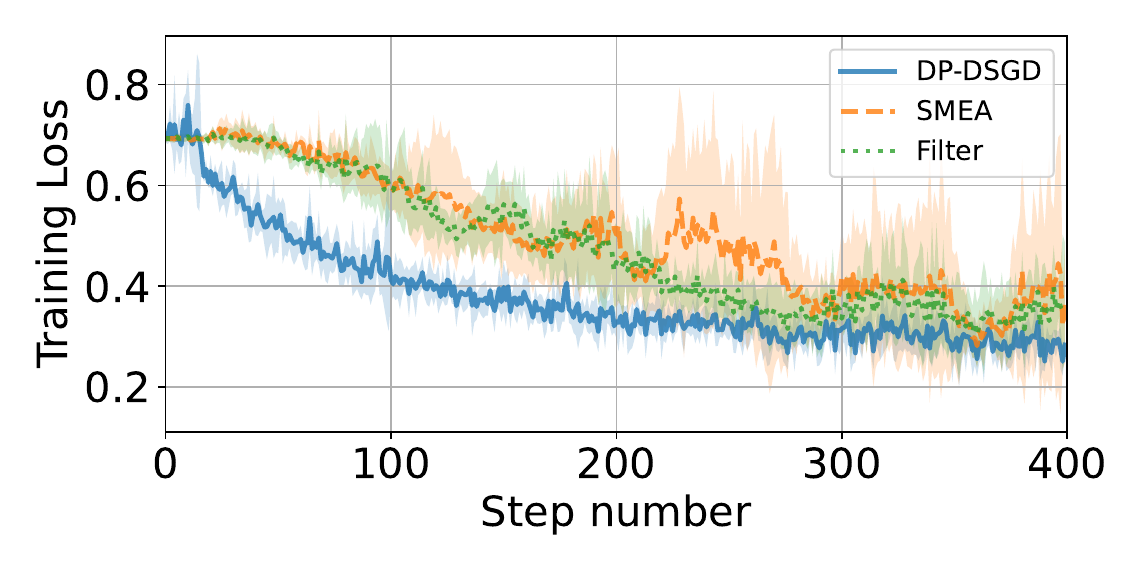}\\%
    \includegraphics[width=0.5\textwidth]{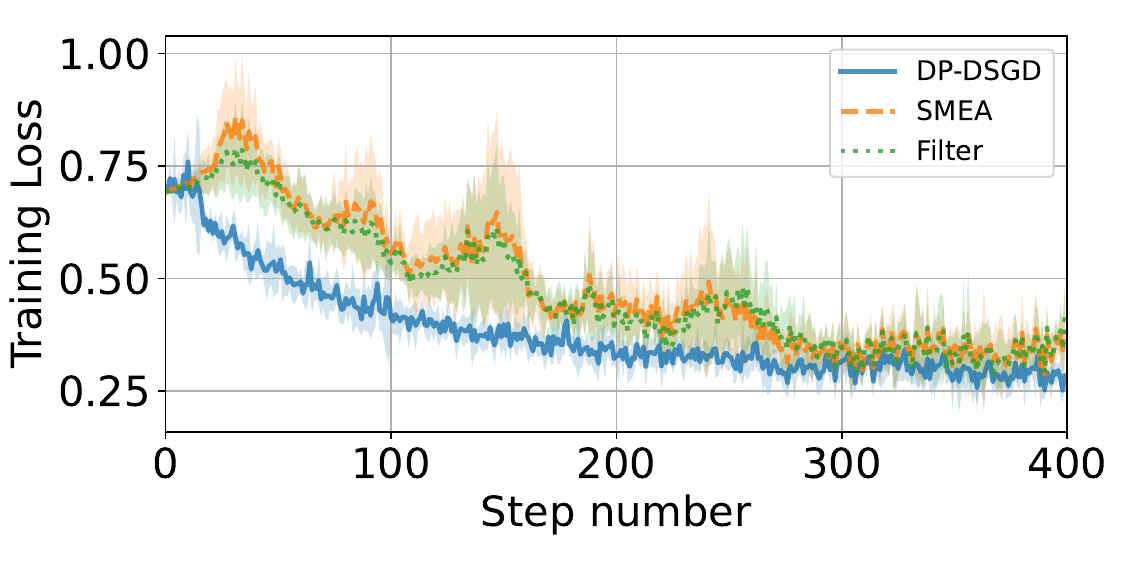}%
    \includegraphics[width=0.5\textwidth]{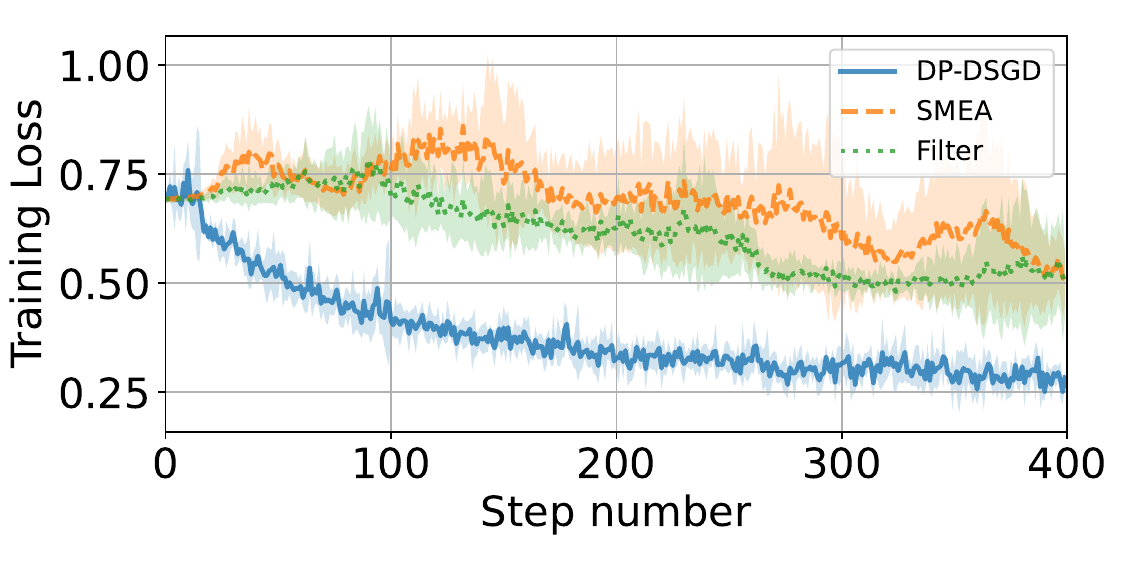}\\%
    \caption{Training loss on Phishing with $f = 3$ adversarial workers among $n = 7$ workers, with $\beta = 0.99$. The adversarial workers execute the LF (\textit{row 1, left}), SF (\textit{row 1, right}), ALIE (\textit{row 2, left}), and FOE (\textit{row 2, right}) attacks. Privacy budget after $T = 400$ steps is $(\epsilon, \delta) = (0.32, 10^{-4})$.}
\label{fig:plots_phishing_4}
\end{figure*}

\clearpage
\paragraph{High Privacy Regime ($\sigma_{\mathrm{NM}} = 3$).}
\begin{figure*}[ht!]
    \centering
    \includegraphics[width=0.5\textwidth]{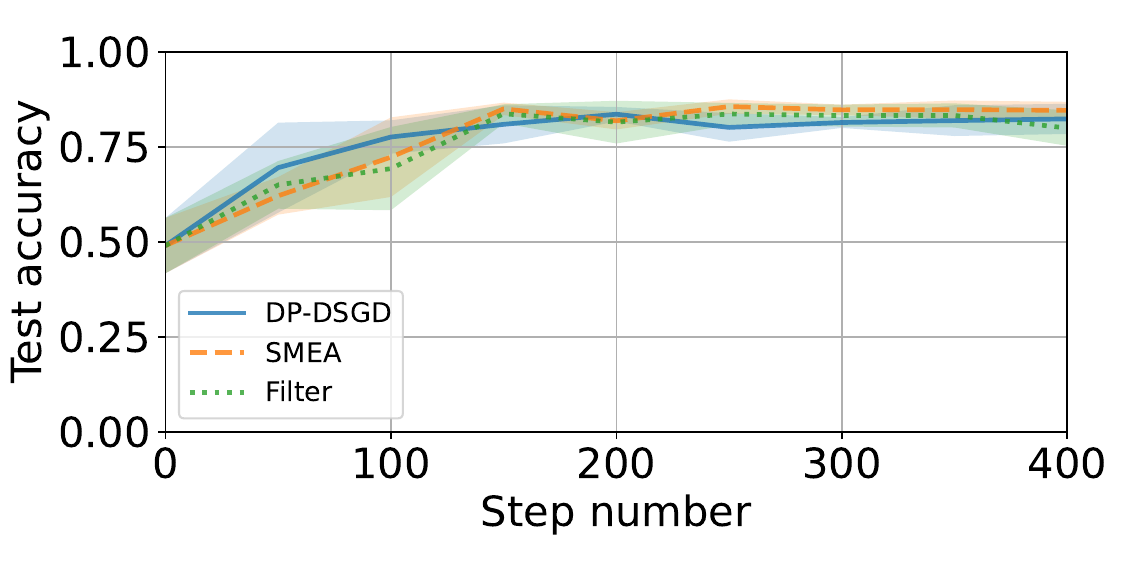}%
    \includegraphics[width=0.5\textwidth]{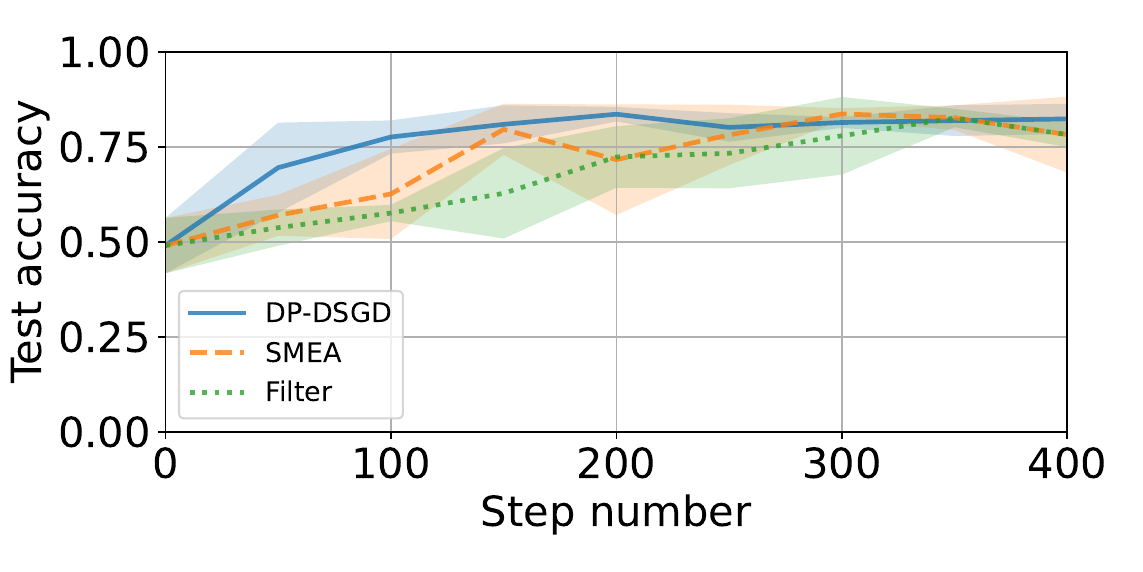}\\%
    \includegraphics[width=0.5\textwidth]{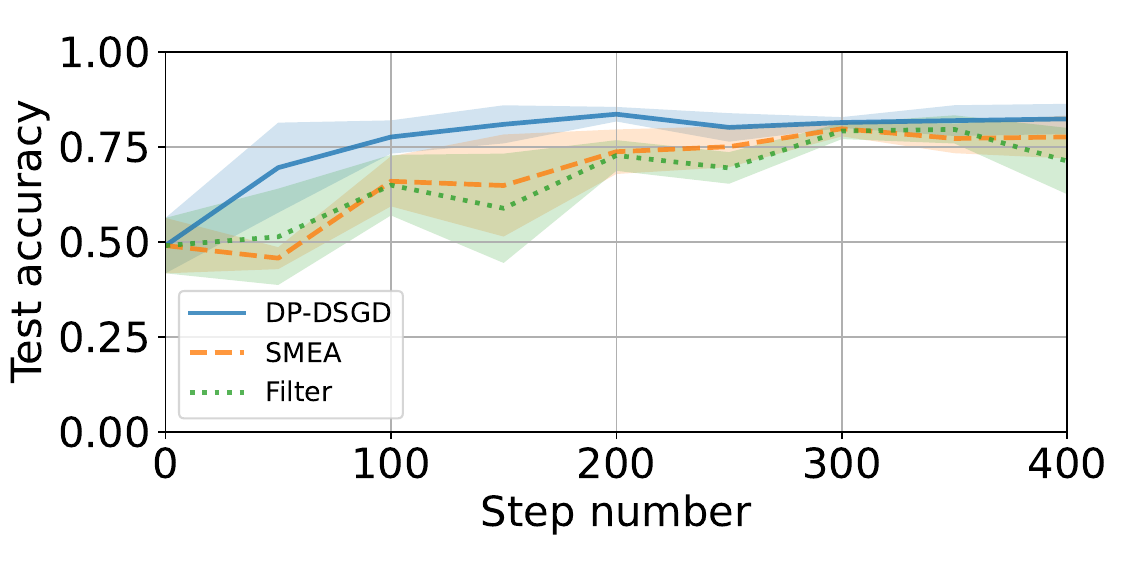}%
    \includegraphics[width=0.5\textwidth]{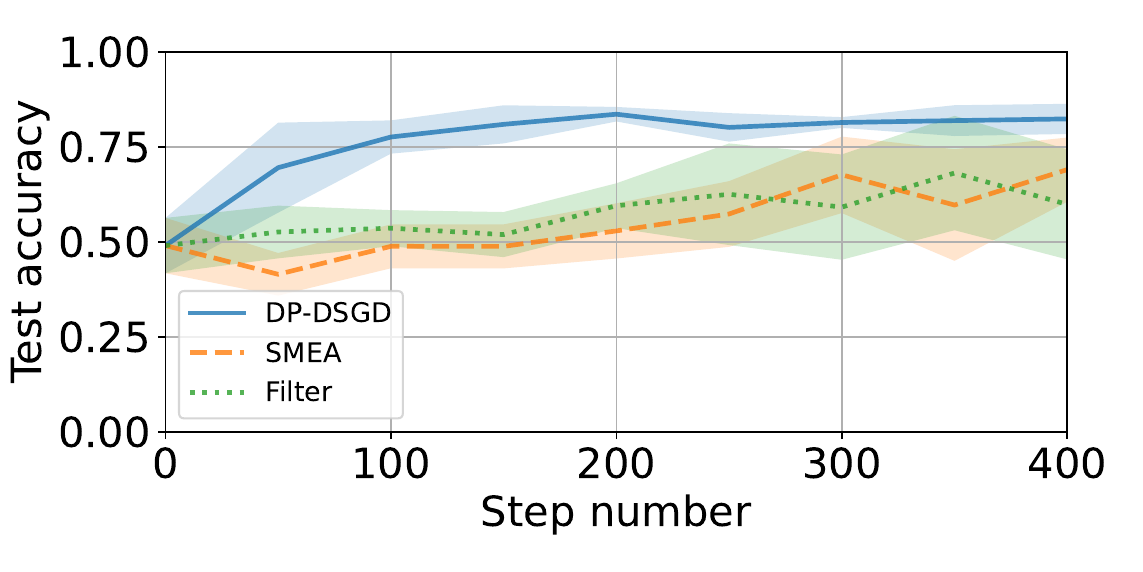}\\%
    \caption{Test accuracy on Phishing with $f = 3$ adversarial workers among $n = 7$ workers, with $\beta = 0.99$. The adversarial workers execute the LF (\textit{row 1, left}), SF (\textit{row 1, right}), ALIE (\textit{row 2, left}), and FOE (\textit{row 2, right}) attacks. Privacy budget after $T = 400$ steps is $(\epsilon, \delta) = (0.19, 10^{-4})$.}
\label{fig:plots_phishing_5}
\end{figure*}

\begin{figure*}[ht!]
    \centering
    \includegraphics[width=0.5\textwidth]{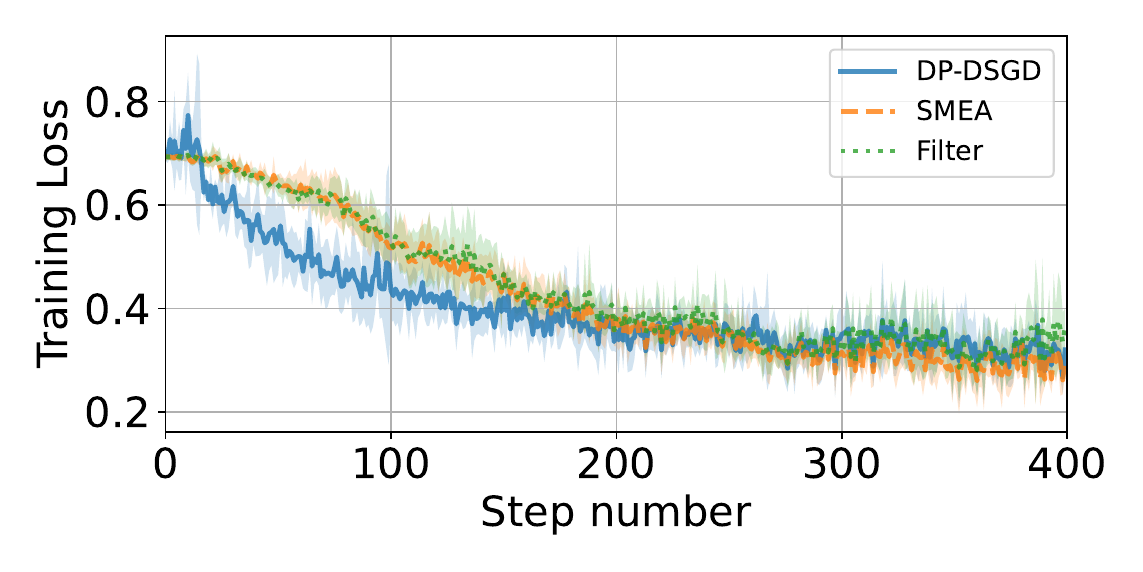}%
    \includegraphics[width=0.5\textwidth]{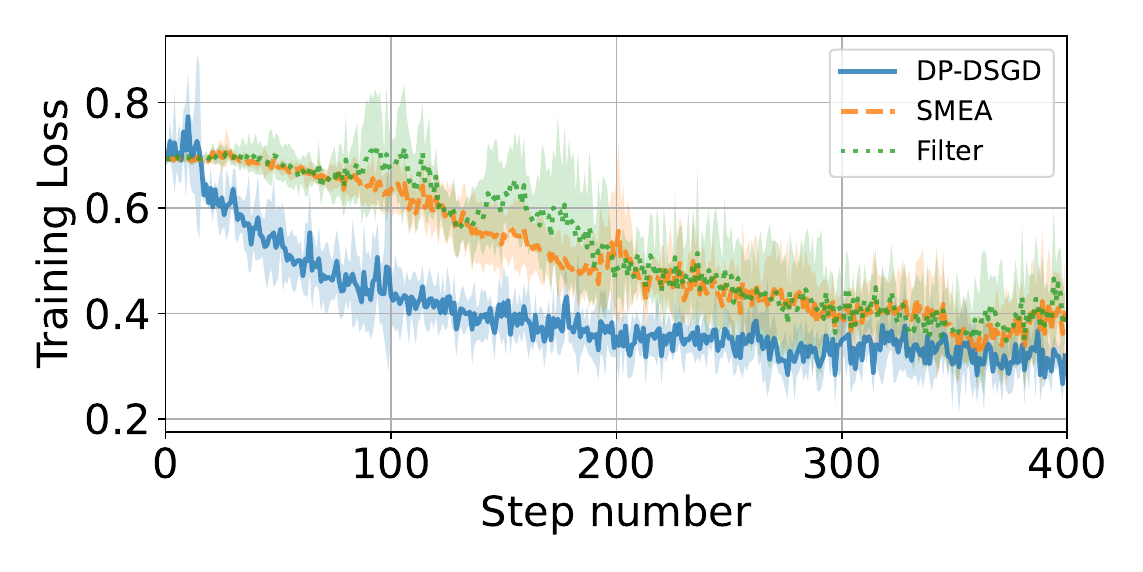}\\%
    \includegraphics[width=0.5\textwidth]{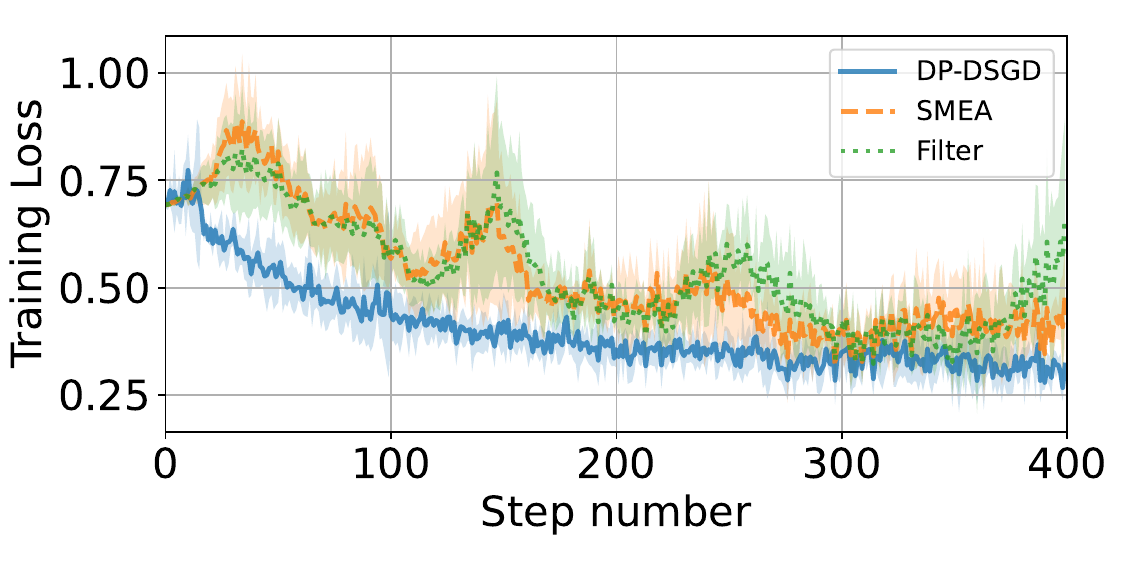}%
    \includegraphics[width=0.5\textwidth]{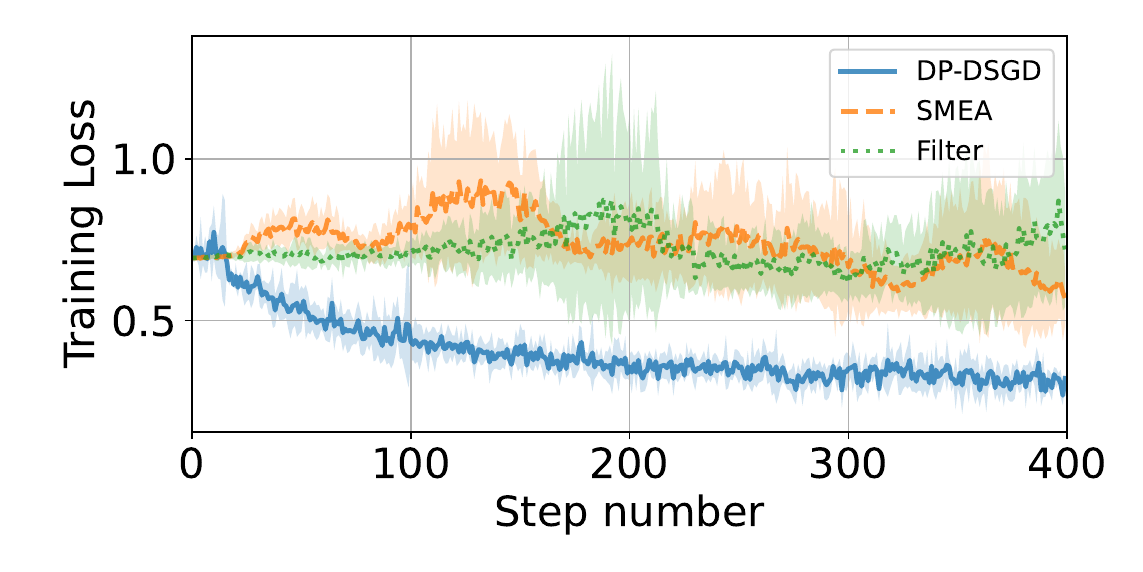}\\%
    \caption{Training loss on Phishing with $f = 3$ adversarial workers among $n = 7$ workers, with $\beta = 0.99$. The adversarial workers execute the LF (\textit{row 1, left}), SF (\textit{row 1, right}), ALIE (\textit{row 2, left}), and FOE (\textit{row 2, right}) attacks. Privacy budget after $T = 400$ steps is $(\epsilon, \delta) = (0.19, 10^{-4})$.}
\label{fig:plots_phishing_6}
\end{figure*}

\clearpage
\paragraph{Discussion.}
We consider four different attacks executed by the adversarial nodes, and report on the performance of the algorithms in three different privacy regimes. Our observations are twofold.

First, as expected, we see that as the privacy regime becomes more demanding, the performances of DP-DSGD and SMEA degrade both in terms of test accuracy and training loss. This confirms that the standard privacy-utility trade-off also occurs in the presence of adversarial workers.
Second, we see that under all three privacy regimes, \algoname{} with SMEA is able to successfully mitigate adversarial attacks while still ensuring strong levels of differential privacy. Indeed, the final accuracies reached by \algoname{} with SMEA are around 80\% in the \textit{low} and \textit{moderate} privacy regimes, and around 75\% in \textit{high} privacy (a bit lower under the FOE attack). On the other hand, the training losses are decreasing under all attacks and in all privacy regimes, sometimes asymptotically matching the curves of DP-DSGD (e.g., the LF attack in all three privacy regimes, the ALIE attack in \textit{low} and \textit{moderate} privacy). The same observations hold for \algoname{} with Filter.

\end{document}